\documentclass[11pt,letterpaper]{article}

\usepackage{amsmath,amssymb,amsthm,amsfonts}
\usepackage{bm}
\usepackage{deepthink}
\usepackage{hyperref}
\usepackage{url}
\usepackage{booktabs}
\usepackage{nicefrac}
\usepackage{microtype}
\usepackage{xcolor}
\usepackage{graphicx}
\usepackage{algorithm}
\usepackage{algorithmic}
\usepackage{wrapfig}
\usepackage{placeins}
\usepackage[nameinlink]{cleveref}
\usepackage[
  backend=biber,
  style=alphabetic,
  natbib=true,
  maxbibnames=100,
  minbibnames=100,
  maxcitenames=2,
  mincitenames=2
]{biblatex}
\usepackage[most]{tcolorbox}
\usepackage{subcaption}
\usepackage{titletoc}
\usepackage{dsfont}

\usepackage{amsmath,amsfonts,bm}

\def\eqref#1{equation~\ref{#1}}

\def\1{\bm{1}}

\DeclareMathAlphabet{\mathsfit}{\encodingdefault}{\sfdefault}{m}{sl}
\SetMathAlphabet{\mathsfit}{bold}{\encodingdefault}{\sfdefault}{bx}{n}

\newcommand{\E}{\mathbb{E}}

\newcommand{\R}{\mathbb{R}}

\newcommand{\Var}{\mathrm{Var}}

\newtheorem{theorem}{Theorem}
\newtheorem{lemma}{Lemma}
\newtheorem{proposition}{Proposition}
\newtheorem{corollary}{Corollary}

\usepackage{enumitem}

\theoremstyle{definition}
\newtheorem{remark}{Remark}

\crefname{assumption}{Assumption}{Assumptions}
\Crefname{assumption}{Assumption}{Assumptions}
\crefname{remark}{Remark}{Remarks}
\Crefname{remark}{Remark}{Remarks}
\crefname{equation}{Eq.}{Eqs.}
\Crefname{equation}{Equation}{Equations}
\crefname{figure}{Fig.}{Figs.}
\Crefname{figure}{Figure}{Figures}
\crefname{table}{Table}{Tables}
\Crefname{table}{Table}{Tables}
\crefname{section}{Sec.}{Secs.}
\Crefname{section}{Section}{Sections}
\crefname{appendix}{App.}{App.}
\Crefname{appendix}{Appendix}{Appendices}
\crefname{theorem}{Thm.}{Thms.}
\Crefname{theorem}{Theorem}{Theorems}
\crefname{lemma}{Lemma}{Lemmas}
\Crefname{lemma}{Lemma}{Lemmas}
\crefname{proposition}{Prop.}{Props.}
\Crefname{proposition}{Proposition}{Propositions}
\crefname{corollary}{Cor.}{Cors.}
\Crefname{corollary}{Corollary}{Corollaries}
\crefname{definition}{Def.}{Defs.}
\Crefname{definition}{Definition}{Definitions}
\crefname{assumption}{Assump.}{Assumps.}
\Crefname{assumption}{Assumption}{Assumptions}

\newcommand{\es}{E_s}

\title{Backward-Consistent Diffusion Sampling for Sparsely Observed PDE Inverse Problems}

\newcommand{\jointfirst}{\textsuperscript{\dag}}

\authorblock{
  \href{https://yidapan.org/}{\textbf{Yida Pan}}\textsuperscript{1,}\jointfirst,
  \href{https://ashiqwisc.github.io/}{\textbf{Muhammad H. Ashiq}}\textsuperscript{1,}\jointfirst, 
  \href{https://sites.google.com/view/jcy132/}{\textbf{Chanyong Jung}}\textsuperscript{1}, 
  \href{https://umjiayx.github.io/}{\textbf{Yixuan Jia}}\textsuperscript{1},
  \href{https://scholar.google.com/citations?user=oQbwqPIAAAAJ&hl=en}{\textbf{Jonah M. Miller}}\textsuperscript{2,1},
\href{https://qingqu.engin.umich.edu/}{\textbf{Qing Qu}}\textsuperscript{1,\ddag}, 
\href{https://sites.google.com/view/ismailalkhouri/about?authuser=0}{\textbf{Ismail R. Alkhouri}}\textsuperscript{2,1,\ddag}
}

\affiliation{
  \textsuperscript{1}University of Michigan \quad $\cdot$ \quad \textsuperscript{2}Los Alamos National Laboratory
}

\authornote{
  \jointfirst\ Joint first author \quad
  \textsuperscript{\ddag}\ Joint last author  
}

\abstracttext{
Recovering Partial Differential Equation (PDE) coefficient fields from extremely sparse observations is a severely ill-posed inverse problem for which generative machine learning methods (e.g., diffusion models) have become a leading way to encode the prior. Recent state-of-the-art diffusion solvers lift these priors to function spaces, finding a physics-consistent reconstruction in the output space of the diffusion denoiser. We prove that, in a discontinuous PDE setting, output space methods can result in failure to appropriately minimize the unobserved error with the correct coefficient field. Consequently, we propose \textbf{Fun}ction Space \textbf{B}ackward-\textbf{C}onsistent \textbf{S}ampling (FunBCS), an \textit{input space} optimization approach for solving PDE problems which aims to find the best input such that the denoiser reconstruction is physics-consistent. We then prove that FunBCS appropriately minimizes the unobserved error, unlike output space optimization methods. Per our theoretical analysis, we also provide insights on how to dynamically allocate the number of input space optimization steps used throughout the sampling process. Our evaluations, across four PDE inverse problems (including the discontinuous Darcy flow), demonstrate that FunBCS reduces the reconstruction error by $27$--$64\%$ while running $1.4$--$2.1\times$ faster when compared to the current state-of-the-art.

}


\date{\today}
\correspondence{\href{mailto:ismailal@umich.edu}{ismailal@umich.edu}}
\resources{\href{https://yidapan.org/funbcs-website/}{Project page} \quad $\cdot$ \quad \href{https://github.com/bamb00zled82/FunBCS}{Code}}

\headerlogo{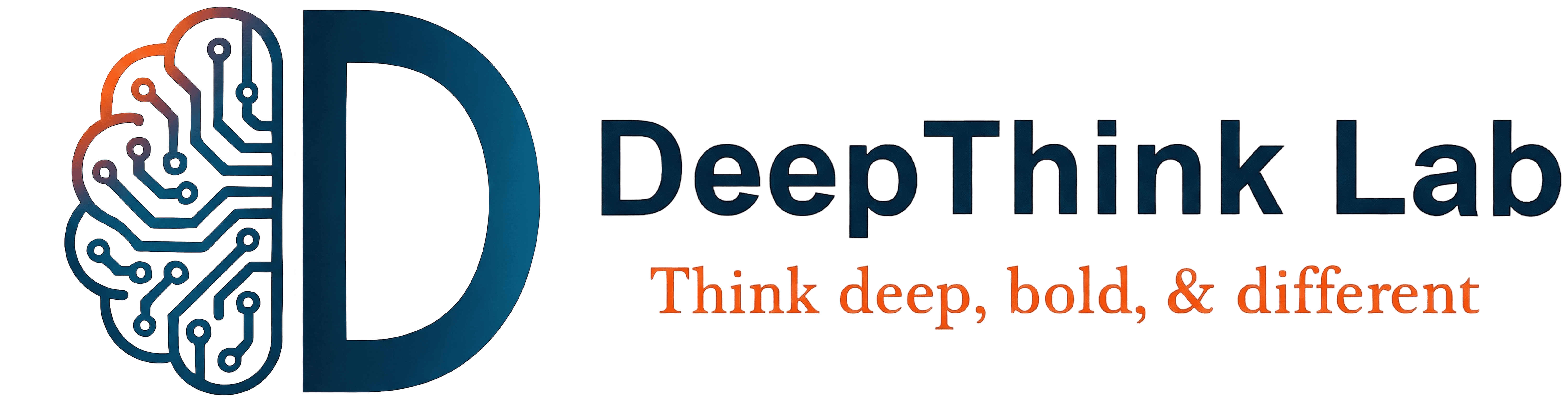}{https://deepthink-umich.github.io}

\begin{document}

\makeDeepthinkHeader

\begin{figure}[h]
\centering
\includegraphics[width=0.85\linewidth]{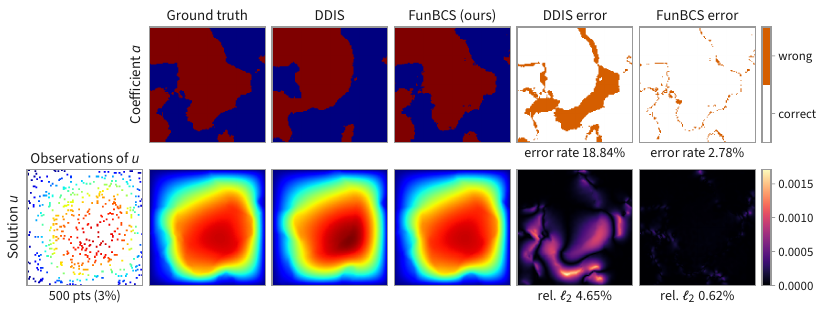}
\caption{\small{Darcy reconstruction from $3\%$ sparse solution observations (first column). The last two columns show the reconstruction error. DDIS, a SOTA output space optimization method, misplaces the discontinuities, mostly by assigning low permeability to regions where the true field is highly permeable, whereas FunBCS recovers the discontinuity structure significantly more faithfully.}}
\label{fig:darcy}
\end{figure}

\newpage
\tableofcontents

\section{Introduction}
\label{sec:intro}

Inverse problems governed by partial differential equations (PDEs) seek to estimate an unknown physical coefficient field $\bm{a}$ from sparse and sometimes noisy observations $\bm{u}_{\text{obs}}$. These problems are ubiquitous in observational science, where an incomplete observation must be described by a physical model, and arise in many applications including geophysics \citep{virieux2010overview,bui2013computational}, astrophysics \citep{PengKNML,GWParameterEsitmationReview}, biomedical imaging \citep{arridge1999optical,poudel2019survey}, and materials science \citep{bonnet2005inverse,avril2008overview}. Recovering the physical coefficient field $\bm{a}$ is \textit{ill-posed}, since the observations $\bm{u}_{\text{obs}}$ often correspond to several possible coefficient fields $\bm{a}$. Thus, one aims to sample from the posterior $\pi(\bm{a}|\bm{u}_{\text{obs}})$ \citep{stuart2010inverse}, e.g., classical methods use Markov Chain Monte Carlo (MCMC), which can exhibit  slow convergence for multimodal posteriors \citep{guan2007small} or in high dimensions \citep{roberts1997weak}. Additionally, practical use of classical methods often requires a fast surrogate for the forward model; these surrogates are frequently highly application-specific, limiting their generalizability \citep{tiglio2022reduced}.

Recently, generative diffusion models \citep{ho2020ddpm} have emerged as a powerful framework for representing expressive, multimodal  priors and enabling posterior sampling in high-dimensional inverse problems \citep{chung2023dps,li2024dcdp,zhang2025daps}. Since $\bm{a}$ and $\bm{u}_{\text{obs}}$ are discretized functions, a recent line of work lifts diffusion inverse solvers to infinite-dimensional function spaces \citep{lim2023scorefs}. Existing diffusion PDE inverse solvers, including the current state-of-the-art Decoupled Diffusion Inverse Solver (DDIS) \citep{lin2026ddis}, first denoise a  diffusion state $\bm{a}_t$ to obtain an estimate of the clean coefficient field %
$\hat{\bm{a}}_0$
and then separately modify this estimate to enforce measurement consistency so that $\bm{a}_0 \approx \hat{\bm{a}}_0$ and $\bm{u}_{\mathrm{obs}} \approx \mathcal{A}(\bm{a}_0)$, where $\mathcal{A}$ is the forward operator. We refer to this class of methods as \textit{output space optimization} methods, since this amounts to optimizing in the denoiser's output space while holding the denoiser input fixed (\Cref{fig:main}a).

\begin{figure}[t]
\centering
\includegraphics[width=\linewidth]{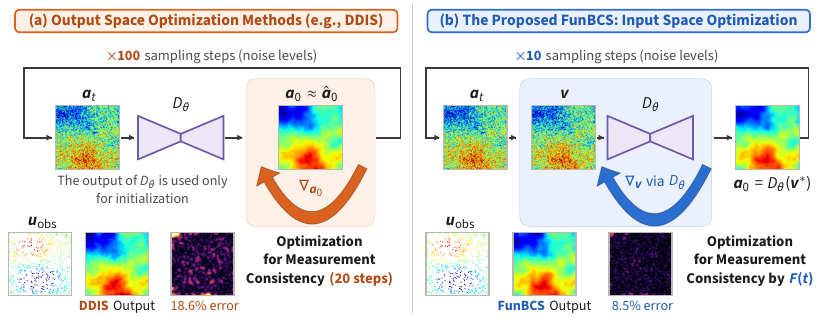}
\caption{Overview of the differences between output space optimization methods (e.g., DDIS) (a) and our input space optimization method (b). DDIS enforces measurement consistency on a reconstruction near the denoiser's ($D_\theta$) output, requiring $20$ optimization steps at each of its $100$ noise levels. In contrast, our input space optimization jointly incorporates measurement consistency and denoising by iteratively computing the denoiser output, requiring $10$ noise levels for which the number of input-space optimization steps, $F(t)$, is dynamically allocated using the scheduler presented in \cref{sec:efficiency}.}
\label{fig:main}
\vspace{-0.4cm}
\end{figure}

However, output space optimization can be ineffective for sparsely observed PDE inverse problems. Measurement consistency updates may leave the unobserved component of the reconstruction $\bm{a}_0$ unchanged from the denoiser prediction, allowing errors in directions invisible to the forward operator to persist. This limitation is particularly pronounced for discontinuous $\bm{a}$, whose prior can exhibit distinct modes corresponding to different discontinuity locations. When these modes differ substantially in unobserved directions, measurement-consistency updates alone may fail to correct an erroneous mode selected by the denoiser. This motivates the following research question:
\begin{tcolorbox}
    \begin{center}
    \textit{Can we obtain a diffusion-based PDE inverse solver that modifies both the observed and unobserved components of the denoiser prediction?}
\end{center}
\end{tcolorbox}

\noindent 
\textbf{Contributions}.
We address the question affirmatively, showing that \textit{input space optimization}, or selecting the \textit{input} to the denoiser such that measurement consistency is satisfied (\cref{fig:main}b), can effectively update unobserved components. For finite-dimensional diffusion, input-space optimization is \textit{backward-consistent} sampling \citep{alkhouri2025sitcom}, where the reconstruction is constrained to be in the denoiser's range. Our contributions can be summarized as follows:

\begin{itemize}[leftmargin=*]
    \item \textbf{Theoretical analysis of output space optimization}. We prove that, on a discontinuous Poisson PDE inverse problem, output space optimization methods, such as DDIS, can fail to change the unobserved components of the denoiser prediction and hence fail to recover the correct prior mode. 
    \item \textbf{Input space optimization analysis, scheduling function, and algorithm}. We theoretically demonstrate that input space optimization can change the unobserved component, thus obtaining less unobserved error than DDIS and recovering the correct prior mode. Furthermore, we analyze where to allocate most of the input optimization steps across the sampling process, and use this to design a scheduling function for various noise levels. Combining these, we introduce \textbf{Fun}ction Space \textbf{B}ackward-\textbf{C}onsistent \textbf{S}ampling (FunBCS), our diffusion input space optimization method. 
    \item \textbf{Comprehensive evaluation}. We extensively evaluate our method across four smooth and discontinuous PDE inverse problems. Compared to the current state-of-the-art, our method reduces the reconstruction error by $27$--$64\%$ (\cref{fig:darcy}) while running $1.4$--$2.1\times$ faster. For example, on the same hardware, FunBCS at $8$ seconds per sample achieves $11.8\%$ relative $\ell_2$ error on Poisson, whereas DDIS, at best, obtains $13.7\%$ error while running for $72$ seconds (\cref{fig:runtime}). 
\end{itemize}

\section{Background}
\label{sec:prelim}

\textbf{PDE inverse problems}. A partial differential equation (PDE) can be written as: 
\begin{equation}
    \begin{cases}
        \mathcal{J}_{a(\cdot)}u(\cdot) = f(\cdot) &\text{ on } \Omega, \\ 
        u(\cdot) = g(\cdot) &\text{ on } \partial \Omega\:,
    \end{cases}
    \label{eq:PDE}
\end{equation}
for functions $u(\cdot) \in \mathcal{H}$ and $a(\cdot) \in \mathcal{F}$, where $\mathcal{H}, \mathcal{F}$ are separable Hilbert spaces, $\Omega \subset \mathbb{R}^d$ is a domain, $\partial\Omega$ is the domain boundary, and $\mathcal{J}_{a(\cdot)}$ is a (potentially nonlinear) differential operator that depends on $a(\cdot)$. The goal of a PDE inverse problem is to reconstruct physical coefficient field $a(\cdot)$ from sparse observations of PDE solution $u(\cdot)$, assuming known $f(\cdot)$ and $g(\cdot)$. Specifically, we have: 
\begin{equation}
    u_{\mathrm{obs}}(\cdot) = \mathcal{A}(a(\cdot)) + \eta(\cdot)\:,
    \label{eq:inverse_problem}
\end{equation}
where $\mathcal{A}$ is a \textit{measurement operator} and $\eta(\cdot)$ is a noise term. Letting $\mathcal{G} : \mathcal{F} \to \mathcal{H}$ be the solution operator which  maps coefficient fields to PDE solution fields, and $\mathcal{O} : \mathcal{H} \to \mathcal{H}$ be an observation operator (e.g., a mask), we can write $\mathcal{A} = \mathcal{O} \circ \mathcal{G}$. Critically, $u_{\mathrm{obs}}(\cdot)$ often corresponds to several possible $a(\cdot)$, and small variations in $u_{\mathrm{obs}}(\cdot)$ result in drastically different $a(\cdot)$; that is, $\mathcal{A}$ is not injective and the inverse problem corresponding to \cref{eq:inverse_problem} is severely ill-posed. To address this, research on PDE inverse problems often takes the Bayesian perspective \citep{stuart2010inverse}, where one assumes a prior $\pi(a(\cdot))$ on $\mathcal{F}$ and aims to sample from the posterior $\pi(a(\cdot) | u_{\mathrm{obs}}(\cdot))$ via Bayes' theorem, i.e., $\pi(a(\cdot) | u_{\mathrm{obs}}(\cdot)) \propto \pi(u_{\mathrm{obs}}(\cdot) | a(\cdot)) \pi(a(\cdot))$. 

Throughout, we consider the \textit{discretized} version of the problem in \cref{eq:inverse_problem}; in particular, we select $n$ collocation points $\bm{x}_i \in \Omega$ for $i = 1, ..., n$, and evaluate only at those finitely many points\footnote{Without loss of generality, we note that not all traditional numerical methods use collocation points or pointwise values. For example, finite volume and finite element methods may use \textit{averages} and spectral methods may use Fourier coefficients \citep{quarteroni2009numerical}.}. For example, we discretize $a(\cdot)$ as $\bm{a} = (a(\bm{x}_1), ..., a(\bm{x}_n))$; we respectively discretize $u(\cdot),u_{\mathrm{obs}}(\cdot), \eta(\cdot), \mathcal{G},$ and $\mathcal{O}$ as $\bm{u} \in \mathbb{R}^n$, $\bm{u}_{\mathrm{obs}}, \bm{\eta} \in \mathbb{R}^m$, $\mathcal{G}_h : \mathbb{R}^n \to \mathbb{R}^n$, and $\mathcal{O}_h : \mathbb{R}^n \to \mathbb{R}^m$ where $m \ll n$. Importantly, note that $\mathcal{G}_h$ and $\mathcal{O}_h$ can remain nonlinear in this setting, and thus do not necessarily admit matrix representations. We thus let $\mathcal{A}_h = \mathcal{O}_h\mathcal{G}_h$. A detailed overview of our notation is provided in \cref{sec:notation}.

\noindent 
\textbf{Function space diffusion}. One important way of obtaining a prior $\pi(\bm{a})$ for the PDE inverse problem in \cref{eq:inverse_problem} is to set $\pi(\bm{a}_0) := \pi(\bm{a})$ and sample $\bm{a}_0 \sim \pi(\bm{a}_0)$ by running a learned reverse diffusion process initialized with a sample from a simple distribution $\bm{a}_T \sim \pi(\bm{a}_T)$. Since $a(\cdot)$ is a function, this involves training diffusion on (discretized) functions \citep{lim2023scorefs}. After discretization, this is equivalent to the standard diffusion process \citep{ho2020ddpm} with two slight alterations, specifically that the noise has kernel matrix $\bm{K}$ as its covariance instead of $\bm{I}$ and that $\bm{K}$ multiplies the score in Tweedie's formula. More specifically, we have the variance exploding (VE) diffusion forward process \citep{song2021sde} at $0 < t_0<\cdots<t_N=T$, where $N$ is the number of sampling steps:
\begin{equation}
 \bm{a}_{t_{k+1}}=\bm{a}_{t_k}
 +\sqrt{\sigma_{t_{k+1}}^2-\sigma_{t_k}^2}\,\bm{\varepsilon}_k,
 \qquad k=0,\ldots,N-1,
 \label{eq:ve-forward-discrete}
\end{equation}
where $\bm{\varepsilon}_k\sim\mathcal{N}(\bm{0}, \bm{K})$ are independent across $k$. We perform reverse diffusion with respect to Tweedie's formula given by:  
\begin{equation}
     D^*(\bm{a}_t,t) := \mathbb{E}[\bm{a}_0 | \bm{a}_t] =\bm{a}_t+\sigma_t^2 \bm{K} \nabla_{\bm{a}_t} \log p_t(\bm{a}_t). 
    \label{eq:tweedie-denoiser}
\end{equation}
For a learned score $\bm{s}_{\bm{\theta}} \approx \nabla_{\bm{a}_t} \log p_t(\bm{a}_t)$ \citep{song2021sde}, we denote the denoiser by $D_{\bm{\theta}} \approx \mathbb{E}[\bm{a}_0 | \bm{a}_t]$. We provide a more detailed formal characterization of function space diffusion in \cref{sec:function_space_diffusion}.

\noindent 
\textbf{Diffusion PDE inverse solvers}. Diffusion models \citep{sohl2015deep} are successful at modeling high-dimensional priors \citep{rombach2022high}, including in function spaces \citep{lim2023scorefs}. Thus, several diffusion solvers have been developed for PDE inverse problems, with standard \citep{huang2024diffusionpde} or function space \citep{yao2025fundps,lin2026ddis} diffusion priors. At a sampling step $t_k$, these solvers perform two key operations:  (\textbf{D}) using the diffusion denoiser to estimate the coefficient field and (\textbf{M}) enforcing measurement consistency. We use the term \textit{output space optimization} to describe methods where (\textbf{D}) and (\textbf{M}) are \textit{disjoint}, since this amounts to optimizing in the output space of the denoiser while leaving the denoiser input fixed. From this perspective, DiffusionPDE \citep{huang2024diffusionpde}, FunDPS \citep{yao2025fundps}, and Decoupled Diffusion Inverse Solver (DDIS) \citep{lin2026ddis} are all output space optimization procedures. We discuss these solvers in further detail in \cref{sec:recurrences}. We provide an extensive related work section in \cref{sec:related_work} covering classical numerical solvers, diffusion inverse solvers and theory, approaches to updating unobserved components, and guidance scheduling.

\section{Analysis of Output Space Optimization Sampling} 
\label{sec:analysis}

 We instantiate DDIS as an example output space optimization method for which we provide a discontinuous PDE inverse problem, demonstrating that unobserved differences in prior modes arise naturally. Consequently, DDIS does not alter the denoiser's unobserved component, and hence fails to recover the correct prior mode.  All proofs are deferred to \cref{sec:proofs}. 
 
 To sample from $\pi(\bm{a} | \bm{u}_{\mathrm{obs}})$, diffusion inverse solvers, as discussed earlier, do the following at timestep $t_k$ with noisy instance $\bm{a}_{t_k}$: %

\vspace{-3mm}
\begin{itemize}
    \item (\textbf{D}) compute the estimated output sample $\hat{\bm{a}}_0^{(t_k)}$ via, e.g., \cref{eq:tweedie-denoiser} from $\bm{a}_{t_k}$ and 
    \item (\textbf{M}) enforce \textit{measurement consistency} on the reconstruction $\bm{a}_0^{(t_k)} \approx \hat{\bm{a}}_0^{(t_k)}$, i.e., incorporate likelihood information $\pi(\bm{u}_{\mathrm{obs}}|\bm{a})$ so that $\bm{u}_{\mathrm{obs}} \approx \mathcal{A}_h(\bm{a}_0^{(t_k)})$. 
\end{itemize}

The defining characteristic of output space inverse solvers is that the operations \textbf{(D)} and \textbf{(M)} are performed \textit{separately} at time $t_k$. Notably, output space inverse solvers may perform other operations before, between, or after \textbf{(D)} and \textbf{(M)}. %
Next, we formalize the DDIS sampler. Fix consecutive diffusion times $t_k>t_{k-1}$. DDIS assumes Gaussian likelihood at timestep $0$, with $\gamma > 0$, yielding \[p(\bm{u}_{\mathrm{obs}}\mid \bm{a}_0)
    \propto\exp\!\left\{-\frac{\|\mathcal{A}_h \bm{a}_0-\bm{u}_{\mathrm{obs}}\|^2}
                                   {2\gamma^2}\right\}.\] 
At time $t_k$, DDIS approximates the denoising law $\pi(\bm{a}_0 \mid \bm{a}_{t_k})$ with $\pi_D(\bm{a}_0 \mid \bm{a}_{t_k}^{(D)})$ which has density $q_D(\bm{a}_0 \mid \bm{a}_{t_k}^{(D)})=\mathcal{N}(\bm{a}_0;\hat{\bm{a}}_0^{(t_k)},r_{t_k}^2\bm I)$ with $\hat{\bm{a}}_0^{(t_k, D)} := D_{\bm{\theta}}(\bm{a}_{t_k}^{(D)},t_k)$, where superscript $(D)$ denotes DDIS. Setting $r_{t_k}=\sigma_{t_k}$ gives: %
\begin{equation}
    q_{D,t_k}(\bm{a}_0\mid \bm{a}_{t_k}^{(D)},\bm{u}_{\mathrm{obs}})
    \propto\exp\!\left\{
       -\frac{\|\mathcal{A}_h \bm{a}_0-\bm{u}_{\mathrm{obs}}\|^2}{2\gamma^2}
       -\frac{\|\bm{a}_0-\hat{\bm{a}}_0^{(t_k, D)}\|^2}{2\sigma_{t_k}^2}
    \right\}.
    \label{eq:ddis-correction}
\end{equation}
The next iteration of DDIS is given by: 
\begin{equation}
    \bm{a}_{t_{k-1}}^{(D)}=\tilde{\bm{a}}_0^{(t_k)}+\sigma_{t_{k-1}}\bm{\varepsilon}_k^{(D)}, \qquad \tilde{\bm{a}}_0^{(t_k)} \sim q_{D,t_k}(\cdot | \bm{a}_{t_k}^{(D)}, \bm{u}_{\mathrm{obs}}), \; \bm{\varepsilon}_k^{(D)} \sim \mathcal{N}(\bm{0}, \bm{K}),
    \label{eq:DDIS-update}
\end{equation}
where $\varepsilon_{k}^{(D)}$ are drawn independently. In practice, DDIS uses Langevin dynamics \citep{roberts1996exponential} to sample from \cref{eq:ddis-correction} for use in \cref{eq:DDIS-update} at every diffusion timestep $t_k$. 
In particular, DDIS finds $\tilde{\bm{a}}_0^{(t_k)}$ close to $\hat{\bm{a}}_0^{(t_k, D)}$, enforces measurement consistency on $\bm{a}_0$, and then renoises to $t_{k - 1}$. This is run until time $t_0$ to obtain a final reconstructed physical coefficient field. %

Next, towards explaining why output space optimization can fail, we provide an example PDE inverse problem on which we will show that DDIS fails:  
\begin{align}
    &u^{\prime\prime}(\cdot) = a(\cdot),\; \text{ on } (0,1), ~~~~
    u(1) = u(0) = 0.
    \label{eq:Poisson's equation}
\end{align}
 We now assume that $a(\cdot)$ has a single discontinuity at  $z \in (0,1)$, so that admissible coefficient fields are $\{ a(\cdot) \in L^2(0,1): a(x) = a^{(1)} \mathds{1}_{(0, z)} + a^{(2)} \mathds{1}_{[z, 1]}, \; z \in (0,1)\}$, where $L^2$ are square-integrable functions. We assume that the coefficients $a^{(1)}, a^{(2)}$ are known beforehand so that the uncertainty is only over the discontinuity location; upon reconstruction one can place $a^{(1)}, a^{(2)}$ on either side of $z$. Furthermore, we consider two physical coefficient fields $a_-(\cdot)$ and $a_+(\cdot)$ with respective distinct discontinuity locations $z_- \neq z_+$. With discretization at $n$ collocation points $\bm{x}_i \in (0,1)$, consider a prior which weights these two discontinuity locations equally, i.e., \[\pi(\bm{a}_0)=\tfrac12\delta_{\bm{a}_-}+\tfrac12\delta_{\bm{a}_+}.\] 
 
 In this setting, we consider a noiseless Bayesian inverse problem per \cref{eq:inverse_problem} given by $\bm{u}_{\mathrm{obs}} = \bm{OG}(\bm{a}_+)$, where $\bm{G}$ is a linear 
 solution operator and $\bm{O}$ is a selection matrix. Thus, given $\bm{u}_{\mathrm{obs}}$, a successful diffusion inverse solver should recover $\bm{a}_+$ at the end of the sampling process. Note that while the prior on $\bm{a}_+$ and $\bm{a}_-$ is not exactly learnable by a diffusion model, it can be approximated by a mixture of Gaussians with small variance. The noiseless setting of $\bm{\eta} = \bm{0}$ matches both our experiments in \cref{sec:exp} and the DDIS experiments in \citet{lin2026ddis}.

  Note that this setting is a linear analogue of the nonlinear Darcy inverse problem in \cref{fig:darcy} where we assume that $a^{(1)}$ and $a^{(2)}$ are known, and find the correct discontinuity location for linear Poisson. We simplify this setting by assuming only two possible discontinuity locations; we find that this is sufficient to observe the limitations of output space optimization methods, and the benefits of input space optimization (which we will discuss in the upcoming sections). 
  
\begin{remark}
  Let $\bm{A} := \bm{OG}$ be the forward operator and $\bm{P}_{\mathrm{ker} \bm{A}}$ be the orthogonal projection matrix onto its nullspace.
We prove in \cref{prop:finite_diff_assumps}, described in \cref{proof:proof_of_finite_diff_assumps}, that the setting described above yields that the measurements differ ($\bm{A}\bm{a}_+ \neq \bm{A}\bm{a}_-$) and that the nullspace components differ ($\bm{P}_{\mathrm{ker} \bm{A}}\bm{a}_+ \neq \bm{P}_{\mathrm{ker} \bm{A}}\bm{a}_-$) with high probability. We thus take these as assumptions throughout. \end{remark}
 In the setting above, we find that output space optimization can do no better than the denoiser in unobserved error. This issue is critical if the denoiser chooses the incorrect prior mode. Formally: 
 \begin{proposition}[DDIS preserves denoiser nullspace component]
 \label{prop:ddis_fail_prelim}
     Fix $t_k > t_{k - 1}$. Suppose $\bm{A}\bm{a}_+ \neq \bm{A}\bm{a}_-$ and $\bm{P}_{\mathrm{ker} \bm{A}}\bm{a}_+ \neq \bm{P}_{\mathrm{ker} \bm{A}}\bm{a}_-$ and let $\bar{\bm{a}}_{D,t_k - 1} := \mathbb{E}[\bm{a}_{t_k - 1}^{(D)} | \bm{a}_{t_k}^{(D)}, \bm{u}_{\mathrm{obs}}]$, where $\bm{a}_{t_k - 1}^{(D)}, \bm{a}_{t_k}^{(D)}$ are given by \cref{eq:DDIS-update} under the exact denoiser $D^*(\cdot, t_k)$ in \cref{eq:tweedie-denoiser}. We then have that: 
     \begin{equation}
          \|\bm{P}_{\mathrm{ker} \bm{A}} (\bar{\bm{a}}_{D, t_k - 1} - \bm{a}_+)\|_2 = \|\bm{P}_{\mathrm{ker} \bm{A}} (D^*(\bm{a}_{t_k}^{(D)}, t_k) - \bm{a}_+)\|_2\;.
          \label{eq:ddis_does_no_better}
     \end{equation}
 \end{proposition} 
 \begin{remark}
 \label{remark:ddis_fail_prelim_remark}
      Per \cref{eq:ddis_does_no_better}, DDIS at step $t_{k -1 }$ does not grow any closer in nullspace to $\bm{a}_+$ than the denoiser does at step $t_k$. This is precisely because the measurement consistency error in \cref{eq:ddis-correction} does not update the unobserved components, while the second term aligns the reconstruction with the denoiser. While the  reconstruction $\tilde{\bm{a}}_0^{(t_k)}$ influences the denoiser input at the next diffusion timestep after renoising, this merely gives the denoiser another chance to correct itself, which is not guaranteed.  Consequently, if the denoiser remains mistaken at later timesteps, particularly near the end of the reverse process, the corresponding nullspace error may persist in the final reconstruction.  We provide experimental validation for this on Darcy flow in \cref{sec:exp}.
 \end{remark}

      We note that in practice DDIS uses the probability flow ODE instead of the denoiser; per \cref{lemma:ddis-kernel} in \cref{sec:lemmas}, the above \cref{prop:ddis_fail_prelim} still holds in this case.  \cref{prop:ddis_fail_prelim} immediately yields that DDIS fails to recover $\bm{a}_+$ more generally:

 \begin{corollary}[DDIS does not recover $\bm{a}_+$]
 \label{corollary:ddis-fail}
Fix $t_k > t_{k-1}$. With the same assumptions as \cref{prop:ddis_fail_prelim}, let $k_{t_k}(\bm{v}) = \mathbb{P}(\bm{a}_0 =\bm{a}_+ | \bm{a}_{t_k} = \bm{v})$ be the regular conditional probability of recovering the correct $\bm{a}_+$. Then, this probability is well-defined and we have that: 
\begin{equation}
    \|\bar{\bm{a}}_{D, t_k - 1} - \bm{a}_+\|_2  \geq (1-k_{t_k}(\bm{a}_{t_k}^{(D)}))||\bm{P}_{\mathrm{ker} \bm{A}}\bm{d}||_2 > 0. 
    \label{eq:ddis-fail}
\end{equation}
 \end{corollary}
\begin{remark}
    Note that \cref{corollary:ddis-fail} holds for $t_k$ arbitrary, including at the end of the sampling process, and that \cref{eq:ddis-fail} continues to hold as $\sigma_{t_0} \to 0$. Furthermore, as discussed in \cref{remark:ddis_fail_prelim_remark}, DDIS' inability to recover the true $\bm{a}_+$ is exacerbated when the wrong mode is preferred (i.e., when $k_{t_k}(\bm{a}_{t_k}^{(D)}) < \frac{1}{2}$), particularly near the end of the reverse process; we confirm this empirically on Darcy flow in \cref{sec:exp}. 
    \end{remark}

\section{Input Space Optimization Sampling in Function Space}
\label{sec:method}

In this section, we propose merging steps (\textbf{D}) and (\textbf{M}) discussed above, iteratively denoising and finding the \textit{input} to the denoiser such that measurement consistency holds. We demonstrate theoretically that this scheme does not fail in the same Poisson setting of the previous section. Subsequently, we provide insights and a procedure on how to select the number of optimization steps for the input at each sampling iteration, and present the FunBCS algorithm. To motivate input space optimization, we note the following fact about the denoiser in the setting of \cref{sec:analysis}: 
\begin{proposition}[Denoiser nullspace component]
  \label{prop:exact_denoiser}
    Suppose $\bm{A}\bm{a}_+ \neq \bm{A}\bm{a}_-$ and $\bm{P}_{\mathrm{ker} \bm{A}}\bm{a}_+ \neq \bm{P}_{\mathrm{ker} \bm{A}}\bm{a}_-$  hold. Let $\bm{d} = \bm{a}_+ - \bm{a}_-$ be the difference between the two prior modes, and $\bm{v}$ be a fixed denoiser input. We then have: 
    \begin{equation}
    \begin{aligned}
             &D^*(\bm{v}, t_k) = \bm{a}_-  + k_{t_k}(\bm{v})\bm{d}, ~~~
            & \bm{P}_{\mathrm{ker} \bm{A}}D^*(\bm{v}, t_k) =  \bm{P}_{{\mathrm{ker} \bm{A}}}  \bm{a}_- + k_{t_k}(\bm{v})\bm{P}_{{\mathrm{ker} \bm{A}}} \bm{d}\;. 
    \end{aligned}
    \end{equation}
\end{proposition}
Per \cref{prop:exact_denoiser}, the unobserved component of $D^*(\cdot, t_k)$ can be changed by updating $k_t(\bm{v})$; since $k_t(\cdot)$ is a function of the denoiser input, it suffices to optimize in the input space of the denoiser. Thus, to enable input space optimization, at every $t_k$, we propose solving: 
\begin{equation}
    \arg \min_{\bm{v}} \|\mathcal{O}_h \circ \mathcal{G}_{h, \bm{\phi}} \circ D_{\bm{\theta}}(\bm{v}, t_k) - \bm{u}_{\mathrm{obs}}\|_2^2\;, \qquad \mathcal{A}_h = \mathcal{O}_h \circ \mathcal{G}_{h, \bm{\phi}}\;, 
    \label{eq:backward_consistency}
\end{equation}
for a neural surrogate $\mathcal{G}_{h, \bm{\phi}}$, since $\mathcal{G}_h$ may be difficult to differentiate. Using  superscript ($B$) for FunBCS, consider the gradient flow on $\bm{v}$ with respect to the objective in \cref{eq:backward_consistency} with $\gamma > 0$ and optimization time $s \geq 0$:  
\begin{equation}
    \frac{d\bm{v}(s)}{ds}
    =-\frac{1}{2\gamma^2}\left.\nabla_{\bm{v}}
       \|\mathcal{O}_h \circ \mathcal{G}_{h, \bm{\phi}} \circ D_{\bm{\theta}}(\bm{v}, t_k)-\bm{u}_{\mathrm{obs}}\|^2\right|_{\bm{v} = \bm{v}(s)},
    \qquad \bm{v}(0)=a_{t_k}^{(B)}\;.
    \label{eq:sitcom-input-flow}
\end{equation}
\begin{remark}
After discretizing \cref{eq:sitcom-input-flow}, this amounts to iteratively denoising and optimizing the input, thus merging steps (\textbf{D}) and (\textbf{M}). Under the exact denoiser, the solution of \cref{eq:sitcom-input-flow} exists and is unique per \cref{lemma:sitcom-existence} in  \cref{sec:lemmas}, and the probability $k_{t_k}(\bm{v}(s))$ is monotonically increasing in $s$ per \cref{lemma:sitcom-scalar-flow} in \cref{sec:lemmas}.  Additionally, input space optimization is equivalent to \textit{preconditioned} output space optimization, with the preconditioner (which we further discuss in \cref{sec:preconditioning}) \[\nabla_{\bm{v}}D^*(\bm{v}, t_k)\nabla_{\bm{v}}D^*(\bm{v}, t_k)^T|_{\bm{v} = \bm{v}(s)}.\] 
\end{remark}

\begin{algorithm}[t]
\caption{Function Space Backward-Consistent Sampling (FunBCS)}
\label{alg:sampler}
\begin{algorithmic}[1]
\REQUIRE observations $\bm{u}_{\mathrm{obs}}$; observation operator $\mathcal{O}$; denoiser $D_{\bm{\theta}}$; neural operator surrogate $\mathcal{G}_{h,\phi}$; diffusion schedule $\{\sigma_{t_k}\}$; optimization step function $F(t_k)$; Adam step size $\alpha$.
\STATE $\bm{a}_{t_N}^{(B)} \leftarrow \sigma_{t_N}\,\varepsilon_N$, \quad $\varepsilon_N \sim \mathcal{N}(0, \bm{K})$ \hfill $\triangleright$ initial latent drawn from the training GRF
\FOR{$k = N , \ldots, 1$}
  \STATE $\bm{v}_{0} \leftarrow \bm{a}^{(B)}_{t_k}$ \hfill $\triangleright$ warm start at the current noisy state
  \FOR{$s = 0, \ldots, F(t_k)-1$}
    \STATE $\mathcal{J} \leftarrow \tfrac{1}{2\gamma^2}\lVert\mathcal{O}_h \circ \mathcal{G}_{h,\phi}(D_{\bm{\theta}}(\bm{v}_s, \sigma_{t_k}))- \bm{u}_{\mathrm{obs}}\rVert^2_2$ \hfill $\triangleright$ compute denoiser output (\textbf{D})  
    \STATE $\bm{v}_{s + 1} \leftarrow \mathrm{AdamStep}_\alpha(\bm{v}_s,\, \nabla_{\bm{v_s}} \mathcal{J})$ \hfill $\triangleright$ enforce measurement consistency (\textbf{M}) 
  \ENDFOR
  \STATE $\bm{a}_{{t_k}-1}^{(B)} \leftarrow D_{\bm{\theta}}(\bm{v}_{F(t_k)}, \sigma_{t_k}) + \sigma_{t_{k-1}}\,\bm{\varepsilon}_K$, \quad $\bm{\varepsilon}_K\sim \mathcal{N}(0, \bm{K})$ \hfill $\triangleright$ renoise
\ENDFOR
\RETURN $\bm{a}_{t_0}^{(B)}$
\end{algorithmic}
\end{algorithm}

Solving \cref{eq:backward_consistency} with \cref{eq:sitcom-input-flow} and renoising yields the update rule at a diffusion time $t_k$ and optimization time $s$, for $\bm{\varepsilon}_k^{(B)}$ drawn independently: 
\begin{equation}
    \bm{a}_{t_{k-1}}^{(B)}=D_{\bm{\theta}}(\bm{v}(s),t_k)+\sigma_{t_{k-1}}\bm{\varepsilon}_k^{(B)},
      \qquad \bm{\varepsilon}_k^{(B)}\sim\mathcal N(\bm{0},\bm{K})\;.
      \label{eq:sitcom-update-rule}
\end{equation}

\subsection{Analysis of Input Space Optimization}
\label{sec:analysis_of_input_space}

To analyze input space optimization, we revisit the setting of \cref{sec:analysis}, where we assume access to the exact operator $\bm{G}$. We begin by analyzing how the recurrence in \cref{eq:sitcom-update-rule} changes the nullspace error: 
\begin{proposition}[Input space optimization reduces nullspace error better]
\label{prop:sitcom_ker_contract}
Under the same assumptions as \cref{prop:ddis_fail_prelim}, we have that
$\mathbb{E}[\bm{a}_{t_k - 1}^{(B)} | \bm{a}_{t_k}^{(B)}, \bm{u}_{\mathrm{obs}}] = D^*(\bm{v}(s), t_k)$ and for any $s > 0$: 
\begin{equation}
      \| \bm{P}_{\mathrm{ker} \bm{A}}(D^*(\bm{v}(s), t_k) - \bm{a}_+)\|_2 \leq \| \bm{P}_{\mathrm{ker} \bm{A}}(D^*(\bm{a}_{t_k}^{(B)}, t_k) - \bm{a}_+)\|_2\;.
\end{equation}
Additionally, there exists an $s^* > 0$ such that $s \geq s^*$ yields:
\begin{equation}
    \| \bm{P}_{\mathrm{ker} \bm{A}}(D^*(\bm{v}(s), t_k) - \bm{a}_+)\|_2 \leq \| \bm{P}_{\mathrm{ker} \bm{A}}(\bar{\bm{a}}_{D, t_k - 1} - \bm{a}_+)\|_2\;.
\end{equation}
\end{proposition}
\begin{remark}
    As expected per our discussion in \cref{prop:exact_denoiser}, the input space optimization recurrence in \cref{eq:sitcom-update-rule} contracts nullspace error better than both the denoiser evaluated at the previous iterate and DDIS, since it can update $k_t(\bm{v})$.  We demonstrate \cref{prop:sitcom_ker_contract} empirically in \cref{sec:discont_poisson}. 
\end{remark} 

Furthermore, input space optimization \textit{can} recover the correct $\bm{a}_+$ in this PDE inverse setting: 
\begin{theorem}[Input space optimization can recover $\bm{a}_+$]
\label{thm:sitcom-success}
Consider diffusion times $t_k > t_{k-1}$. Under the same assumptions as \cref{prop:ddis_fail_prelim}, we have that:
\begin{equation}
\pi_s(\bm{a}_{t_k - 1}^{(B)} | \bm{a}_{t_k}^{(B)}, \bm{u}_{\mathrm{obs}}) = \mathcal{N}(D^*(\bm{v}(s), t_k), \sigma_{t_k - 1}^2\bm{K}) \Rightarrow \mathcal{N}(\bm{a}_+, \sigma_{t_{k - 1}}^2\bm{K}),~\text{as}~s \to \infty\:.
\label{eq:sitcom-success-eq}
\end{equation}
At the final step $t_{k - 1} = t_0$ the limiting law converges in distribution to $\delta_{\bm{a}_+}$ as $\sigma_{t_0} \to 0$.  
\end{theorem}
Thus, for a multimodal prior and with differences between prior modes partially lying in measurement nullspace, which arises naturally in the discontinuous PDE inverse problem setting per \cref{prop:finite_diff_assumps} in \cref{proof:proof_of_finite_diff_assumps}, input space optimization can significantly outperform output optimization. We note that this happens regardless of whether the denoiser begins by preferring the wrong mode; we provide experimental validation for this on Darcy flow in \cref{sec:exp}. The resulting algorithm not only provides a mechanism to change the unobserved components of the denoiser prediction (as we theoretically support), but also reduces the need for a large number of sampling steps as we show in our experimental results in \cref{sec:exp}.

\subsection{Allocating Input Space Optimization Steps Across Noise Levels}
\label{sec:efficiency}

In this subsection, we address the question of how many measurement consistency steps we should take per one noise level (or sampling step). Let $F(t_k)$ denote the number of gradient flow steps taken at noise level $t_k$, so that the sampler takes $M = \sum_{k=1}^{N} F(t_k)$ steps in total. Every such step differentiates through the denoiser, so this allocation sets the cost of the sampler. 

To make our sampler more efficient, we ask how a fixed budget of steps should be distributed across noise levels, and we answer in the setting of \cref{sec:analysis}. We begin with the case in which the sampler leans towards the wrong prior mode and a noise level receives no optimization steps at all. In particular, for $t_{k + 1} > t_{k}$, we let 
\[k_{t_{k+1}}(\bm{a}_{t_{k+1}}^{(B)})
    \leq \frac12-\delta, ~\text{for}~ \delta\in(0,1/2), ~\text{and}~ \bm{a}_{t_k}^{(B)}
    =D^*(\bm{a}_{t_{k+1}}^{(B)},t_{k+1})
    +\sigma_{t_k}\bm{\varepsilon}_k^{(B)}.\] We can then characterize the number of gradient flow steps necessary to escape to the correct mode in the subsequent diffusion step $t_k$ as follows:  
\begin{proposition}[Time explosion with no optimization steps]
\label{prop:wrong-mode-slowdown}
Consider the setting of \cref{sec:efficiency} with $\bm{A}\bm{a}_+ \neq \bm{A}\bm{a}_-$ and $\bm{P}_{\mathrm{ker} \bm{A}}\bm{a}_+ \neq \bm{P}_{\mathrm{ker} \bm{A}}\bm{a}_-$. Let $\bm{v}_{t_k}(s)$ solve \cref{eq:sitcom-input-flow} and define the hitting time $\tau_{t_k}:=\inf\left\{s\geq0:
    k_{t_k}(\bm{v}_{t_k}(s))\geq\frac12\right\}$. Then, for constants $Q, C > 0$, we have that: 
\begin{equation}
    \mathbb{P}\!\left(
        \tau_{t_k}\geq
        \frac{\sigma_{t_k}^4}{C}
        \left[
            \exp\!\left(\frac{Q\delta}{2\sigma_{t_k}^2}\right)-1
        \right]
        \,~~\middle|~~\,
        \bm{a}_{t_{k+1}}^{(B)}
    \right)
    \geq
    1-\exp\!\left(-\frac{Q\delta^2}{8\sigma_{t_k}^2}\right)\;.
    \label{eq:skipped-optimization-hitting-time}
\end{equation}
\end{proposition}
\begin{remark}
     For fixed $\delta>0$, the lower bound on the hitting time tends to $\infty$ exponentially, with probability tending to 1 as $\sigma_{t_k} \to 0$. That is, a lean towards the wrong mode that is not corrected with input space optimization steps can result in an exponentially larger cost to correct in later steps. We conclude that we should retain a positive number of optimization steps to avoid this critical scenario. 
\end{remark}

We next consider at which noise level(s) to allocate most of the gradient flow steps in our method. To address this, we allow the noise standard deviation $\sigma>0$ to vary continuously and consequently use $\sigma$, in place of $t_k$, to index. We then fix a single denoiser input leaning to the wrong mode, and see where the minimal time to escape to the correct mode is. Formally, we fix an input $\bm{v}_*\in\mathbb{R}^n$ satisfying $k_{\sigma}(\bm{v}_*) < \frac{1}{2}$, and compute the fastest hitting time across $\sigma$: 
\begin{proposition}[Selection window]
\label{prop:interior-selection-window}
Under the assumptions  of \cref{prop:wrong-mode-slowdown}, let $w_\sigma(s)$ denote $k_\sigma(\bm{v}(s))$ for $\bm{v}(s)$ a solution to \cref{eq:sitcom-input-flow} at noise level $\sigma$ with initialization $\bm{v}_*$, and define $\tau(\sigma):=\inf\left\{s\geq0:
    w_\sigma(s)\geq\frac12\right\}$. Then, there exist $\sigma_{\mathrm{low}},\sigma_{\mathrm{high}}\in(0,\infty)$
and $\sigma^\star\in[\sigma_{\mathrm{low}},\sigma_{\mathrm{high}}]$ such that $ \tau(\sigma^\star)
    =
    \min_{\sigma>0}\tau(\sigma)$. 
\end{proposition}
\begin{remark} At very high noise levels, the denoiser is  undecided between the two modes, but its output responds only weakly to its input, so the gradient flow moves prohibitively slowly. At very low noise levels, the denoiser responds strongly to input changes, but it has already committed to the wrong mode. The correct prior mode is therefore selected most cheaply at intermediate noise levels, inside an interval $[\sigma_{\mathrm{low}}, \sigma_{\mathrm{high}}]$ that we call the \emph{selection window}. By our experiments in \cref{sec:exp} with our scheduler described in \cref{app:schedule}, we find this generalizes across PDEs.\end{remark}

Together, the two propositions constrain how a budget of optimization steps should be distributed. By \cref{prop:wrong-mode-slowdown}, every noise level must retain a positive number of steps, so that a lean towards the wrong mode is never carried uncorrected into the lower noise levels. By \cref{prop:interior-selection-window}, the remaining steps are best spent inside the selection window. We therefore treat the total number of optimization steps $M$ as a budget, which a scheduling function allocates across noise levels by choosing $F(t_k)$, which keeps a floor of steps at every level and concentrates the rest in a computed selection window, with further details in \cref{app:schedule}. At a fixed budget, this moves steps to where they matter; at a smaller budget, it removes the steps that matter least. 

\noindent
\textbf{Algorithm}. Our sampler is given in \Cref{alg:sampler}. \Cref{eq:sitcom-update-rule} yields backward-consistent sampling, but not in the finite-dimensional regime as in  \citet{alkhouri2025sitcom}. This means that the reconstruction $\bm{a}_0^{(t_k)}$ before renoising is \textit{exactly} the output of the denoiser, rather than approximately. %

\section{Experiments}
\label{sec:exp}

\textbf{Setup.} We evaluate on four PDE inverse problems, each at a resolution of $128^2$ points: \emph{Poisson}, \emph{Helmholtz}, \emph{Navier--Stokes} (NS) and \emph{Darcy flow}, each observed at $500$ randomly placed points of the solution field ($3\%$ of the grid) without noise following \cite{lin2026ddis}. We report the relative $\ell_2$ error of the recovered coefficient field in percent and the spectral energy error $\es$ averaged across 100 inference samples. FunBCS uses $N = 10$ noise levels (sampling steps) with two scheduling schemes, Accuracy mode (AC) and Efficiency mode (EF). See \cref{app:schedule} for details. Other experimental details are given in \cref{app:details}.

\begin{table}[t]
\caption{Recovery of physical coefficient fields $\bm{a}$ with relative $\ell_2$ error and spectral energy error $\es$ on $100$ sampled instances from $3\%$ observations. s/sample is seconds per sample on one NVIDIA A100. FunBCS-AC and FunBCS-EF denote FunBCS in accuracy mode and in efficiency mode, respectively (\cref{app:schedule}). FunBCS (uniform) is FunBCS without the scheduling function, which corresponds to the case of assigning $F(t_k) = 80$ at every noise level ($M = 800$). Errors are mean $\pm$ standard deviation over these instances. Lower is better.}
\vspace{-0.4cm}
\label{tab:main}
\centering
\small
\setlength{\tabcolsep}{3.3pt}
\def\tstd#1{{\scriptsize$\pm$#1}}
\resizebox{\linewidth}{!}{%
\begin{tabular}{l cc cc cc cc c}
\toprule
 & \multicolumn{2}{c}{Poisson} & \multicolumn{2}{c}{Helmholtz} & \multicolumn{2}{c}{Navier--Stokes} & \multicolumn{2}{c}{Darcy} & \\
\cmidrule(lr){2-3}\cmidrule(lr){4-5}\cmidrule(lr){6-7}\cmidrule(lr){8-9}
Method & rel.\ $\ell_2$ & $\es$ & rel.\ $\ell_2$ & $\es$ & rel.\ $\ell_2$ & $\es$ & err.\ rate & $\es$ & s/sample \\
\midrule
DiffusionPDE & 21.20\tstd{6.86} & 0.150\tstd{0.119} & 18.93\tstd{5.04} & 0.133\tstd{0.084} & 9.49\tstd{1.65} & \underline{0.230}\tstd{0.025} & 7.49\tstd{4.60} & 0.311\tstd{0.169} & 274.5 \\
FunDPS & 22.05\tstd{7.11} & 0.169\tstd{0.116} & 16.25\tstd{4.09} & 0.156\tstd{0.077} & 8.18\tstd{1.53} & 0.381\tstd{0.032} & 5.04\tstd{2.71} & 0.152\tstd{0.060} & \underline{14.5} \\
DDIS & 15.76\tstd{5.08} & 0.071\tstd{0.053} & 15.47\tstd{4.34} & 0.070\tstd{0.051} & 8.36\tstd{1.55} & \textbf{0.158}\tstd{0.017} & 7.48\tstd{3.89} & 0.165\tstd{0.089} & 18.0 \\
\midrule
FunBCS (uniform) & 10.51\tstd{3.38} & \underline{0.050}\tstd{0.038} & \underline{9.90}\tstd{2.71} & \textbf{0.056}\tstd{0.034} & 6.23\tstd{1.13} & 0.300\tstd{0.069} & 2.82\tstd{1.56} & 0.113\tstd{0.053} & 16.5 \\
FunBCS-EF & \underline{10.22}\tstd{3.33} & 0.056\tstd{0.049} & 10.04\tstd{2.95} & 0.085\tstd{0.053} & \underline{6.14}\tstd{1.11} & 0.280\tstd{0.059} & \underline{2.69}\tstd{1.50} & \underline{0.086}\tstd{0.041} & \textbf{10.4} \\
FunBCS-AC & \textbf{9.60}\tstd{3.08} & \textbf{0.047}\tstd{0.039} & \textbf{9.06}\tstd{2.51} & \underline{0.069}\tstd{0.037} & \textbf{5.89}\tstd{1.04} & 0.288\tstd{0.051} & \textbf{2.32}\tstd{1.27} & \textbf{0.080}\tstd{0.033} & 16.5 \\
\bottomrule
\end{tabular}}
\vspace{-0.3cm}
\end{table}

\noindent 
\textbf{Main reconstruction results.} Per \cref{tab:main}, our sampler attains the lowest reconstruction error on every task when compared to baselines. Relative to DDIS, the error falls by $39\%$ on Poisson, $41\%$ on Helmholtz, $30\%$ on Navier--Stokes, and $69\%$ on Darcy. Box plots of the \cref{tab:main} errors are visualized in \cref{fig:box-main} in \cref{sec:box_plots}. %
In line with our theory, the largest gain is on Darcy flow, whose piecewise-constant coefficient makes the prior sharply multimodal: the error rate drops from $7.48\%$ to $2.32\%$, a $3.2\times$ reduction, and the visualization in \cref{fig:darcy} shows that DDIS misplaces interfaces whereas our reconstructions recover the correct discontinuity locations and geometries. Qualitative visualizations of the results on all four PDEs are in \cref{fig:qual-ph,fig:qual-nd} in \cref{app:visualizations}. The spectral energy error of our method is better on Poisson, Helmholtz and Darcy, whereas DDIS retains a lower $\es$ on Navier--Stokes despite its higher $\ell_2$ error. This is because $\es$ only checks whether a reconstruction has the right amount of detail at each spatial scale, not whether that detail is in the right place. As we show in \cref{app:ns}, DDIS produces more fine-scale detail than FunBCS, which brings this amount closer to that of the ground truth, but neither method places the fine-scale detail where it appears in the true field.

\noindent 
\textbf{Impact of the proposed scheduling function.} 
At a budget of $M = 800$ steps, concentrating them in the selection window (i.e., FunBCS-AC) lowers the error when compared to the uniform allocation (i.e., FunBCS without the scheduling function in \cref{tab:main}). For example, on Poisson, the error is reduced by $9\%$. Trimming the steps outside the window, i.e., FunBCS-EF, removes $24$--$49\%$ of them, yet FunBCS-EF still outperforms every baseline on every task, running $1.4$--$2.1\times$ faster than DDIS with $27$--$64\%$ lower reconstruction error (\cref{app:sched-ablations}). We show in \cref{app:sched-ddis} that our scheduling function also improves DDIS on its efficiency, though not as much in magnitude as the gain we observed for FunBCS.

\par\addvspace{\smallskipamount}
\noindent
\begin{minipage}[t]{0.58\textwidth}
\vspace{0pt}
\textbf{Scaling with compute.} \Cref{fig:runtime} presents how the Poisson error of every method scales (\cref{fig:runtime-app} in \cref{app:details} presents the results on the other PDEs). At every time budget, FunBCS is more accurate than baselines. For example, FunBCS at $8$ seconds per sample achieves $11.8\%$ relative $\ell_2$ error on Poisson, whereas DDIS, at best, obtains $13.7\%$ error while running for $72$ seconds.
\par\medskip
\noindent
\textbf{FunBCS ability of escaping a wrong mode}. In \cref{sec:analysis} and \cref{sec:method}, we highlighted that DDIS on a discontinuous Poisson inverse problem can fail when the denoiser selects the wrong mode, while FunBCS fixes this problem. We test this mechanism on Darcy flow, where several possible discontinuity locations result in a learned multimodal prior and, as discussed in \cref{sec:analysis} and \cref{prop:finite_diff_assumps} in \cref{proof:proof_of_finite_diff_assumps}, a component of the denoiser decision lies in measurement nullspace with high probability. Results across $50$ test instances are given in \cref{fig:darcy-escape} where we move every discontinuity location of the true field $\bm{a}_+$ by five pixels to obtain a wrong field $\bm{a}_-$, resulting in $11\%$ pixel error. We say an instance \emph{escapes} when its final reconstruction assigns most of the pixels where $\bm{a}_+$ and $\bm{a}_-$ differ to the true field $\bm{a}_+$, with a lower error rate than $\bm{a}_-$. We find that when started at $\sigma = 0.5$, FunBCS escapes $\bm{a}_-$ on 94\% of the instances, whereas DDIS \textit{escapes on none}; doubling the inference budget of DDIS only increases this to 18\%. We visualize this across noise levels $\sigma$ in \cref{fig:darcy-escape}. This is precisely the behavior predicted by \cref{corollary:ddis-fail} for DDIS and \cref{thm:sitcom-success} for FunBCS: if the denoiser is wrong, DDIS can fail, and this is exacerbated towards the end of the sampling process, while in contrast FunBCS can adjust the denoiser's decision and solve this problem. %
\par
\end{minipage}\hfill%
\begin{minipage}[t]{0.38\textwidth}
\vspace{0pt}
\centering
\captionsetup{font=small, skip=3pt}   %
\includegraphics[width=\linewidth]{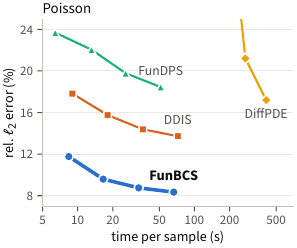}
\captionof{figure}{Poisson reconstruction error with respect to the time per sample. All methods are run using the same hardware. FunBCS and DDIS use the same diffusion prior.}
\label{fig:runtime}
\par\smallskip
\includegraphics[width=\linewidth]{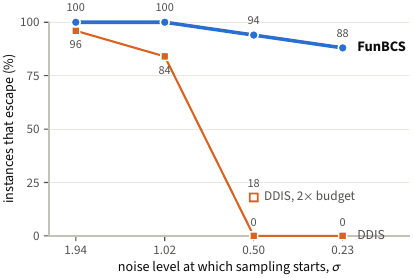}
\captionof{figure}{Share of the $50$ Darcy test instances on which each sampler escapes the wrong field $\bm{a}_-$, against the noise level $\sigma$ at which sampling starts from renoised $\bm{a}_-$. The hollow marker is DDIS with twice its inference budget.}
\label{fig:darcy-escape}
\end{minipage}
\par\addvspace{\smallskipamount}

\section{Conclusion}
\label{sec:conclusion}

We introduce FunBCS, an input space optimization approach for solving sparse PDE inverse problems. We demonstrate that existing methods, such as the state-of-the-art DDIS, perform output space optimization and consequently can fail to obtain the correct coefficient field, especially when the prior is naturally multimodal (e.g., when discontinuities are present in the coefficient field). We demonstrate that FunBCS does not suffer from this issue: across four PDE inverse problems, it significantly reduces the reconstruction error of the state-of-the-art at significantly lower runtime. We discuss limitations and future directions in \cref{sec:lims_and_future_dirs}.

\section*{Acknowledgements}

This material is based upon work supported by the National Science Foundation Graduate
Research Fellowship Program under Grant No. DGE 2241144, which supports M.H.A.  Any opinions, findings, and conclusions or recommendations expressed in this material are those of the author(s) and do not necessarily reflect the views of the National Science Foundation. J.M.M. and I.A. were supported by the U.S. Department of Energy through the Laboratory Directed Research and Development at Los Alamos National Laboratory (LANL) under projects 20260060DR and 20250048DR, respectively. LANL is operated by Triad National Security, LLC, for the National Nuclear Security Administration of U.S. Department of Energy (Contract No. 89233218CNA000001). Q.Q., C.J., and Y.J. acknowledge support from NSF CAREER grant CCF-2143904; NSF grants CCF-2212066, CCF-2212326, and IIS-2402950; ONR grant N000142512339; and Google TPU Research Awards.

\printbibliography

\newpage
\appendix
\crefalias{section}{appendix}
\crefalias{subsection}{appendix}

\newpage 

\section{Notation}
\label{sec:notation}

We use $a(\cdot)$ to denote a function; $\bm{a}$ to denote a vector, which may be a discretization of function $a(\cdot)$ at collocation points $\bm{x}_1, ..., \bm{x}_n$; and $a$ to denote a scalar. We use $\mathcal{A}$ to denote an operator, $\mathcal{A}_h$ to denote a discretized operator at  points $\bm{x}_1, ..., \bm{x}_n$; and $\bm{A}$ to denote a linear operator (i.e., a matrix). We use $\pi$ to denote a probability distribution and $p$ or $q$ to denote its density, if it exists. $\nabla_{\bm{x}} f(\bm{x})$ is taken to be either the gradient or Jacobian of $f$, depending on the codomain dimension of $f$. We use $\mathds{1}_A$ to denote the indicator on event $A$. We use $\bm{a}_{t_k}$ to denote a sample drawn from the VE diffusion reverse process at timestep $t_k$, per \cref{eq:ve-reverse-sde}, and 
$\bm{a}_0$ to denote a sample drawn from the end of this process; $\hat{\bm{a}}_0^{(t_k)} = D_{\bm{\theta}}(\bm{a}_{t_k}, t_k)$ to denote the denoiser output given sample $\bm{a}_{t_k}$; $\bm{a}_0^{(t_k)}$ to denote the reconstruction of the signal $\bm{a}$ obtained from a generic diffusion inverse solver at time $t_k$; and $\tilde{\bm{a}}_0^{(t_k)}$ to denote the reconstruction of the signal $\bm{a}$ obtained from DDIS specifically. We use $\bm{u}_{\mathrm{obs}}$ systematically to denote the measurements. $\pi \Rightarrow \pi'$ represents convergence in distribution of $\pi$ to $\pi'$. We denote by $L^2(\Omega)$ the space of square (Lebesgue) integrable functions on domain $\Omega$.

\section{Function Space Diffusion Setup}
\label{sec:function_space_diffusion}

Let $\bm{K}$ be a kernel matrix, with entries
$\bm{K}_{ij}=k(\bm{x}_i, \bm{x}_j)$, where $k : \mathbb{R}^d \times \mathbb{R}^d \to \mathbb{R}$ is a trace-class, positive definite kernel. For example, $k(\cdot, \cdot)$ can be the radial basis function (RBF) kernel \[k(\bm{x}_i, \bm{x}_j) = \exp(-||\bm{x}_i - \bm{x}_j||_2^2 / \ell)\] for a fixed length scale $\ell>0$. We assume $K$ is positive definite and write $K^{1/2}$ for its symmetric positive definite square root. Let $\sigma_t>\sigma_0 = 0$ for $t>0$ be a diffusion noise schedule, and assume that $t\mapsto\sigma_t^2$ is continuously differentiable and nondecreasing. We write $p_t$ for the density of $\bm{a}_t$ so that $\bm{a}_t | \bm{a}_0 \sim \mathcal{N}(\bm{a}_0, \sigma_t^2 \bm{K})$ and consider the variance exploding (VE) diffusion forward process at $0 < t_0<\cdots<t_N=T$: 
\begin{equation}
 \bm{a}_{t_{k+1}}=\bm{a}_{t_k}
 +\sqrt{\sigma_{t_{k+1}}^2-\sigma_{t_k}^2}\,\bm{\varepsilon}_k\;,
 \qquad k=0,\ldots,N-1\;,
 \label{eq:ve-forward-discrete_appx}
\end{equation}
where the $\bm{\varepsilon}_k\sim\mathcal N(0,K)$ are independent across $k$.

With $g(t)=\sqrt{\frac{d}{dt}\sigma_t^2}$, the continuous-time forward process is given by the VE SDE: 
\begin{equation}
 d\bm{a}_t=g(t)K^{1/2}d\bm{W}_t\;,\qquad 0\le t\le T\;,
 \label{eq:ve-forward-sde}
\end{equation}
where $W_t$ is an $n$-dimensional Brownian motion. With $t$ decreasing from $T$ to $0$, the reverse VE SDE is given by: 
\begin{equation}
 d\bm{a}_t
 =-g(t)^2\bm{K}\nabla_{\bm{a}_t} \log p_t(\bm{a}_t)\,dt
   +g(t)\bm{K}^{1/2}\,d\overline{\bm{W}}_t\;,
 \qquad 0\le t\le T\;,
 \label{eq:ve-reverse-sde}
\end{equation}
where $\overline{\bm{W}}_t$ is $n$-dimensional Brownian motion in reverse time. To obtain the appropriate reverse diffusion process, one learns $\bm{s}_{\bm{\theta}}(\bm{a}_t,t)\approx\nabla_{\bm{a}_t} \log p_t(\bm{a}_t)$ with score matching and then discretizes \cref{eq:ve-reverse-sde} at times $t_k$ via, e.g., Euler-Maruyama \citep{platen1992numsde}. Additionally, by the function-space analogue of Tweedie's formula \citep{lim2023scorefs}, which after discretization yields the standard multivariate Tweedie's formula  \citep{efron2011tweedie}, we have:
\begin{equation}
     D^*(\bm{a}_t,t) := \mathbb{E}[\bm{a}_0 | \bm{a}_t] =\bm{a}_t+\sigma_t^2 \bm{K} \nabla_{\bm{a}_t} \log p_t(\bm{a}_t)\;.
    \label{eq:tweedie-denoiser-appx}
\end{equation}
We note that this does not require $\bm{a}_0$ to have a density, making it compatible with our prior in \cref{sec:analysis}; under this prior, a density is admitted for all numerical times $t_0 < \dots < t_{N}$, yielding that \cref{eq:tweedie-denoiser-appx} is well-defined. For a learned score $\bm{s}_{\bm{\theta}}$, we denote the denoiser by $D_{\bm{\theta}}$.

\section{Extended Related Work}
\label{sec:related_work}

\subsection{Classical Inverse Solvers}

The partial differential equation (PDE) inverse problem in \cref{sec:prelim} is prevalent across many subfields of science and engineering, including geophysics \citep{virieux2010overview}, astrophysics \citep{PengKNML,GWParameterEsitmationReview}, fluid dynamics \citep{bui2013computational}, imaging \citep{arridge1999optical}, and materials \citep{avril2008overview}. Thus, many methods have been developed to solve PDE inverse problems \citep{belov2012inverse}. 

One key approach involves specifying different structured priors, such as Besov \citep{dashti2011besovpriorsbayesianinverse} or level set \citep{iglesias2016bayesian} priors, and sampling from the posterior with an MCMC method such as the preconditioned Crank-Nicolson method \citep{cotter2013mcmc}. While these methods often come with formal guarantees, they can take prohibitively long to sample from the posterior, especially under multimodality or in high dimensions \citep{guan2007small,lee2018beyond}. 

Because of the expense of these methods, exact forward models are rarely used in the inner loop, and rather domain-specific emulators are used instead. These forward models are often highly tuned and problem-specific, limiting generalization. For example, the gravitational wave community replaces its expensive numerical relativity simulations with hand-parametrized effective models, and doing so has been a decades-long endeavor \citep{tiglio2022reduced}.

Recently, multi-fidelity techniques have been leveraged to build emulators in these data-sparse environments \citep{gorodetsky2020mfnets}. Other methods include optimization-based samplers, such as randomize-then-optimize methods \citep{bardsley2014randomize}, which repeatedly compute the maximum a posteriori estimate under different random perturbations of the observations and prior, as well as ensemble approaches based on the Kalman filter \citep{iglesias2013ensemble}. These methods, however, fail to sample from the correct posterior under nonlinear forward operators and complicated priors \citep{ernst2015analysis}. 

Recently a concerted effort has been made in many communities, such as plasma physics \citep{citrin2024torax}, to make differentiable forward models that may be placed into optimization loops \citep{sapienza2024differentiable}.
We note that many other methods have been developed for the inverse problem in \cref{sec:prelim} without taking the Bayesian perspective, such as compressed sensing \citep{alberti2021infinite}, but these only produce point estimates, thus falling outside the scope of our work.

\subsection{Physics-Informed Learning}

Physics-informed learning has become a widely used approach for accelerating scientific computation \citep{karniadakis2021physics} across diverse fields such as climate modeling \citep{kashinath2021physics} and plasma physics \citep{qin2023data}. Below, we discuss key approaches towards solving forward and inverse problems. 

\subsubsection{Forward Solvers}

Representative approaches to solving the forward problem include physics-informed neural networks (PINNs) \citep{raissi2019physics}, which encode the governing PDE into the training objective, and neural operators \citep{li2021fno}, which learn mappings between function spaces and can amortize repeated PDE solves. For modeling the forward operator, neural ODEs \citep{chen2018neural} are often leveraged to replace a component of a solver with a neural network. Deterministic formulations of these methods return a single reconstruction rather than samples across the posterior. Probabilistic variants have been developed \citep{park2026generative}, but require problem-specific training and additional modeling assumptions. We study methods which can be used across several measurement operators without retraining. Thus, in our setting, we primarily use neural operators as differentiable surrogates for the forward PDE operator and as the denoiser for function space diffusion in our experiments. Separately, several approaches have been presented for physics-informed \citep{bastek2025physicsinformed} and physics-constrained \citep{utkarsh2025physicsconstrained} generation; these approaches are complementary to, but distinct from, solving PDE inverse problems with diffusion. 

\subsubsection{Priors}

Diffusion models have seen enormous success in modeling complicated, high-dimensional priors \citep{rombach2022high}. Thus, several recent works have constructed approaches to function space diffusion priors. \citet{kerrigan2023diffusion} develop diffusion models on possibly infinite-dimensional Hilbert spaces using Gaussian measures, and \citet{pidstrigach2023infinite} similarly formulate diffusion models in infinite dimensions and study well-posedness and approximation bounds. \citet{lim2023scorefs} develop Denoising Diffusion Operators (DDO), where neural operator architectures parameterize the learned score during score matching. These provide a concrete mathematical and algorithmic basis for using diffusion models to model function space priors for use in diffusion-based PDE inverse solvers. 

\subsubsection{Inverse Solvers}

Due to their ability to model complicated priors, diffusion models have been adapted to PDE inverse problems. In our work, we focus on developing diffusion solvers which are agnostic to the choice of measurement operator; that is, one learned prior can be used with many different measurement operators during inference. DiffusionPDE \citep{huang2024diffusionpde} replaces the diffusion likelihood $\pi(\bm{u}_{\mathrm{obs}} | \bm{a}_t)$ with a Gaussian centered at the denoiser estimate $\hat{\bm{a}}_0^{(t)}$, and samples from the joint distribution $(\bm{a}, \bm{u})$. Notably, DiffusionPDE trains a standard diffusion prior at fixed discretization; instead, FunDPS \citep{yao2025fundps} trains a function space prior. More recent work, such as the state-of-the-art DDIS \citep{lin2026ddis}, has decoupled the physical parameter and solution fields $(\bm{a}, \bm{u})$, specifically approximating $\bm{u}$ with a neural operator surrogate while performing prior learning and posterior sampling on $\bm{a}$. In our work, we also decouple the learning of $\bm{a}$ and $\bm{u}$. DDIS also relaxes the assumptions of DiffusionPDE and FunDPS; instead of approximating the diffusion likelihood with a Gaussian, DDIS approximates $\pi(\bm{a}_0 | \bm{a}_t)$ with an isotropic Gaussian, as discussed in \cref{sec:analysis}. Another approach, RED-DiffEq \citep{shan2026regularization}, performs regularized optimization on the measurement consistency objective, but uses diffusion to encode the regularization; however, this approach is not Bayesian and has not been compared with the strong diffusion posterior samplers discussed above. The concurrent work DiLO \citep{liu2026dilo} adopts the same decoupling as DDIS, but instead discusses the issue of evaluating the neural operator surrogate on noisy diffusion states $\bm{a}_t$ instead of clean reconstructions, which the operator was originally trained on.  We note that their solution for this issue can easily be incorporated into our input optimization procedure introduced in \cref{sec:method}, and the failure modes of the neural operator surrogate are not the focus of our work. Additionally, the concurrent work FAPS \citep{shi2026flowannealingposteriorsampling} is similar to DDIS, instead using operator flow matching \citep{shi2025stochastic} to encode the prior and approximating the denoising distribution with an anisotropic Gaussian. 

To distinguish these diffusion PDE inverse solvers, consider two key operations that occur within a reverse sampling step $t_k$ with noisy $\bm{a}_{t_k}$: (\textbf{D}) computing an estimate of the signal $a(\cdot)$ as $\hat{\bm{a}}_{0}^{(t_k)} = D_{\bm{\theta}}(\bm{a}_{t_k}, t_k)$, where $D_{\bm{\theta}}$ is the diffusion denoiser and (\textbf{M}) enforcing measurement consistency on a reconstruction $\bm{a}_0^{(t_k)} \approx \hat{\bm{a}}_0^{(t_k)}$, so that it is consistent with the observations (i.e., $\mathcal{A}(\bm{a}_0^{(t_k)}) \approx \bm{u}_{\mathrm{obs}}$). We use the term \textit{output space optimization} to describe methods where (\textbf{D}) and (\textbf{M}) are \textit{disjoint}, since this amounts to optimizing in the output space of the denoiser, while leaving the denoiser input fixed. From this perspective, DiffusionPDE \citep{huang2024diffusionpde}, FunDPS \citep{yao2025fundps}, and Decoupled Diffusion Inverse Solver (DDIS) \citep{lin2026ddis} are all output space optimization procedures. We discuss this further in \cref{sec:recurrences}.

In \cref{sec:analysis}, we demonstrate that output space approaches can fail to update unobserved components of the denoiser estimate appropriately for a PDE inverse problem. We thus use this to motivate \textit{input space optimization} in \cref{sec:method}, where steps (\textbf{D}) and (\textbf{M}) are \textit{combined}: we iteratively compute the denoiser output (\textbf{D}) and optimize the \textit{input} such that measurement consistency is satisfied (\textbf{M}). We show that this updates the unobserved components as desired. This results in \textit{backward-consistent} sampling \citep{alkhouri2025sitcom}, where the reconstruction is constrained to be in the denoiser's range. We note that our approach is different, motivated by the setting of PDE inverse problems where priors can be sharply multimodal and with a new approach to allocate input optimization steps motivated by our theory. Additionally, the concurrent work FAPS \citep{shi2026flowannealingposteriorsampling} uses its anisotropic Gaussian covariance to precondition the measurement gradient. We note that input space optimization can be viewed as \textit{preconditioned} output space optimization, as discussed in \cref{sec:preconditioning}. Note that our preconditioner differs from FAPS, and furthermore is grounded in the theory of PDE inverse problems and does not require additional estimation from training data after the diffusion model has been trained.

The computer vision literature also contains several approaches to address the unobserved component of the reconstruction while solving the inverse problem. Most notably, DDNM \citep{wang2023zeroshot} assumes a linear forward operator and does diffusion only in nullspace; NullFlow \citep{shi2026nullflow} extends this approach to one-step sampling with flow matching. Note that these approaches do not solve the issue that DDIS suffers from discussed in \cref{prop:ddis_fail_prelim}. In particular, the nullspace component of the reconstruction is still the same as that of the denoiser. Separately, Posterior-DDIM \citep{jiao2026posteriorddim} performs ODE-based diffusion sampling in each singular direction separately, incorporating measurement consistency in a singular direction only when the diffusion signal-to-noise ratio is small. Importantly, this approach assumes a linear forward operator so that the singular value decomposition is easily computable. However, PDE inverse problems can have highly nonlinear forward operators, e.g., Darcy flow as considered in our experiments in \cref{sec:exp}. One can obtain a local approximation with the Jacobian, but this still comes at a high computational expense and may not have the same desirable properties. 

Finally, note that there have been several recent approaches developed to use diffusion for data assimilation \citep{andrae2026daisi,jia2026forcingdas}. While data assimilation is closely related to Bayesian inverse problems, it considers a sequential setting; rather than recovering a fixed coefficient field $\bm{a}$ from measurements $\bm{u}_{\mathrm{obs}}$, data assimilation repeatedly propagates a state $\bm{a}_0^{(\ell)}$ and assimilates new observations $\bm{u}_{\mathrm{obs}}^{(\ell)}$. Accordingly, recent generative data assimilation work focuses on combining learned generative priors with state transition operators. These approaches are complementary to our work. Furthermore, moment-matching posterior samplers like DPS \citep{chung2023dps} or MMPS \citep{rozet2024learning} have been adopted for use in data assimilation \citep{chen2026flowdas,savary2026trainingfree}; backward-consistent sampling can be adopted as well, and we leave a comprehensive study of this to future work.

\subsection{Theoretical Guarantees}

Classical inverse solver theory provides  guarantees for recovering an unknown signal from incomplete observations. For deterministic inverse problems, where one recovers a point estimate, regularization theory extensively studies the existence, stability, and convergence of solutions, especially with linear forward operators \citep{engl1996regularization}. Bayesian inverse solver theory studies when the posterior is well-defined and when the problem is well-posed and stable \citep{stuart2010inverse}; theory for MCMC methods is likewise used to study MCMC-based solvers for these problems \citep{hairer2014spectral}, including using diffusion priors \citep{wu2024principled}. Due to the prevalence of these diffusion priors in recent inverse solvers, recent works likewise have developed theory for diffusion inverse solvers. For Gaussian priors, approximating the denoising distribution exactly by first- and second-moment Tweedie's formulas can recover the posterior \citep{boys2024tweedie}. Other recent work obtains insights into the bias and stability of diffusion posterior samplers \citep{delgadino2026feynmankacdps}; the robustness of these samplers to likelihood misspecification and outliers \citep{yang2026outlierrobustdiffusionposteriorsampling}; connections to projected gradient descent and its convergence theory \citep{leong26recovery}; convergence to the posterior \citep{jiao2026posteriorddim}; and how this relates to maximum a posteriori estimation \citep{xu2025rethinking}. Notably, \citet{burns2026whenwhydiffusionposterior} study multimodal priors; however, they focus on a different failure mode of diffusion posterior samplers, specifically issues with likelihood approximations such as those made by DPS. Additionally, several works have used multimodal Gaussian mixture priors to empirically study the behaviors of diffusion posterior samplers \citep{crafts2025benchmarking,yismaw2025gaussian}; however, these works do not study the PDE setting or the unobserved components of the denoiser. Much of the literature studies the finite-dimensional image inverse problem setting; for the function space setting present in our PDE inverse problems, FunDPS \citep{yao2025fundps} develops an analogue of Tweedie's formula while DDIS \citep{lin2026ddis} obtains theoretical results characterizing why learning the coupled prior $(\bm{a}, \bm{u})$ is problematic, motivating the learning of the prior over only $\bm{a}$ done in both their and our algorithms. More broadly, there has been significant work on how diffusion models select modes under multimodal priors \citep{wu2024theoretical,ashiq2026why}. Our work adds to this literature in the PDE inverse problem setting; we study when and why diffusion solvers based on output space optimization fail in the PDE setting, where multimodal priors naturally arise, and instantiate a solution based on input space optimization, validating our results empirically. We also study what unique conditions on the measurements and nullspace components of different prior modes arise due to the PDE setting in \cref{prop:finite_diff_assumps} in \cref{proof:proof_of_finite_diff_assumps} and connect our results to preconditioning in \cref{sec:preconditioning}.

The effectiveness of conditioning in diffusion models depends on when and how guidance is applied throughout the reverse process; while classifier and classifier-free guidance are frequently implemented using a constant guidance scale, several works show that this is suboptimal across diffusion noise levels \citep{papalampidi2026dynamic}. In the setting of Bayesian inverse solvers based on diffusion models, there has been work done on designing schedules for likelihood or measurement guidance across the diffusion sampling process. For example, Jacobian-aware posterior sampling \citep{hen2026jacobianaware} adapts the step size of the likelihood gradient based on the denoiser Jacobian, while the recent work TriPS \citep{bang2026triadic} adjusts this step size for conditional priors based on the classifier-free guidance schedule and added stochasticity in the sampling process. Additionally, \citet{delgadino2026feynmankacdps} show that DPS guidance can become numerically unstable towards the end of the reverse process, providing a principled explanation for early stopping the DPS measurement-consistency guidance in the sampling process. We add to this body of work, designing a principled approach to scheduling input optimization steps across noise levels based on our theoretical analysis of an inverse problem on Poisson's equation. Notably, we control a different quantity than, e.g., gradient step size or guidance weight; instead, we control the \textit{computational effort} used to enforce measurement consistency at each diffusion noise level, done so by a function that automatically allocates more steps on critical noise levels during inference.

\section{Limitations \& Future Directions}
\label{sec:lims_and_future_dirs}

Limitations include that FunBCS is reliant on the accuracy of the neural operator surrogate $\mathcal{G}_{h, \bm{\phi}}$ and the diffusion prior $\pi(\bm{a})$. Additionally, our current input optimization scheduler requires task-specific calibration. Future directions include addressing these limitations, improving our method further for discontinuous inverse problems by, e.g., using discontinuous bases \citep{iglesias2016bayesian}, and extending input space optimization to PDE data assimilation \citep{chen2026flowdas}.  

\section{Overview of Recent Diffusion PDE Inverse Solvers}
\label{sec:recurrences}

We provide an overview of the recurrences of diffusion PDE inverse solvers, to make clear how existing methods perform output space optimization (with separate (\textbf{D}) and (\textbf{M}) steps) while our method performs input optimization (combined (\textbf{D}) and (\textbf{M}) steps). We present these all for fixed diffusion times $t_k > t_{k - 1}$. 

\textbf{DiffusionPDE/FunDPS recurrence}. DiffusionPDE \citep{huang2024diffusionpde} applies diffusion posterior sampling \citep{chung2023dps} with a standard diffusion prior at a fixed discretization, and FunDPS \citep{yao2025fundps} extends it to function space priors. Both use the following recurrence: 

\begin{equation}
    \bm{a}_{t_k - 1}^{(P)} = \tilde{\bm{a}}_{t_k - 1}^{(P)} - \zeta_k \nabla_{\bm{a}_{t_k}^{(P)}} \|\mathcal{A}_h\hat{\bm{a}}_0^{(P, t_k)} - \bm{u}_{\mathrm{obs}}\|_2^2\;,
    \label{eq:DPS-update}
\end{equation}
where: 
\begin{equation}
    \tilde{\bm{a}}_{t_k - 1}^{(P)} = \bm{a}_{t_k}^{(P)} + \Delta_{t_k}\bm{s}_{\bm{\theta}}(\bm{a}_{t_k}^{(P)}, t_k) + \sqrt{\Delta_{t_k}} \bm{\varepsilon}_k^{(P)}\;, 
    \label{eq:DPS-prelim}
\end{equation}
and $\Delta_{t_k} = t_k - t_{k - 1}$ with $\hat{\bm{a}}_0^{(P, t_k)} := D_{\bm{\theta}}(\bm{a}_{t_k}^{(P)}, t_k)$, and $\bm{\varepsilon}_k^{(P)} \sim \mathcal{N}(\bm{0}, \bm{I})$ for DiffusionPDE and $\mathcal{N}(\bm{0}, \bm{K})$ for FunDPS, independent across $k$.  

\begin{remark}
    At a timestep $t_k$, one first computes $\hat{\bm{a}}_0^{(P, t_k)}$ (i.e., \textbf{(D)}); this can then be used in the measurement consistency term in \cref{eq:DPS-update} as well as to estimate $\bm{s}_{\bm{\theta}}$ by Tweedie's formula per \cref{eq:tweedie-denoiser}, used in \cref{eq:DPS-prelim}. DPS then performs a renoising-type step in \cref{eq:DPS-prelim}, and subsequently enforces measurement consistency (i.e., performs (\textbf{M}) in \cref{eq:DPS-update}). We note that these two steps are separate at diffusion timestep $t_k$, so that measurement consistency is enforced by optimizing over the output space of the denoiser. That is, methods such as DiffusionPDE and FunDPS perform output space optimization. \end{remark}

\textbf{DDIS recurrence}. As presented in \cref{eq:DDIS-update}, the DDIS recurrence is given by:

\begin{equation}
    \bm{a}_{t_{k-1}}^{(D)}=\tilde{\bm{a}}_0^{(t_k)}+\sigma_{t_{k-1}}\bm{\varepsilon}_k^{(D)}, \qquad \tilde{\bm{a}}_0^{(t_k)} \sim q_{D,t_k}(\cdot | \bm{a}_{t_k}^{(D)}, \bm{u}_{\mathrm{obs}}), \; \bm{\varepsilon}_k^{(D)} \sim \mathcal{N}(\bm{0}, \bm{K})\;,
    \label{eq:DDIS-update_appx}
\end{equation}
with 
\begin{equation}
    q_{\mathrm D,t_k}(\bm{a}_0\mid \bm{a}_{t_k}^{(D)},\bm{u}_{\mathrm{obs}})
    \propto\exp\!\left\{
       -\frac{\|\mathcal{A}_h \bm{a}_0-\bm{u}_{\mathrm{obs}}\|^2}{2\gamma^2}
       -\frac{\|\bm{a}_0-\hat{\bm{a}}_0^{(D, t_k)}\|^2}{2\sigma_{t_k}^2}
    \right\}\;.
    \label{eq:ddis-correction_appx}
\end{equation}
where $\hat{\bm{a}}_0^{(D, t_k)} := D_{\bm{\theta}}(\bm{a}_{t_k}^{(D)}, t_k)$ and $\varepsilon_{k}^{(D)}$ are drawn independently.

\begin{remark}
    At a timestep $t_k$, one first computes $\hat{\bm{a}}_0^{(D, t_k)}$ (i.e., \textbf{(D)}). DDIS then incorporates this into \cref{eq:ddis-correction_appx} to sample from $q_{D, t_k}$; from the perspective of Langevin dynamics, this is precisely optimizing in the output space of the denoiser by enforcing measurement consistency (\textbf{M}) with the first term and ensuring the reconstruction does not deviate too far from the denoiser estimate per the second term. Subsequently, DDIS renoises up to timestep $t_{k - 1}$ per \cref{eq:DDIS-update_appx}. Computing the denoiser output and enforcing measurement consistency are clearly separated. Thus, DDIS performs output space optimization. 
    \end{remark}

\textbf{FunBCS recurrence}. As presented in \cref{eq:sitcom-update-rule}, the FunBCS recurrence is given by: 

\begin{equation}
    \bm{a}_{t_{k-1}}^{(B)}=D_{\bm{\theta}}(\bm{v}(s),t_k)+\sigma_{t_{k-1}}\bm{\varepsilon}_k^{(B)}\;,
      \qquad \bm{\varepsilon}_k^{(B)}\sim\mathcal N(\bm{0},\bm{K})\;,
      \label{eq:sitcom-update-rule_appx}
\end{equation}
where $\bm{v}(s)$ is a solution to the gradient flow: 
\begin{equation}
    \frac{d\bm{v}(s)}{ds}
    =-\left.\nabla_{\bm{v}}
       \frac{\|\mathcal{O}_h \circ \mathcal{G}_{h, \phi} \circ D_{\bm{\theta}}(\bm{v}, t_k)-\bm{u}_{\mathrm{obs}}\|^2}
            {2\gamma^2}\right|_{\bm{v} = \bm{v}(s)}\;,
    \qquad \bm{v}(0)=a_{t_k}^{(B)},
    \label{eq:sitcom-input-flow_appx}
\end{equation}
and $\bm{\varepsilon}_k^{(B)}$ are drawn independently. 

\begin{remark}
    At a fixed diffusion timestep $t_k$, per \cref{eq:sitcom-input-flow_appx}, for an optimization step $s_i$, we alternate between computing the denoiser output $D_{\bm{\theta}}(\bm{v}(s_i), t_k)$ and finding the input $\bm{v}(s_i + \Delta_s)$ so that measurement consistency is satisfied. This thus optimizes in the \textit{input} space of the denoiser and couples steps (\textbf{D}) and (\textbf{M}). Hence, our method FunBCS is an input space optimization procedure. 
    \end{remark}

\section{Experimental Details}
\label{app:details}

\paragraph{Tasks.} We use the four PDE inverse problems of DDIS \citep{lin2026ddis}, whose data were generated by DiffusionPDE \citep{huang2024diffusionpde} with the solvers of FNO \citep{li2021fno}; we use the normalized version provided with DDIS's public codebase. Every problem is posed on the unit square $\Omega = (0, 1)^2$, discretized on a $128 \times 128$ grid. In each problem, we recover the coefficient field $a$ from $500$ noiseless observations of the solution $u$ at random grid points, about $3\%$ of the grid. The observation locations are drawn anew for each test sample and are the same for FunBCS, DDIS and FunDPS, whereas DiffusionPDE's public code uses one fixed set of locations for every sample. Below, $\nabla^2$ denotes the Laplacian, and a Gaussian random field with parameters $(\alpha, \tau)$ has covariance proportional to $(-\nabla^2 + \tau^2 I)^{-\alpha}$. The coefficient fields of Poisson, Helmholtz and Darcy flow are all built from such a field with $(\alpha, \tau) = (2, 3)$ and zero Neumann boundary conditions.

\textit{Poisson.} The source term $a$ is the Gaussian random field above, and
\begin{equation*}
    \nabla^2 u(\bm{x}) = a(\bm{x}), \quad \bm{x} \in \Omega, \qquad u(\bm{x}) = 0, \quad \bm{x} \in \partial\Omega.
\end{equation*}

\textit{Helmholtz.} The source term $a$ is drawn as for Poisson and set to zero on $\partial\Omega$, and with wavenumber $k = 1$,
\begin{equation*}
    \nabla^2 u(\bm{x}) + k^2 u(\bm{x}) = a(\bm{x}), \quad \bm{x} \in \Omega, \qquad u(\bm{x}) = 0, \quad \bm{x} \in \partial\Omega.
\end{equation*}

\textit{Navier--Stokes.} The vorticity $w = \nabla \times \bm{v}$ of an incompressible velocity field $\bm{v}$ on the periodic domain $\Omega$ evolves as
\begin{equation*}
    \partial_t w + \bm{v} \cdot \nabla w = \nu \nabla^2 w + f, \qquad \nabla \cdot \bm{v} = 0, \qquad \bm{x} \in \Omega, \; t \in (0, T],
\end{equation*}
with viscosity $\nu = 10^{-3}$ (Reynolds number $1{,}000$) and forcing $f(\bm{x}) = 0.1\big(\sin(2\pi(x_1 + x_2)) + \cos(2\pi(x_1 + x_2))\big)$. The coefficient field is the initial vorticity $a = w(\cdot, 0)$, a periodic Gaussian random field with $(\alpha, \tau) = (2.5, 7)$, and the solution is the vorticity $u = w(\cdot, T)$ at $T = 1$.

\textit{Darcy flow.} The permeability $a$ equals $12$ where a Gaussian random field drawn as for Poisson is nonnegative and $3$ elsewhere, so it is piecewise constant, and
\begin{equation*}
    -\nabla \cdot \big(a(\bm{x}) \nabla u(\bm{x})\big) = 1, \quad \bm{x} \in \Omega, \qquad u(\bm{x}) = 0, \quad \bm{x} \in \partial\Omega.
\end{equation*}

DDIS's public codebase supports Darcy flow, but DDIS does not report it, so we include it because its piecewise-constant coefficient makes the prior multimodal.

\paragraph{Evaluation.} All methods are evaluated on the same $100$ test samples. On Poisson, Helmholtz and Navier--Stokes, we report the relative $\ell_2$ error $\|\hat{a} - a\|_2 / \|a\|_2$ in percent. On Darcy flow, we threshold the reconstruction at $7.5$, the midpoint of the two coefficient values, and report the percentage of misclassified pixels \citep{huang2024diffusionpde}. The spectral energy error, from \cite{lin2026ddis}, $\es$ compares the radially averaged power spectra $\hat{P}(k)$ and $P(k)$ of the reconstruction and of the ground truth. We take the geometric mean of $|\hat{P}(k) - P(k)| / P(k)$ over wavenumbers $k$, and then average over the test samples. 

\paragraph{Models and samplers.} FunBCS uses the same pretrained diffusion prior, neural operator surrogate and noise schedule as DDIS \cite{lin2026ddis}. DDIS uses the settings of DDIS's public codebase: $100$ noise levels, $20$ Langevin steps per level, and the likelihood weights recommended in the paper. For Darcy flow, which DDIS does not report, we tuned and chose the weight $30$ from $\{5, 10, 15, 20, 30\}$. FunBCS uses $N = 10$ noise levels and optimizes the denoiser input with Adam at learning rate $\gamma = 0.03$. Its scheduling function (\cref{app:schedule}) uses the selection windows $(\sigma_{\mathrm{low}}, \sigma_{\mathrm{high}}) = (0, 0.711)$ on Poisson, $(0, 0.335)$ on Helmholtz, $(0, 0.144)$ on Navier--Stokes and $(0.335, 2.6)$ on Darcy flow. FunBCS-AC uses $M = NK = 800$ optimization steps per sample, and FunBCS-EF uses $M^{\star} = 607$, $506$, $406$ and $488$ steps on the four tasks, respectively. We note that these $M^{\star}$ budgets are not hardcoded, but instead dynamically determined by our method during inference. 

\paragraph{Baselines.} DiffusionPDE \citep{huang2024diffusionpde} trains a diffusion model on the joint distribution of the coefficient and solution fields at a fixed resolution, and guides its sampling with the observation misfit and the PDE residual, evaluated on the denoiser estimate. FunDPS \citep{yao2025fundps} follows the same approach with a joint diffusion prior in function space built from neural operators. DDIS \citep{lin2026ddis} learns a diffusion prior on the coefficient field only, predicts the solution with a neural operator surrogate, and enforces the observations with annealed Langevin sampling in the output space of the denoiser. We run each baseline with its authors' public code and default settings. For the scaling curves of \cref{fig:runtime,fig:runtime-app}, we run every method at $0.5\times$, $1\times$, $2\times$ and $4\times$ its default budget and change nothing else: FunBCS with $400$, $800$, $1{,}600$ and $3{,}200$ optimization steps, distributed over its $N = 10$ noise levels by the scheduling function; DDIS with $50$, $100$, $200$ and $400$ noise levels of $20$ Langevin steps each, that is, $1{,}000$ to $8{,}000$ Langevin steps; and FunDPS with $250$, $500$, $1{,}000$ and $2{,}000$ sampling steps. DiffusionPDE, which is much slower, is run with $500$, $2{,}000$ and $3{,}000$ sampling steps, where $2{,}000$ is its default.

\paragraph{Choosing $N$ and $K$.} We fix the max budget of FunBCS at $NK = 800$ optimization steps. In order to determine the optimal $N$, assuming a uniform allocation of $K$ inner steps per outer level, we had to choose how to split between the number of noise levels $N$ and the number of steps $K$ per level. \Cref{fig:convergence} compares three splits on Poisson and Darcy flow. These runs place the same $K$ steps on every noise level and do not use the scheduling function at all, so that the choice of $N$ and $K$ is not confounded with how the steps are allocated. The split $N = 10$, $K = 80$ reaches the lowest error on both tasks, and we use it for every task. The scheduling function then redistributes these $800$ steps across the $N = 10$ noise levels (\cref{app:schedule}).

\begin{figure}[t]
\centering
\includegraphics[width=\linewidth]{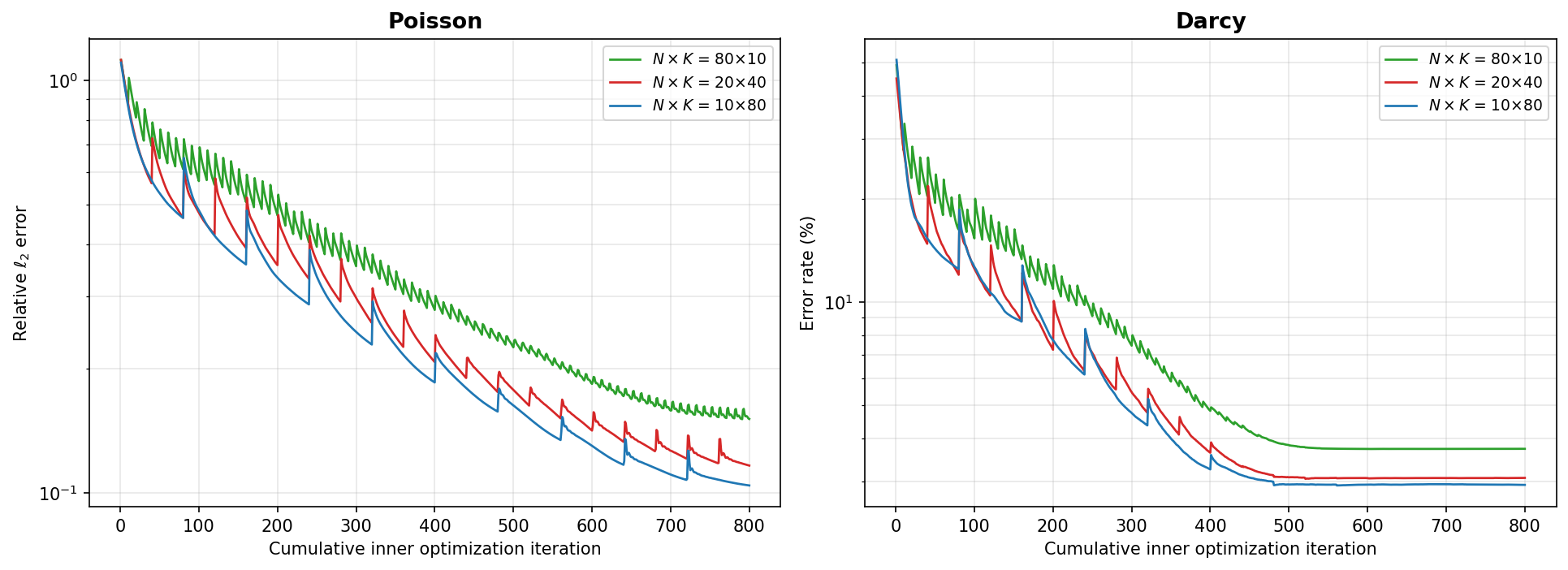}
\caption{Reconstruction error of FunBCS during sampling for three splits $N \times K$ of $800$ optimization steps, with the same $K$ at every noise level and without the scheduling function, on $50$ samples: relative $\ell_2$ error on Poisson (left) and pixel error rate on Darcy flow (right). The error jumps up at the end of each noise level, when the sample is renoised.}
\label{fig:convergence}
\end{figure}

\paragraph{Timing.} We measure seconds per sample on a single NVIDIA A100 GPU with one process per GPU, at batch size $25$ for FunBCS, DDIS and FunDPS and at batch size $1$ for DiffusionPDE. DiffusionPDE's public code reconstructs one test sample at a time, with a single ground truth, observation mask and guidance loss per run, so we time it unmodified at batch size $1$. The time covers the whole sampling process, including all denoiser and surrogate evaluations. Each optimization step of FunBCS differentiates through both the denoiser and the surrogate, so it costs more than a Langevin step of DDIS, which differentiates through the surrogate only. FunBCS, however, uses $10$ noise levels instead of $100$, which DDIS uses.

\paragraph{Scaling with compute.} \Cref{fig:runtime-app} shows the error of every method against its time per sample on Helmholtz, Navier--Stokes and Darcy, as \cref{fig:runtime} does for Poisson. As on Poisson, FunBCS is more accurate than every baseline at every budget, and at about $8$ seconds per sample it is more accurate than every baseline at any of the budgets we tested.

\begin{figure}[t]
\centering
\includegraphics[width=0.33\textwidth]{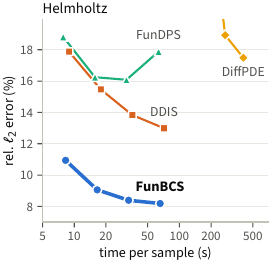}\hfill
\includegraphics[width=0.33\textwidth]{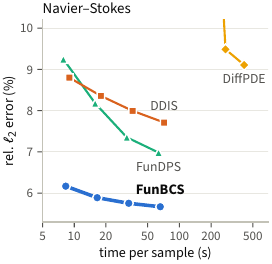}\hfill
\includegraphics[width=0.33\textwidth]{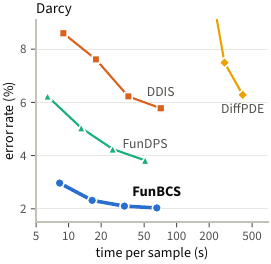}
\caption{Scaling behavior of all methods with increasing inference time per sample on Helmholtz, Navier--Stokes and Darcy, as in \cref{fig:runtime} for Poisson. Darcy reports the pixel error rate. All methods are run using the same hardware.}
\label{fig:runtime-app}
\end{figure}

\section{The Proposed Scheduling Function}
\label{app:schedule}
\edef\schedfloatrestore{\noexpand\renewcommand{\noexpand\topfraction}{\topfraction}%
\noexpand\renewcommand{\noexpand\bottomfraction}{\bottomfraction}%
\noexpand\renewcommand{\noexpand\textfraction}{\textfraction}%
\noexpand\renewcommand{\noexpand\floatpagefraction}{\floatpagefraction}%
\noexpand\setcounter{topnumber}{\the\value{topnumber}}%
\noexpand\setcounter{totalnumber}{\the\value{totalnumber}}}
\renewcommand{\topfraction}{0.9}\renewcommand{\bottomfraction}{0.5}
\renewcommand{\textfraction}{0.08}\renewcommand{\floatpagefraction}{0.85}
\setcounter{topnumber}{3}\setcounter{totalnumber}{4}

A uniform schedule gives every noise level the same number of optimization steps, although the analysis of \cref{sec:efficiency} shows that steps matter unequally across noise levels: every level needs a few, so that a lean towards the wrong mode is corrected before it is carried to lower noise levels (\cref{prop:wrong-mode-slowdown}), and the remaining steps are worth the most inside the selection window (\cref{prop:interior-selection-window}). The scheduling function therefore treats the optimization steps of \cref{alg:sampler} as a total budget $M$ and allocates it across noise levels, keeping a floor at every level and concentrating the remainder in the selection window. \Cref{fig:schedule} shows the resulting allocations, and that withholding the steps of a single noise level hurts most where the scheduler concentrates them. We define the allocation below, then ablate it (\cref{app:sched-ablations}), describe how the window is tuned (\cref{app:sched-tuning}), and apply it to DDIS (\cref{app:sched-ddis}).  

\begin{figure}[t]
\centering
\includegraphics[width=\linewidth]{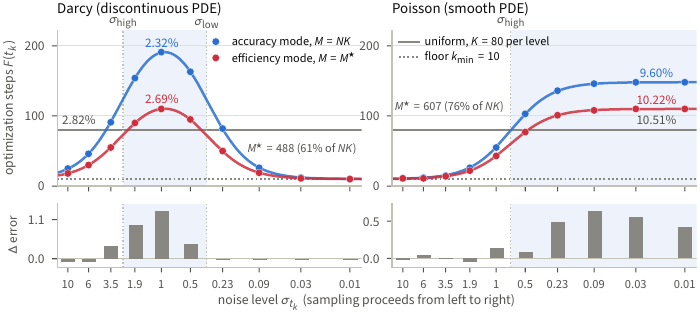}
\vspace{-0.5cm}
\caption{\textbf{Top}: Optimization steps per noise level under the scheduling function (\cref{app:schedule}). The accuracy mode (blue) reallocates the $NK = 800$ steps of the uniform allocation (grey) towards the selection window (shaded), and the efficiency mode (red) cuts the budget to $M^{\star}$; every noise level keeps at least $k_{\min} = 10$ steps (dotted). Numbers give the resulting error (\cref{tab:main}). \textbf{Bottom}: increase in error, in percentage points, when a single noise level receives no optimization steps, which is largest where the scheduler places its steps.}
\label{fig:schedule}
\end{figure} 

\textbf{Effort weights.} The scheduling function assigns each noise level $\sigma$ an effort weight through a smooth window whose only task-specific parameters are the two edges $\sigma_{\mathrm{low}} \leq \sigma_{\mathrm{high}}$ of the selection window:
\begin{equation}
    \omega(\sigma) = \Big[1 + (\sigma/\sigma_{\mathrm{high}})^{\beta}\Big]^{-1}\Big[1 + (\sigma_{\mathrm{low}}/\sigma)^{\beta}\Big]^{-1}.
    \label{eq:effort-weight}
\end{equation}
The weight is close to one inside the window and decays polynomially outside it, at a rate set by the exponent $\beta > 0$. Setting $\sigma_{\mathrm{low}} = 0$ removes the lower edge, setting $\sigma_{\mathrm{high}} = \infty$ removes the upper edge, and $\beta \to \infty$ recovers a hard window. Because $\omega$ depends on the noise level rather than on the index $k$, the same edges apply to any number and spacing of noise levels.

\textbf{Allocation.} Given a total budget of $M$ optimization steps, the scheduler reserves a floor of $k_{\min}$ steps at every noise level and distributes the remainder in proportion to the effort weights $\omega_k := \omega(\sigma_{t_k})$:
\begin{equation}
    F(t_k) = k_{\min} + \big(M - N k_{\min}\big)\,\frac{\omega_k}{\sum_{j=1}^{N} \omega_j}, \qquad k = 1, \ldots, N.
    \label{eq:allocation}
\end{equation}
We round each share down and give the remaining steps, one each, to the noise levels with the largest remainders, so that $\sum_{k} F(t_k) = M$ exactly. A uniform allocation of $K$ steps per noise level corresponds to $\omega \equiv 1$, and at the same budget $M = NK$, \cref{eq:allocation} moves steps towards the selection window without changing the cost of the sampler.

\textbf{Accuracy and efficiency modes.} Since the effort weights identify the noise levels on which optimization steps matter, the scheduler also determines a natural budget below $NK$:
\begin{equation}
    M^{\star} = N k_{\min} + (K - k_{\min})\,\frac{\sum_{j} \omega_j}{\tau \max_{j} \omega_j}, \qquad \tau \in (0, 1],
    \label{eq:budget-cut}
\end{equation}
capped at $NK$. At this budget, \cref{eq:allocation} assigns exactly $K$ steps to a noise level whose weight is a fraction $\tau$ of the peak, so every noise level with at least this weight keeps at least as many steps as under the uniform allocation, and the saving comes from the noise levels outside the window. We call $M = NK$ the \emph{accuracy mode} and $M = M^{\star}$ the \emph{efficiency mode}, and write FunBCS-AC and FunBCS-EF for FunBCS in these two modes. For every run, FunBCS-AC seeks to non-uniformly disperse the $M$ total optimization steps in order to reach higher accuracy without changing the budget, while FunBCS-EF seeks to cut as much budget as possible while maintaining similar accuracy to the uniform allocation. During inference, one can easily interpolate between the two modes with no tuning as the two modes differ only in the budget. 

\textbf{Shared constants.} All FunBCS experiments use $\beta = 2$, $k_{\min} = 10$ and $\tau = 0.7$ on every PDE and in both modes, with $N = 10$ noise levels and $K = 80$, so that the accuracy mode has a budget of $800$ steps. Only the window edges are chosen per PDE (\cref{app:sched-tuning}): $(\sigma_{\mathrm{low}}, \sigma_{\mathrm{high}}) = (0, 0.711)$ on Poisson, $(0, 0.335)$ on Helmholtz, $(0, 0.144)$ on Navier--Stokes and $(0.335, 2.6)$ on Darcy.

\subsection{Scheduler Ablations}
\label{app:sched-ablations}

\Cref{tab:main} compares FunBCS with the scheduler (FunBCS-AC and FunBCS-EF) against FunBCS without it, which places $K = 80$ optimization steps on each of its $N = 10$ noise levels. At the same budget of $800$ steps, the accuracy mode lowers the error by $10\%$ on average, and on every task, at unchanged cost. The efficiency mode removes $37\%$ of the steps on average, which makes FunBCS $1.6\times$ faster, with $2\%$ lower error on average. On Poisson, Helmholtz, Navier--Stokes and Darcy, it uses $607$, $506$, $406$ and $488$ of the $800$ steps, which take $12.6$, $10.5$, $8.4$ and $10.1$ seconds per sample; \cref{tab:main} reports their mean. Even without the scheduler, FunBCS is more accurate than every baseline on every task, and faster than DDIS.

\begin{figure}[t]
\centering
\includegraphics[width=\linewidth]{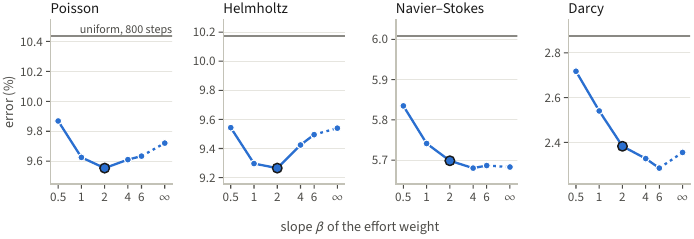}
\caption{Error of FunBCS in the accuracy mode on $50$ samples as the slope $\beta$ of \cref{eq:effort-weight} varies, at the selected window edges; $\infty$ denotes a hard window, and the outlined marker is the value used, $\beta = 2$. The grey line is the uniform allocation of the same $800$ steps. Darcy reports the pixel error rate and the other tasks the relative $\ell_2$ error.}
\label{fig:sched-beta}
\end{figure}

\begin{figure}[t]
\centering
\includegraphics[width=\linewidth]{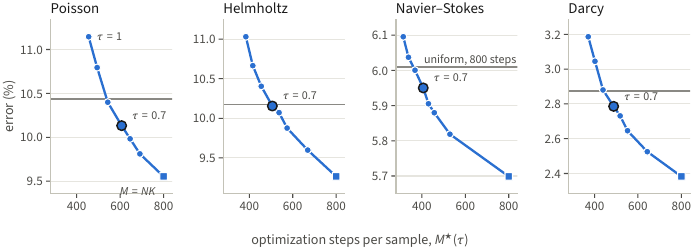}
\caption{Error of FunBCS on $50$ samples against the budget $M^{\star}$ resolved by \cref{eq:budget-cut} for $\tau = 1, 0.9, 0.8, 0.7, 0.65, 0.6$ and $0.5$, from left to right. The outlined marker is the value used, $\tau = 0.7$, the square is the accuracy mode ($M = NK = 800$), and the grey line is the uniform allocation at $800$ steps. On Poisson, $\tau = 0.5$ reaches the cap and coincides with the accuracy mode.}
\label{fig:sched-tau}
\end{figure}

\textbf{Varying the shared constants.} \Cref{fig:sched-beta,fig:sched-tau,fig:sched-kmin} vary the three constants that are shared across tasks, on $50$ samples with everything else fixed. The slope $\beta$ controls how sharply the effort weight of \cref{eq:effort-weight} falls off at the edges of the window. As $\beta \to 0$, the weight becomes flat and the allocation approaches the uniform one; as $\beta \to \infty$, the window becomes a hard box, and every noise level outside it keeps only the floor. We use $\beta = 2$, which gives the lowest error averaged over the four tasks. The constant $\tau$ controls the budget of the efficiency mode through \cref{eq:budget-cut}. At $\tau = 1$, the noise level with the largest weight keeps the $K$ steps of the uniform allocation and all other levels receive fewer; lowering $\tau$ scales the steps above the floor by $1/\tau$, until $M^{\star}$ reaches the budget $NK$ of the accuracy mode. The floor $k_{\min}$ is the number of steps that every noise level keeps regardless of its weight. Without a floor, the noise levels outside the window receive almost no steps, while a large floor moves the allocation back towards the uniform one.

\begin{figure}[t]
\centering
\begin{minipage}[t]{0.455\linewidth}
\centering
\includegraphics[height=2.03in]{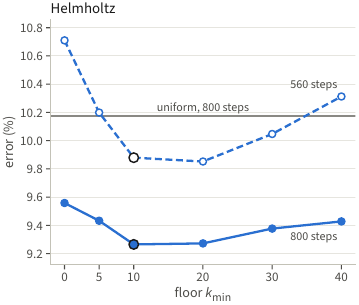}
\caption{Error of FunBCS on Helmholtz on $50$ samples as the floor $k_{\min}$ varies, at $800$ and at $560$ steps. Outlined markers are the value used, $k_{\min} = 10$, and the grey line is the uniform allocation at $800$ steps.}
\label{fig:sched-kmin}
\end{minipage}\hfill
\begin{minipage}[t]{0.505\linewidth}
\centering
\includegraphics[height=2.03in]{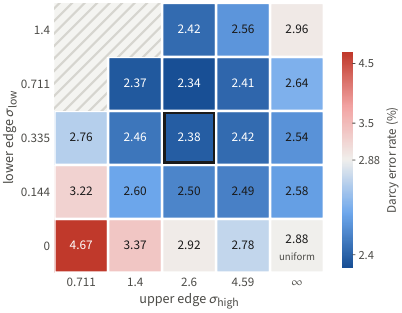}
\caption{Darcy error rate (\%) of FunBCS on $50$ samples for each window; $\sigma_{\mathrm{low}} = 0$ and $\sigma_{\mathrm{high}} = \infty$ remove an edge, so the bottom-right cell is the uniform allocation. Blue cells beat it, and the outlined cell is the selected window.}
\label{fig:sched-window}
\end{minipage}
\end{figure}

\subsection{Tuning the Selection Window}
\label{app:sched-tuning}

Since $\beta$, $k_{\min}$ and $\tau$ are shared across tasks, only the two edges of the selection window are set for a new PDE. We choose them with the grid search of \cref{alg:locate-window}, run with FunBCS in the accuracy mode. The candidate edges are the geometric midpoints between consecutive noise levels, roughly a factor of two apart; the upper edge is selected first, and a lower edge is kept only if it lowers the error further. Each candidate window is evaluated on $20$ samples, so the search, which needs at most nine runs, takes about $50$ minutes per PDE on a single A100 GPU.

\textbf{Sensitivity to the window edges.} A coarse grid suffices because the error depends on the shape of the window far more than on the exact position of its edges. \Cref{fig:sched-window} shows this on Darcy, the only task whose selected window has both edges. Moving one or both edges of the selected window by one grid step, a factor of about two in $\sigma$, keeps the error rate between $2.34$ and $2.60\%$, against $2.88\%$ for the uniform allocation, and $12$ of the $13$ windows with both edges improve on the uniform allocation. A window of the wrong shape, in contrast, does not: three of the four one-sided low-noise windows that suit the smooth tasks are worse than uniform on Darcy, with error rates up to $4.67\%$. The selection window is therefore not a sensitive choice. Once its shape is right, each edge only needs to be located to within a grid step, which the coarse search of \cref{alg:locate-window} does in a few short runs.

\begin{algorithm}[t]
\caption{Finding the selection window for each PDE}
\label{alg:locate-window}
\begin{algorithmic}[1]
\REQUIRE A few samples of the PDE (outside the training and testing sets); set of candidate window bounds $\mathcal{S}$, the midpoints $\sqrt{\sigma_{t_k}\sigma_{t_{k+1}}}$ between consecutive noise levels; a tolerance $\epsilon$ for noise in the measured error.
\STATE Run every configuration below with FunBCS (\cref{alg:sampler}) in the accuracy mode, using the shared $\beta$ and $k_{\min}$.
\STATE Measure the error $e_{\mathrm{uni}}$ of the uniform allocation.
\STATE \textbf{Upper bound.} For each $s \in \mathcal{S}$, measure the error $e(s)$ of the window $(0, s)$.
\STATE Set $\sigma_{\mathrm{high}}$ to the $s$ with the lowest error, or to $\infty$ (no upper bound) if no window beats the uniform allocation. If $\sigma_{\mathrm{high}}$ is the smallest or largest candidate, also try the next midpoint beyond it and keep the better of the two.
\STATE \textbf{Lower bound.} For each $s \in \mathcal{S}$ below $\sigma_{\mathrm{high}}$, measure the error $e'(s)$ of the window $(s, \sigma_{\mathrm{high}})$, and let $s^{\star}$ be the $s$ with the lowest error.
\STATE Set $\sigma_{\mathrm{low}} = s^{\star}$ if $e'(s^{\star})$ is lower than the error of the window $(0, \sigma_{\mathrm{high}})$ by more than $\epsilon$, and $\sigma_{\mathrm{low}} = 0$ (no lower bound) otherwise.
\RETURN $(\sigma_{\mathrm{low}}, \sigma_{\mathrm{high}})$
\end{algorithmic}
\end{algorithm}

\subsection{Generalizing Scheduling Function to DDIS}
\label{app:sched-ddis}

We evaluate whether the scheduling function generalizes to DDIS, an output space optimization method. It significantly improves the speed and also the accuracy of DDIS after only a short tuning of the window $[\sigma_{\mathrm{low}}, \sigma_{\mathrm{high}}]$. DDIS spends its budget on the Langevin steps with which it samples \cref{eq:ddis-correction} at each of its $N = 100$ noise levels. We allocate these steps with \cref{eq:allocation}, using $K = 20$ steps per noise level for the uniform budget of $2{,}000$ steps, a floor of $k_{\min} = 2$, and the same $\beta$ and $\tau$ as for FunBCS. The floor differs from that of FunBCS only because DDIS takes fewer steps per noise level: we keep the same fraction, one eighth of the steps of each noise level ($10$ of $K = 80$ for FunBCS), which gives $k_{\min} = 2$ for $K = 20$ after rounding down. The edges are chosen with the same search (\cref{alg:locate-window}), which takes under an hour per PDE on a single A100 GPU. On Poisson, Helmholtz and Navier--Stokes it returns a one-sided low-noise window, $(0, 0.335)$ on Poisson and Helmholtz and $(0, 0.144)$ on Navier--Stokes, identical to or one grid step from the window of FunBCS. On Darcy it returns the window $(0.144, \infty)$.

\Cref{tab:sched-ddis} reports the results on the $100$ test samples of \cref{tab:main}. In the accuracy mode, the scheduler lowers the error of DDIS by $8\%$ on average at unchanged cost, and on all four tasks. In the efficiency mode, it makes DDIS $1.3\times$ faster on average while lowering its error by $1\%$ on average. On Darcy, the window covers most noise levels of DDIS, so $M^{\star}$ reaches the full budget and the two modes coincide. The time falls by less than the number of steps because the probability-flow solve that DDIS runs at every noise level does not depend on the budget.

\begin{table}[!ht]
\caption{DDIS with and without the scheduling function on the $100$ test samples of \cref{tab:main}. The uniform allocation is DDIS as in \cref{tab:main}, with $K = 20$ Langevin steps on each of its $N = 100$ noise levels. The accuracy mode redistributes the same $2{,}000$ steps, and the efficiency mode cuts steps while maintaining baseline accuracy. On Darcy, $M^{\star}$ reaches the full budget, so the two modes coincide. Errors are mean $\pm$ standard deviation over the test samples. }
\label{tab:sched-ddis}
\centering
\small
\setlength{\tabcolsep}{3.3pt}
\def\tstd#1{{\scriptsize$\pm$#1}}
\resizebox{\linewidth}{!}{%
\begin{tabular}{l cc cc cc cc}
\toprule
 & \multicolumn{2}{c}{Poisson} & \multicolumn{2}{c}{Helmholtz} & \multicolumn{2}{c}{Navier--Stokes} & \multicolumn{2}{c}{Darcy} \\
\cmidrule(lr){2-3}\cmidrule(lr){4-5}\cmidrule(lr){6-7}\cmidrule(lr){8-9}
DDIS & rel.\ $\ell_2$ & s/sample & rel.\ $\ell_2$ & s/sample & rel.\ $\ell_2$ & s/sample & err.\ rate & s/sample \\
\midrule
uniform allocation & 15.76\tstd{5.08} & \underline{18.0} & \underline{15.47}\tstd{4.34} & \underline{18.0} & 8.36\tstd{1.55} & \underline{18.0} & \underline{7.48}\tstd{3.89} & 18.0 \\
accuracy mode & \textbf{14.68}\tstd{4.71} & \underline{18.0} & \textbf{13.95}\tstd{4.08} & \underline{18.0} & \textbf{8.08}\tstd{1.48} & \underline{18.0} & \textbf{6.64}\tstd{3.62} & 18.0 \\
efficiency mode & \underline{15.56}\tstd{5.00} & \textbf{13.1} & 16.66\tstd{6.10} & \textbf{13.1} & \underline{8.28}\tstd{1.56} & \textbf{11.3} & \textbf{6.64}\tstd{3.62} & 18.0 \\
\bottomrule
\end{tabular}}
\end{table}

\FloatBarrier
\schedfloatrestore

\begin{figure}[t]
    \centering
    \begin{subfigure}[t]{0.48\linewidth}
        \centering
        \includegraphics[width=\linewidth]{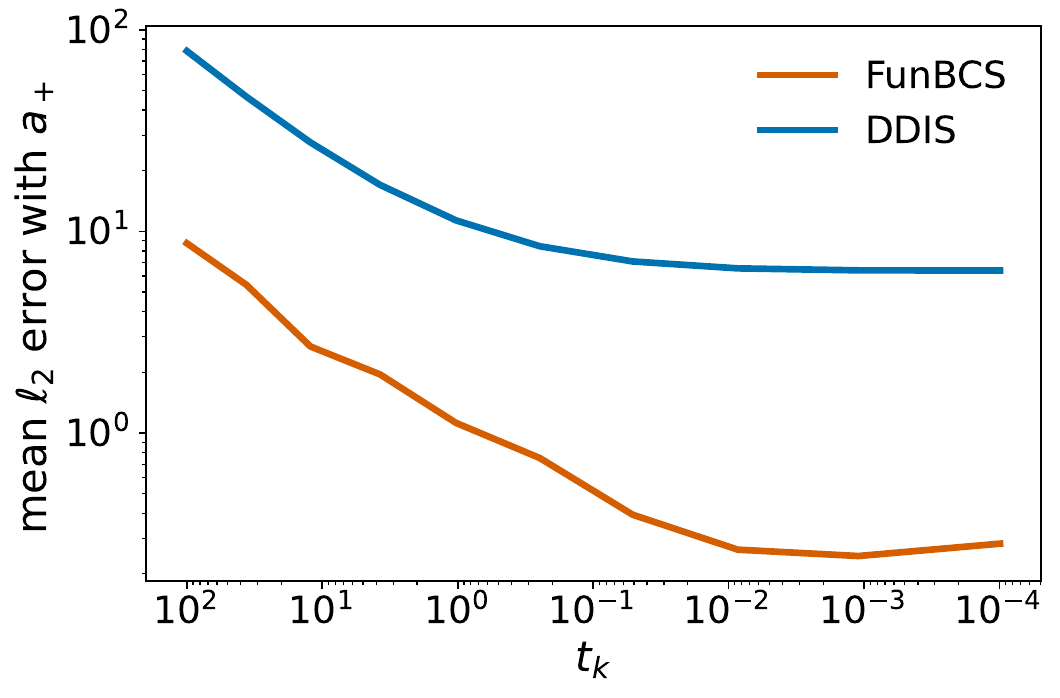}
        \caption{Nullspace error to $\bm{a}_+$. FunBCS (orange) performs significantly better than DDIS (blue).}
        \label{fig:poisson_ker_contract}
    \end{subfigure}
    \hfill 
     \begin{subfigure}[t]{0.48\linewidth}
        \centering
        \includegraphics[width=\linewidth]{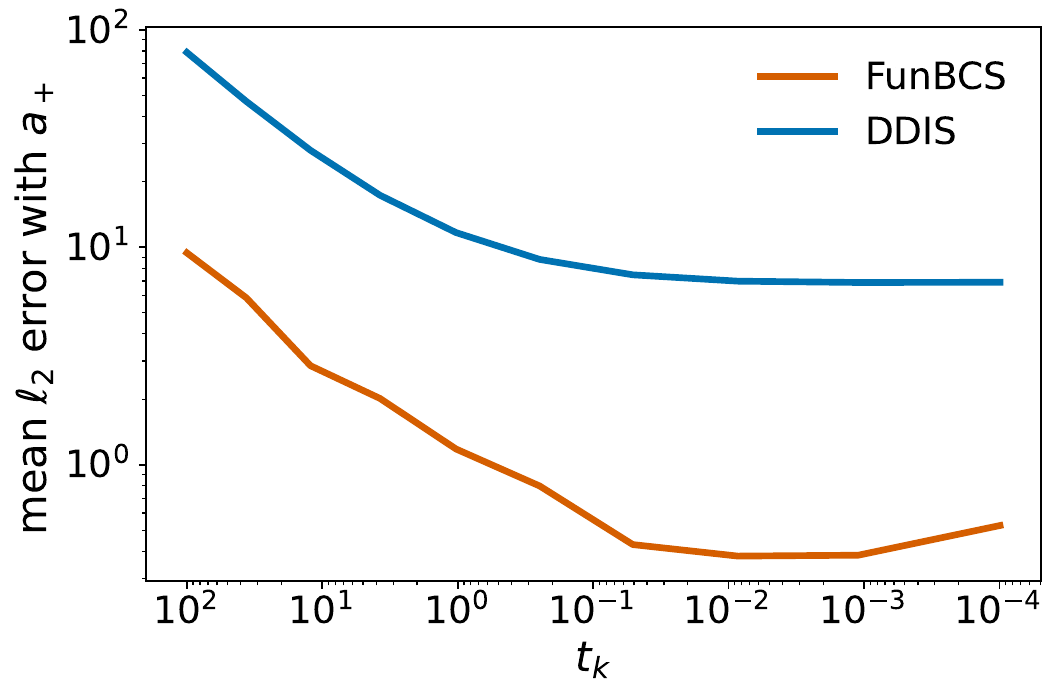}
        \caption{Cumulative mean error to $\bm{a}_+$. FunBCS (orange) performs significantly better than DDIS (blue).}
        \label{fig:poisson_overall_contract}
    \end{subfigure}
    \caption{We find that in nullspace (a) and overall (b), FunBCS (orange) has significantly less error on our discontinuous Poisson problem discussed in \cref{sec:analysis}, as compared to DDIS (blue) validating our theoretical results.}
    \label{fig:discont_poisson}
\end{figure}

\section{Navier--Stokes Spectral Analysis}
\label{app:ns}

Navier--Stokes is the only task where our two metrics disagree: our method improves relative $\ell_2$ over DDIS but worsens the spectral energy error $\es$ (Table~\ref{tab:main}). This section examines what drives that gap.

The two metrics can disagree because $\es$ compares radially averaged power spectra, and is therefore phase-blind. It measures whether a reconstruction carries the right amount of energy at each spatial frequency, but not whether that energy sits in the right places. A reconstruction whose high-frequency content is entirely unrelated to the ground truth can still score well on $\es$, as long as its spectrum has roughly the right shape.

To tell these two cases apart, we take the 2D Fourier transform of the reconstruction and of the ground truth, group the modes into radial bins by wavenumber $k$, and compute a normalized inner product between the two sets of coefficients in each bin,
\begin{equation}
  \mathcal{K}(k) \;=\; \frac{\big|\sum_{\mathbf{k} \in S_k} \overline{\hat A(\mathbf{k})}\, A(\mathbf{k})\big|}
  {\sqrt{\sum_{\mathbf{k} \in S_k} |\hat A(\mathbf{k})|^2} \; \sqrt{\sum_{\mathbf{k} \in S_k} |A(\mathbf{k})|^2}},
  \label{eq:coherence}
\end{equation}
where $\hat A$ and $A$ are the Fourier coefficients of the reconstruction and the ground truth, $\overline{(\cdot)}$ denotes complex conjugation, and $S_k$ collects every mode falling in radial bin $k$, across all test instances.

This is a correlation coefficient computed in the Fourier domain, and Cauchy--Schwarz places it in $[0,1]$. It equals $1$ when the reconstructed modes match the true modes in both amplitude and phase, and falls toward $0$ when the two are unrelated, no matter how much energy the reconstruction places at that frequency. 

\begin{figure}[t]
\centering
\includegraphics[width=\linewidth]{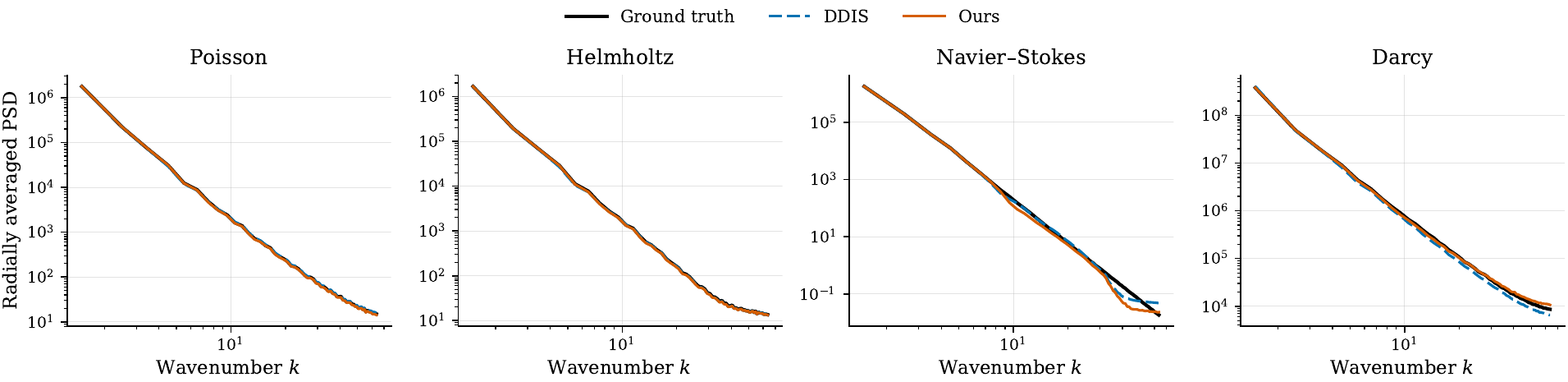}
\caption{Radially averaged power spectral density of the Navier--Stokes reconstructions against the ground truth, alongside three other PDEs. Above $k\!\approx\!8$, DDIS retains substantially more power than our method and exceeds the ground-truth spectrum over part of this range, which is what drives its lower $\es$ in Table~\ref{tab:main}.}
\label{fig:psd}
\end{figure}

\begin{figure}[t]
\centering
\includegraphics[width=\linewidth]{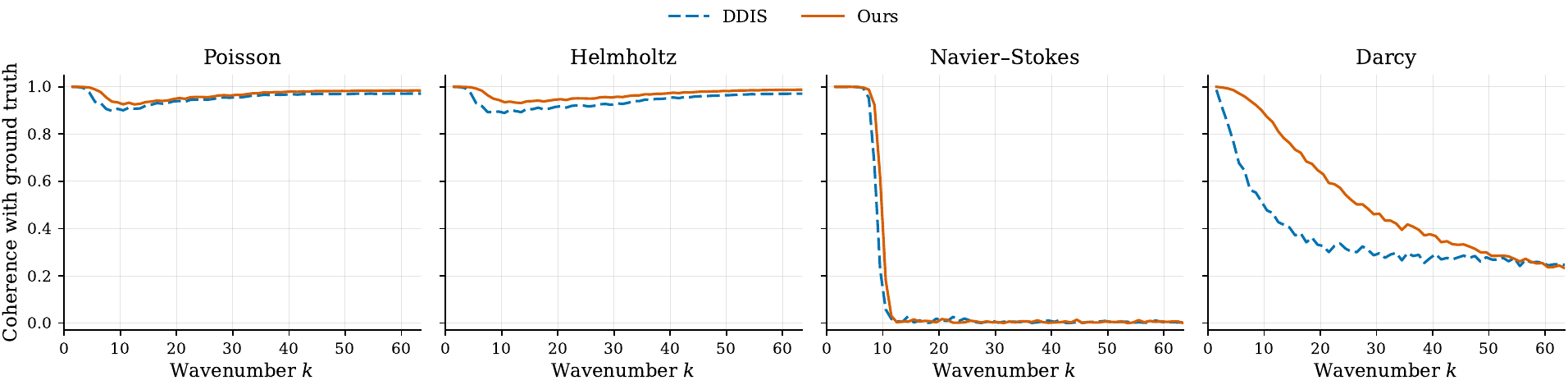}
\caption{Per-wavenumber coherence with the ground truth on four PDEs. Beyond the same $k\!\approx\!8$ crossover, coherence is near zero for \emph{both} methods, indicating that neither recovers the true high-wavenumber modes. Together with \cref{fig:psd}, this suggests the $\es$ gap reflects prior-consistent texture in unidentified modes rather than recovered information.}
\label{fig:coherence}
\end{figure}

We observe that in \cref{fig:psd}, at the crossover at $k \approx 8$, DDIS retains considerably more power than we do, and this is what produces its lower $\es$. However, the coherence of both methods is near zero over exactly this range, so the extra power DDIS retains is not aligned with the true modes.

Comparing both curves against the ground truth sharpens this. Our reconstructions fall below the true spectrum in this range, whereas DDIS rises above it over part of it, so DDIS is adding high-frequency detail that the true field does not contain and that does not match the true modes. Because this content is uncorrelated with the ground truth, it adds to the squared error rather than reducing it, which is consistent with our lower relative $\ell_2$, as shrinking unrecoverable modes toward zero is the better choice under an $\ell_2$ objective, even though it costs spectral realism. We therefore report $\es$ alongside coherence, since on this benchmark a lower $\es$ reflects plausible but spurious high-frequency texture rather than recovered structure.

\section{Additional Experiments}
\label{sec:additional_exp}

\subsection{Visualizing Reconstructed Fields \& Solutions} 
\label{app:visualizations}

In this section we provide additional visualizations of the reconstructed coefficient fields, together with the solution fields estimated from the same sparse observations, on Poisson and Helmholtz (\cref{fig:qual-ph}) and on Navier--Stokes and Darcy (\cref{fig:qual-nd}).

\subsection{Box Plots of Main Results}
\label{sec:box_plots}

\Cref{fig:box-main} shows the distribution of the errors in \cref{tab:main} over the $100$ test samples.

\begin{figure}[ht]
\centering
\includegraphics[width=\linewidth]{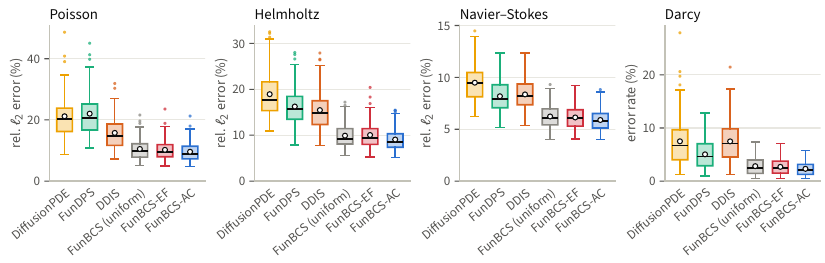}
\caption{Per-sample errors of the recovered coefficient field on the $100$ test samples of \cref{tab:main}: relative $\ell_2$ error on Poisson, Helmholtz and Navier--Stokes, and pixel error rate on Darcy. Each box spans the interquartile range, with the median as a line and the mean reported in \cref{tab:main} as a white dot. Whiskers extend to $1.5$ times the interquartile range, and samples beyond them are drawn individually.}
\label{fig:box-main}
\end{figure}

\subsection{Study of Discontinuous Poisson's Equation}
\label{sec:discont_poisson}

In order to validate \cref{prop:sitcom_ker_contract}, and the distinction between DDIS and FunBCS more broadly, we perform an experiment on discontinuous Poisson following our setting in \cref{sec:analysis} except we use a neural operator surrogate; learned prior; compute $\bar{\bm{a}}_{D, t_{k - 1}}$ by averaging 100 samples $\bm{a}_{t_{k - 1}}^{(D)}$. We consider $\Omega = (-32, 32)$, and $z_- = -2, z_+ = 2$. Note that we do not configure an input optimization scheduler, instead using 10 reverse sampling steps each with 30 input optimization steps. We compute $\bm{P}_{\mathrm{ker} \bm{A}}$ with respect to the finite difference approximation $\bm{G}$. As expected, per \cref{fig:poisson_ker_contract}, FunBCS obtains significantly less error than DDIS in nullspace; per \cref{fig:poisson_overall_contract}, FunBCS thus obtains significantly less error overall.

\begin{figure}[p]
\centering
\includegraphics[width=0.8\linewidth]{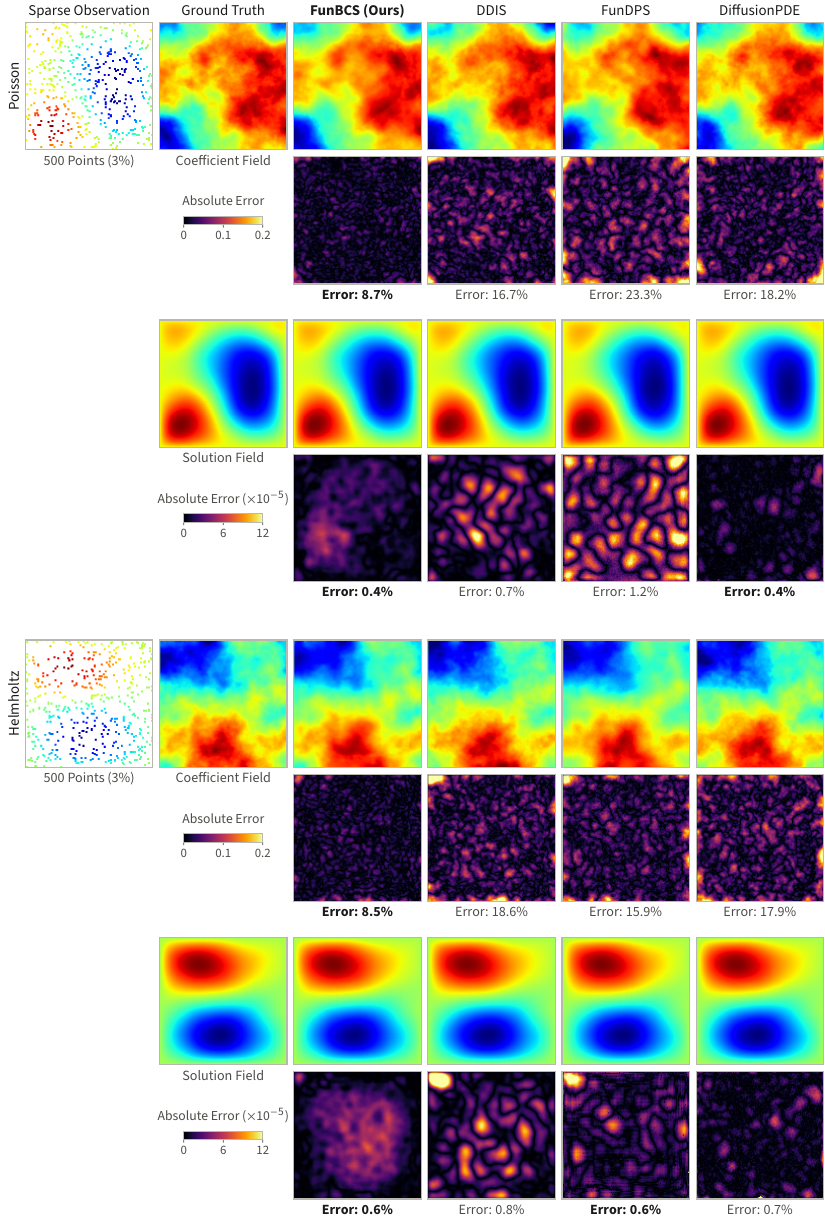}
\caption{Reconstructions on one test sample of Poisson and of Helmholtz. For each PDE, the upper block shows the coefficient field and the lower block the solution field estimated from the same sparse observations, each with its absolute error to the ground truth underneath and its error on that sample. Every method here was run with their published default test-time compute budget, and FunBCS was run in efficiency mode, which is $1.7\times$, $1.4\times$, and $26.4\times$ faster than DDIS, FunDPS, and DiffusionPDE, respectively. FunBCS and DDIS obtain the solution field by applying the neural operator to their recovered coefficient field, whereas FunDPS and DiffusionPDE sample both fields jointly.}
\label{fig:qual-ph}
\end{figure}

\begin{figure}[p]
\centering
\includegraphics[width=0.8\linewidth]{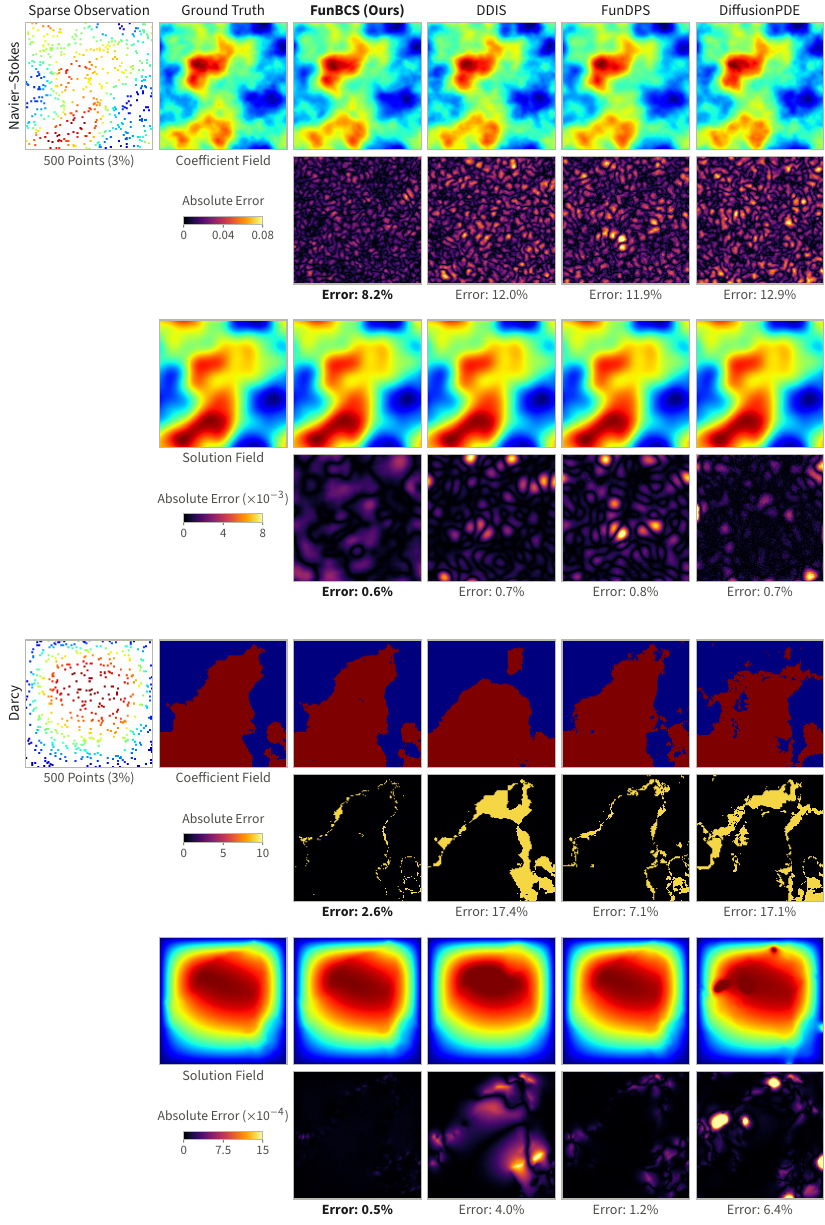}
\caption{Reconstructions on one test sample of Navier--Stokes and of Darcy, arranged as in \cref{fig:qual-ph}. On Darcy, the error of the coefficient field is the pixel error rate.}
\label{fig:qual-nd}
\end{figure}

\section{Connection with Preconditioned Output Space Optimization}
\label{sec:preconditioning}

As discussed in \cref{sec:related_work}, input optimization is equivalent to preconditioned output space optimization. We discuss this further below. We consider the setting of \cref{sec:analysis}. Firstly, we note that output space optimization fails because the unobserved component of the measurement consistency gradient is 0: 

\begin{proposition}
    Let $\mathcal{L}(\bm{a}) := \frac{\|\bm{Aa} - \bm{u}_{\mathrm{obs}}\|_2^2 }{2\gamma^2}$ for a $\gamma > 0$. Then, 
    \begin{equation}
        \nabla_{\bm{a}} \mathcal{L}(\bm{a}) = \frac{1}{\gamma^2}\bm{A}^T(\bm{Aa} - \bm{u}_{\mathrm{obs}}) 
    \end{equation}
    and thus 
    \begin{equation}
        \bm{P}_{\mathrm{ker} \bm{A}} \nabla_{\bm{a}} \mathcal{L}(\bm{a}) = 0.
    \end{equation}
    \label{prop:measurement_consistency_grad}
\end{proposition}

\textit{Proof}: See \cref{proof:proof_of_measurement_consistency_grad}. 

Consequently, an output space optimization procedure which leverages the measurement consistency gradient, unchanged, will not update the nullspace component. In particular, output space optimization can be viewed in continuous time as computing $\dot{\bm{a}} = -\nabla_{\bm{a}} \mathcal{L}(\bm{a})$, and consequently $\bm{P}_{\mathrm{ker} \bm{A}} \dot{\bm{a}} = 0$. However, one can leverage a \textit{preconditioner} to couple the observed and unobserved components. Then, changing the observed component in output space by minimizing the measurement consistency loss will also change the unobserved component. In particular, our approach as outlined in \cref{sec:method} is equivalent to performing preconditioned output space optimization: 

\begin{proposition}
    Assume $\bm{Ad} \neq 0$ and $\bm{P}_{\mathrm{ker} \bm{A}}\bm{d} \neq 0$. Fix diffusion time $t_k$ and consider the exact denoiser $D_{\bm{\theta}}(\cdot, t_k) = D^*(\cdot, t_k)$. Let $\bm{v}(s)$ be defined per the gradient flow in \cref{eq:sitcom-input-flow}. Let $\bm{a}(s) := D^*(\bm{v}(s), t_k)$. Then, we have that: 
    \begin{equation}
        \frac{d}{ds}\bm{a}(s) = -P(v(s), t_k) \nabla_{\bm{a}} \mathcal{L}(\bm{a})|_{\bm{a} = \bm{a}(s)}, 
    \end{equation}
    where
\begin{equation}
    \bm{P}(\bm{v},t_k)
    :=
    \nabla_{\bm{v}}D^*(\bm{v},t_k)
    \nabla_{\bm{v}}D^*(\bm{v},t_k)^\top.
\end{equation}
    Consequently,  
    \begin{equation}
        \bm{P}_{\mathrm{ker} \bm{A}} \frac{d}{ds}\bm{a}(s) \neq 0. 
    \end{equation}
    \label{prop:preconditioning}
\end{proposition}

\textit{Proof}: See \cref{proof:proof_of_preconditioning}. 

Thus, optimizing over the input is equivalent to preconditioned output space optimization, which can update the unobserved components. We note that the recent approach of \citep{shi2026flowannealingposteriorsampling} also incorporates a preconditioner, but does not motivate it rigorously and requires extra computation; our approach is instead motivated from the perspective of PDE inverse problems and leverages the diffusion denoiser. 

\section{Proofs}
\label{sec:proofs}

Consider the setting of \cref{sec:function_space_diffusion}. Fix diffusion times  $t_k>t_{k-1}$ 
with $\sigma_{t_k}\ge\sigma_{t_0}>0$. Throughout, we consider the setting of \cref{sec:analysis}. 
In proofs, $k_{t_k}(\cdot)$ denotes the function defined by
$k_{t_k}(\bm{v}):=\mathbb{P}(\bm{a}_0=\bm{a}_+\mid \bm{a}_{t_k}=\bm{v})$, and
$w(s):=k_{t_k}(\bm{v}(s))$. Here, $k_{t_k}(\bm{a}_{t_k})$ means evaluation
at the forward random vector $\bm{a}_{t_k}$, while
$k_{t_k}(\bm{a}_{t_k}^{(D)})$ and $k_{t_k}(\bm{a}_{t_k}^{(B)})$ mean
evaluation of this same function at the algorithmic states. We show this regular conditional probability is well-defined in \cref{lemma:two-point-denoiser}.

\subsection{Helpful Lemmas}
\label{sec:lemmas}

\begin{lemma}[Conditional laws after renoising]
\label{lemma:renoising}
The update in \cref{eq:DDIS-update} satisfies
\begin{equation}
 \bm{a}_{t_{k-1}}^{(D)}\mid \bm{a}_{t_k}^{(D)},\bm{u}_{\mathrm{obs}}
 \sim\mathcal N\left(\bar{\bm{a}}_{\mathrm D,t_{k-1}},
 \left(\frac{\bm{A}^\top \bm{A}}{\gamma^2}+\frac{\bm{I}}{\sigma_{t_k}^2}\right)^{-1}
 +\sigma_{t_{k-1}}^2\bm{K}\right),
 \label{eq:ddis-renoised-law}
\end{equation}
where $\bar{\bm{a}}_{\mathrm D,t_{k-1}}$ is given by \cref{lemma:ddis-projected-mean}, and under the exact denoiser $D^*(\cdot, t_k)$, the update in \cref{eq:sitcom-update-rule} satisfies: 
\begin{equation}
 \bm{a}_{t_{k-1}}^{(B)}\mid \bm{a}_{t_k}^{(B)},\bm{u}_{\mathrm{obs}}
 \sim\mathcal N(D^*(\bm{v}(s),t_k),\sigma_{t_{k-1}}^2\bm{K}).
 \label{eq:sitcom-renoised-law}
\end{equation}
\end{lemma}
\begin{proof}
Conditional on $\bm{a}_{t_k}^{(D)},\bm{u}_{\mathrm{obs}}$, \cref{lemma:ddis-projected-mean} gives
\begin{equation}
 \tilde{\bm{a}}_0\sim\mathcal N(\bar{\bm{a}}_{\mathrm D,t_{k-1}},\bm{\Sigma}_{\mathrm D,t_k}),
 \qquad \bm{\varepsilon}_k^{(D)}\sim\mathcal N(\bm{0},\bm{K}),
 \qquad \tilde{\bm{a}}_0\perp\!\!\!\perp\bm{\varepsilon}_k^{(D)}.
\end{equation}
Therefore
\begin{equation}
 \sigma_{t_{k-1}}\bm{\varepsilon}_k^{(D)}\sim\mathcal N(\bm{0},\sigma_{t_{k-1}}^2\bm{K}),
 \qquad
 \tilde{\bm{a}}_0+\sigma_{t_{k-1}}\bm{\varepsilon}_k^{(D)}
 \sim\mathcal N(\bar{\bm{a}}_{\mathrm D,t_{k-1}},
               \bm{\Sigma}_{\mathrm D,t_k}+\sigma_{t_{k-1}}^2\bm{K}),
\end{equation}
proving 
\cref{eq:ddis-renoised-law}.

For \cref{eq:sitcom-renoised-law}, \cref{lemma:sitcom-existence} determines $\bm{v}(s)$ uniquely
for fixed $\bm{a}_{t_k}^{(B)},\bm{u}_{\mathrm{obs}}$. Thus $D^*(\bm{v}(s),t_k)\in\R^n$ is
fixed under this conditioning, and
\begin{equation}
 \bm{\varepsilon}_k^{(B)}\sim\mathcal N(\bm{0},\bm{K})
 \quad\Longrightarrow\quad
 D^*(\bm{v}(s),t_k)+\sigma_{t_{k-1}}\bm{\varepsilon}_k^{(B)}
 \sim\mathcal N(D^*(\bm{v}(s),t_k),\sigma_{t_{k-1}}^2\bm{K}).
\end{equation}
This proves \cref{lemma:renoising}.
\end{proof}

\begin{lemma}[DDIS mean and covariance]
\label{lemma:ddis-projected-mean}
The density in \cref{eq:ddis-correction}, given $\mathcal{A}_h = \bm{A}$, is that of
$\mathcal N(\bar{\bm{a}}_{\mathrm D,t_{k-1}},\bm{\Sigma}_{\mathrm D,t_k})$, where
\begin{equation}
    \bar{\bm{a}}_{\mathrm D,t_{k-1}}
    =\left(\frac{\bm{A}^\top \bm{A}}{\gamma^2}
                 +\frac{\bm{I}}{\sigma_{t_k}^2}\right)^{-1}
      \left(\frac{\bm{A}^\top \bm{u}_{\mathrm{obs}}}{\gamma^2}
                 +\frac{D_{\bm{\theta}}(\bm{a}_{t_k}^{(D)},t_k)}{\sigma_{t_k}^2}\right)
    \label{eq:ddis-mean}
\end{equation}
and
\begin{equation}
    \bm{\Sigma}_{\mathrm D,t_k}
    :=\left(\frac{\bm{A}^\top \bm{A}}{\gamma^2}
                 +\frac{\bm{I}}{\sigma_{t_k}^2}\right)^{-1}.
    \label{eq:ddis-covariance}
\end{equation}
\end{lemma}

\begin{proof}
Fix $t_k>0$, $\bm{a}_{t_k}^{(D)}\in\R^n$, and $\bm{u}_{\mathrm{obs}}\in\R^m$.
Throughout the proof, $\bm{a}_0\in\R^n$ is the variable in the density.
DDIS approximates the denoising conditional density by
\begin{equation}
    q_D(\bm{a}_0\mid \bm{a}_{t_k}^{(D)})
    =
    \frac{1}{(2\pi\sigma_{t_k}^2)^{n/2}}
    \exp\!\left\{
        -\frac{\|\bm{a}_0-D_{\bm{\theta}}(\bm{a}_{t_k}^{(D)},t_k)\|^2}{2\sigma_{t_k}^2}
    \right\}.
    \label{eq:ddis-denoising-density}
\end{equation}
The observation likelihood is
\begin{equation}
    p(\bm{u}_{\mathrm{obs}}\mid \bm{a}_0)
    =
    \frac{1}{(2\pi\gamma^2)^{m/2}}
    \exp\!\left\{
        -\frac{\|\bm{A}\bm{a}_0-\bm{u}_{\mathrm{obs}}\|^2}{2\gamma^2}
    \right\}.
    \label{eq:ddis-likelihood-density}
\end{equation}
DDIS defines the following density
\begin{equation}
    q_{D,t_k}(\bm{a}_0\mid \bm{a}_{t_k}^{(D)},\bm{u}_{\mathrm{obs}})
    :=
    \frac{
        p(\bm{u}_{\mathrm{obs}}\mid \bm{a}_0)\,q_D(\bm{a}_0\mid \bm{a}_{t_k}^{(D)})
    }{
        \displaystyle\int_{\R^n}
        p(\bm{u}_{\mathrm{obs}}\mid \bm{x})\,q_D(\bm{x}\mid \bm{a}_{t_k}^{(D)})\,d\bm{x}
    }.
    \label{eq:ddis-normalized-definition}
\end{equation}
The denominator is strictly positive because its integrand is
everywhere positive. It is finite because
$p(\bm{u}_{\mathrm{obs}}\mid \bm{x})\le(2\pi\gamma^2)^{-m/2}$
and $q_D(\cdot\mid \bm{a}_{t_k}^{(D)})$ integrates to one. Thus
\cref{eq:ddis-normalized-definition} is well-defined.

Substituting
\cref{eq:ddis-denoising-density,eq:ddis-likelihood-density}
into \cref{eq:ddis-normalized-definition} gives
\begin{equation}
    q_{D,t_k}(\bm{a}_0\mid \bm{a}_{t_k}^{(D)},\bm{u}_{\mathrm{obs}})
    =
    \frac{
        \displaystyle
        \frac{
            \exp\!\left\{
                -\frac{\|\bm{A}\bm{a}_0-\bm{u}_{\mathrm{obs}}\|^2}{2\gamma^2}
                -\frac{\|\bm{a}_0-D_{\bm{\theta}}(\bm{a}_{t_k}^{(D)},t_k)\|^2}{2\sigma_{t_k}^2}
            \right\}
        }{
            (2\pi\gamma^2)^{m/2}(2\pi\sigma_{t_k}^2)^{n/2}
        }
    }{
        \displaystyle
        \frac{
            \int_{\R^n}
            \exp\!\left\{
                -\frac{\|\bm{A}\bm{x}-\bm{u}_{\mathrm{obs}}\|^2}{2\gamma^2}
                -\frac{\|\bm{x}-D_{\bm{\theta}}(\bm{a}_{t_k}^{(D)},t_k)\|^2}{2\sigma_{t_k}^2}
            \right\}\,d\bm{x}
        }{
            (2\pi\gamma^2)^{m/2}(2\pi\sigma_{t_k}^2)^{n/2}
        }
    }.
    \label{eq:ddis-full-density-ratio}
\end{equation}

We next complete the square in the exponent. The matrix
$\bm{A}^\top \bm{A}/\gamma^2+\bm{I}/\sigma_{t_k}^2$ is symmetric, and for every
nonzero $\bm{x}\in\R^n$,
\begin{equation}
    \bm{x}^\top
    \left(\frac{\bm{A}^\top \bm{A}}{\gamma^2}
                 +\frac{\bm{I}}{\sigma_{t_k}^2}\right)\bm{x}
    =
    \frac{\|\bm{A}\bm{x}\|^2}{\gamma^2}
    +\frac{\|\bm{x}\|^2}{\sigma_{t_k}^2}
    >0.
    \label{eq:ddis-positive-precision}
\end{equation}
Hence it is positive definite and invertible, so
$\bm{\Sigma}_{\mathrm D,t_k}$ in \cref{lemma:ddis-projected-mean} is
well-defined and positive definite. By definition, we have 
\begin{equation}
    \bm{\Sigma}_{\mathrm D,t_k}^{-1}\bar{\bm{a}}_{\mathrm D,t_{k-1}}
    =
    \frac{\bm{A}^\top \bm{u}_{\mathrm{obs}}}{\gamma^2}
    +\frac{D_{\bm{\theta}}(\bm{a}_{t_k}^{(D)},t_k)}{\sigma_{t_k}^2}.
    \label{eq:ddis-mean-linear-system}
\end{equation}

Expanding the two squared norms and using
\cref{eq:ddis-mean-linear-system} yields
\begin{equation}
\begin{aligned}
    &\frac{\|\bm{A}\bm{a}_0-\bm{u}_{\mathrm{obs}}\|^2}{2\gamma^2}
    +\frac{\|\bm{a}_0-D_{\bm{\theta}}(\bm{a}_{t_k}^{(D)},t_k)\|^2}{2\sigma_{t_k}^2}
    \\
    &=
    \frac{1}{2\gamma^2}
    (\bm{A}\bm{a}_0-\bm{u}_{\mathrm{obs}})^\top
    (\bm{A}\bm{a}_0-\bm{u}_{\mathrm{obs}})
    \\
    &\qquad
    +\frac{1}{2\sigma_{t_k}^2}
    (\bm{a}_0-D_{\bm{\theta}}(\bm{a}_{t_k}^{(D)},t_k))^\top
    (\bm{a}_0-D_{\bm{\theta}}(\bm{a}_{t_k}^{(D)},t_k))
    \\
    &=
    \frac{1}{2\gamma^2}
    \left(
        \bm{a}_0^\top\bm{A}^\top\bm{A}\bm{a}_0
        -2\bm{a}_0^\top\bm{A}^\top\bm{u}_{\mathrm{obs}}
        +\|\bm{u}_{\mathrm{obs}}\|^2
    \right)
    \\
    &\qquad
    +\frac{1}{2\sigma_{t_k}^2}
    \left(
        \bm{a}_0^\top\bm{a}_0
        -2\bm{a}_0^\top D_{\bm{\theta}}(\bm{a}_{t_k}^{(D)},t_k)
        +\|D_{\bm{\theta}}(\bm{a}_{t_k}^{(D)},t_k)\|^2
    \right)
    \\
    &=
    \frac12\bm{a}_0^\top
    \left(
        \frac{\bm{A}^\top\bm{A}}{\gamma^2}
        +\frac{\bm{I}}{\sigma_{t_k}^2}
    \right)
    \bm{a}_0
    \\
    &\qquad
    -\bm{a}_0^\top
    \left(
        \frac{\bm{A}^\top\bm{u}_{\mathrm{obs}}}{\gamma^2}
        +\frac{D_{\bm{\theta}}(\bm{a}_{t_k}^{(D)},t_k)}{\sigma_{t_k}^2}
    \right)
    +\frac{\|\bm{u}_{\mathrm{obs}}\|^2}{2\gamma^2}
    +\frac{\|D_{\bm{\theta}}(\bm{a}_{t_k}^{(D)},t_k)\|^2}{2\sigma_{t_k}^2}
    \\
    &=
    \frac12\bm{a}_0^\top\bm{\Sigma}_{\mathrm D,t_k}^{-1}\bm{a}_0
    -\bm{a}_0^\top
    \left(
        \frac{\bm{A}^\top\bm{u}_{\mathrm{obs}}}{\gamma^2}
        +\frac{D_{\bm{\theta}}(\bm{a}_{t_k}^{(D)},t_k)}{\sigma_{t_k}^2}
    \right)
    +\frac{\|\bm{u}_{\mathrm{obs}}\|^2}{2\gamma^2}
    +\frac{\|D_{\bm{\theta}}(\bm{a}_{t_k}^{(D)},t_k)\|^2}{2\sigma_{t_k}^2}
    \\
    &=
    \frac12 \bm{a}_0^\top\bm{\Sigma}_{\mathrm D,t_k}^{-1}\bm{a}_0
    -\bm{a}_0^\top\bm{\Sigma}_{\mathrm D,t_k}^{-1}\bar{\bm{a}}_{\mathrm D,t_{k-1}}
    +\frac{\|\bm{u}_{\mathrm{obs}}\|^2}{2\gamma^2}
    +\frac{\|D_{\bm{\theta}}(\bm{a}_{t_k}^{(D)},t_k)\|^2}{2\sigma_{t_k}^2}.
\end{aligned}
\label{eq:ddis-energy-expansion}
\end{equation}

Symmetry of $\bm{\Sigma}_{\mathrm D,t_k}^{-1}$ gives
\begin{equation}
\begin{aligned}
    &\frac12(\bm{a}_0-\bar{\bm{a}}_{\mathrm D,t_{k-1}})^\top
        \bm{\Sigma}_{\mathrm D,t_k}^{-1}
        (\bm{a}_0-\bar{\bm{a}}_{\mathrm D,t_{k-1}})
    \\
    &=
    \frac12
    (\bm{a}_0^\top-\bar{\bm{a}}_{\mathrm D,t_{k-1}}^\top)
    \bm{\Sigma}_{\mathrm D,t_k}^{-1}
    (\bm{a}_0-\bar{\bm{a}}_{\mathrm D,t_{k-1}})
    \\
    &=
    \frac12\Big(
        \bm{a}_0^\top\bm{\Sigma}_{\mathrm D,t_k}^{-1}\bm{a}_0
        -\bm{a}_0^\top\bm{\Sigma}_{\mathrm D,t_k}^{-1}\bar{\bm{a}}_{\mathrm D,t_{k-1}}
        -\bar{\bm{a}}_{\mathrm D,t_{k-1}}^\top
            \bm{\Sigma}_{\mathrm D,t_k}^{-1}\bm{a}_0
        +\bar{\bm{a}}_{\mathrm D,t_{k-1}}^\top
            \bm{\Sigma}_{\mathrm D,t_k}^{-1}\bar{\bm{a}}_{\mathrm D,t_{k-1}}
    \Big)
    \\
    &=
    \frac12\Big(
        \bm{a}_0^\top\bm{\Sigma}_{\mathrm D,t_k}^{-1}\bm{a}_0
        -2\bm{a}_0^\top
            \bm{\Sigma}_{\mathrm D,t_k}^{-1}\bar{\bm{a}}_{\mathrm D,t_{k-1}}
        +\bar{\bm{a}}_{\mathrm D,t_{k-1}}^\top
            \bm{\Sigma}_{\mathrm D,t_k}^{-1}\bar{\bm{a}}_{\mathrm D,t_{k-1}}
    \Big)
    \\
    &=
    \frac12\bm{a}_0^\top\bm{\Sigma}_{\mathrm D,t_k}^{-1}\bm{a}_0
    -\bm{a}_0^\top\bm{\Sigma}_{\mathrm D,t_k}^{-1}\bar{\bm{a}}_{\mathrm D,t_{k-1}}
    +\frac12\bar{\bm{a}}_{\mathrm D,t_{k-1}}^\top
        \bm{\Sigma}_{\mathrm D,t_k}^{-1}\bar{\bm{a}}_{\mathrm D,t_{k-1}}.
\end{aligned}
\label{eq:ddis-square-expansion}
\end{equation}
Define 
\begin{equation}
    c_{\mathrm D,t_k}
    :=
    \frac{\|\bm{u}_{\mathrm{obs}}\|^2}{2\gamma^2}
    +\frac{\|D_{\bm{\theta}}(\bm{a}_{t_k}^{(D)},t_k)\|^2}{2\sigma_{t_k}^2}
    -\frac12\bar{\bm{a}}_{\mathrm D,t_{k-1}}^\top
        \bm{\Sigma}_{\mathrm D,t_k}^{-1}\bar{\bm{a}}_{\mathrm D,t_{k-1}}.
    \label{eq:ddis-constant-remainder}
\end{equation}
Incorporating \cref{eq:ddis-square-expansion} into \cref{eq:ddis-energy-expansion}
yields
\begin{equation}
    \frac{\|\bm{A}\bm{a}_0-\bm{u}_{\mathrm{obs}}\|^2}{2\gamma^2}
    +\frac{\|\bm{a}_0-D_{\bm{\theta}}(\bm{a}_{t_k}^{(D)},t_k)\|^2}{2\sigma_{t_k}^2}
    =
    \frac12(\bm{a}_0-\bar{\bm{a}}_{\mathrm D,t_{k-1}})^\top
        \bm{\Sigma}_{\mathrm D,t_k}^{-1}(\bm{a}_0-\bar{\bm{a}}_{\mathrm D,t_{k-1}})
    +c_{\mathrm D,t_k}.
    \label{eq:ddis-completed-square}
\end{equation}
The scalar $c_{\mathrm D,t_k}$ is independent of $\bm{a}_0$.
Since \cref{eq:ddis-completed-square} holds for every $\bm{a}_0\in\R^n$,
it also holds with the integration variable $\bm{x}$ in place of $\bm{a}_0$,
with the same scalar $c_{\mathrm D,t_k}$. Consequently,
\cref{eq:ddis-full-density-ratio} becomes
\begin{equation}
\begin{aligned}
    q_{D,t_k}(\bm{a}_0\mid \bm{a}_{t_k}^{(D)},\bm{u}_{\mathrm{obs}})
    &=
    \frac{
        e^{-c_{\mathrm D,t_k}}
        \exp\!\left\{
            -\frac12(\bm{a}_0-\bar{\bm{a}}_{\mathrm D,t_{k-1}})^\top
            \bm{\Sigma}_{\mathrm D,t_k}^{-1}(\bm{a}_0-\bar{\bm{a}}_{\mathrm D,t_{k-1}})
        \right\}
    }{
        e^{-c_{\mathrm D,t_k}}
        \displaystyle\int_{\R^n}
        \exp\!\left\{
            -\frac12(\bm{x}-\bar{\bm{a}}_{\mathrm D,t_{k-1}})^\top
            \bm{\Sigma}_{\mathrm D,t_k}^{-1}(\bm{x}-\bar{\bm{a}}_{\mathrm D,t_{k-1}})
        \right\}\,d\bm{x}
    }
    \\
    &=
    \frac{
        \exp\!\left\{
            -\frac12(\bm{a}_0-\bar{\bm{a}}_{\mathrm D,t_{k-1}})^\top
            \bm{\Sigma}_{\mathrm D,t_k}^{-1}(\bm{a}_0-\bar{\bm{a}}_{\mathrm D,t_{k-1}})
        \right\}
    }{
        \displaystyle\int_{\R^n}
        \exp\!\left\{
            -\frac12(\bm{x}-\bar{\bm{a}}_{\mathrm D,t_{k-1}})^\top
            \bm{\Sigma}_{\mathrm D,t_k}^{-1}(\bm{x}-\bar{\bm{a}}_{\mathrm D,t_{k-1}})
        \right\}\,d\bm{x}
    }.
\end{aligned}
\label{eq:ddis-constant-cancellation}
\end{equation}

To evaluate the remaining integral, the spectral theorem provides
the symmetric positive definite square root
$\bm{\Sigma}_{\mathrm D,t_k}^{1/2}$. Define $\bm{z} = \bm{\Sigma}^{-\frac{1}{2}}(\bm{x} - \bar{\bm{a}}_{D, t_{k-1}})$, i.e., $\bm{x} = \bar{\bm{a}}_{D, t_{k-1}} + \bm{\Sigma}^{\frac{1}{2}} \bm{z}$. The map
$\bm{z}\mapsto\bar{\bm{a}}_{\mathrm D,t_{k-1}}+\bm{\Sigma}_{\mathrm D,t_k}^{1/2}\bm{z}$
is a continuously differentiable bijection of $\R^n$ with
continuously differentiable inverse and Jacobian determinant
$(\det\bm{\Sigma}_{\mathrm D,t_k})^{1/2}>0$.
The change of variables theorem therefore gives
\begin{equation}
\begin{aligned}
    &\int_{\R^n}
    \exp\!\left\{
        -\frac12(\bm{x}-\bar{\bm{a}}_{\mathrm D,t_{k-1}})^\top
        \bm{\Sigma}_{\mathrm D,t_k}^{-1}(\bm{x}-\bar{\bm{a}}_{\mathrm D,t_{k-1}})
    \right\}\,d\bm{x}
    \\
    &=
    (\det\bm{\Sigma}_{\mathrm D,t_k})^{1/2}
    \int_{\R^n}e^{-\|\bm{z}\|^2/2}\,d\bm{z}
    \\
    &=
    (\det\bm{\Sigma}_{\mathrm D,t_k})^{1/2}(2\pi)^{n/2},
\end{aligned}
\label{eq:ddis-gaussian-integral}
\end{equation}
by Tonelli's theorem and the
 Gaussian integral
$\int_{\R}e^{-r^2/2}\,dr=\sqrt{2\pi}$.
Substituting \cref{eq:ddis-gaussian-integral} into
\cref{eq:ddis-constant-cancellation} yields
\begin{equation}
    q_{D,t_k}(\bm{a}_0\mid \bm{a}_{t_k}^{(D)},\bm{u}_{\mathrm{obs}})
    =
    \frac{
        \exp\!\left\{
            -\frac12(\bm{a}_0-\bar{\bm{a}}_{\mathrm D,t_{k-1}})^\top
            \bm{\Sigma}_{\mathrm D,t_k}^{-1}(\bm{a}_0-\bar{\bm{a}}_{\mathrm D,t_{k-1}})
        \right\}
    }{
        (2\pi)^{n/2}(\det\bm{\Sigma}_{\mathrm D,t_k})^{1/2}
    }.
    \label{eq:ddis-normalized-density}
\end{equation}
This is exactly the density of
$\mathcal N(\bar{\bm{a}}_{\mathrm D,t_{k-1}},\bm{\Sigma}_{\mathrm D,t_k})$, as desired. 

\end{proof}

\begin{lemma}[Orthogonal projection onto $\ker \bm{A}$]
\label{lemma:projection-identities}
For every $\bm{u}\in\R^m$ and $\bm{a}\in\R^n$,
\begin{equation}
    \bm{P}_{\ker \bm{A}}(\bm{A}^\top \bm{u})=\bm{0},
    \qquad
    \|\bm{P}_{\ker \bm{A}}\bm{a}\|_2\le\|\bm{a}\|_2.
    \label{eq:projection-identities}
\end{equation}
\end{lemma}

\begin{proof}
For every $\bm{u}\in\R^m$ and $\bm{z}\in\ker \bm{A}$,
\begin{equation}
    \langle \bm{A}^\top \bm{u},\bm{z}\rangle=\langle \bm{u},\bm{A}\bm{z}\rangle=0.
\end{equation}
Thus $\bm{A}^\top \bm{u}\in(\ker \bm{A})^\perp$, so
$\bm{P}_{\ker \bm{A}}(\bm{A}^\top \bm{u})=\bm{0}$.

For every $\bm{a}\in\R^n$, the vectors $\bm{P}_{\ker \bm{A}}\bm{a}$ and
$\bm{a}-\bm{P}_{\ker \bm{A}}\bm{a}$ are orthogonal with $\bm{a} = \bm{P}_{\ker \bm{A}}\bm{a} + (\bm{I}-\bm{P}_{\ker \bm{A}})\bm{a}$. Hence
\begin{equation}
    \|\bm{a}\|_2^2
    =\|\bm{P}_{\ker \bm{A}}\bm{a}\|_2^2+\|\bm{a}-\bm{P}_{\ker \bm{A}}\bm{a}\|_2^2
    \ge\|\bm{P}_{\ker \bm{A}}\bm{a}\|_2^2,
\end{equation}
which gives $\|\bm{P}_{\ker \bm{A}}\bm{a}\|_2\le\|\bm{a}\|_2$ as desired. 
\end{proof}

\begin{lemma}[Exact denoiser under $\pi_0$]
\label{lemma:two-point-denoiser}
For every $\bm{v}\in\R^n$,
\begin{equation}
 D^*(\bm{v},t_k)=\bm{a}_-+\mathbb{P}(\bm{a}_0=\bm{a}_+\mid \bm{a}_{t_k}=\bm{v})(\bm{a}_+-\bm{a}_-),
 \label{eq:two-point-denoiser}
\end{equation}
where
\begin{equation}
 \mathbb{P}(\bm{a}_0=\bm{a}_+\mid \bm{a}_{t_k}=\bm{v})
 =\left(1+\exp\left\{
 -\frac{(\bm{a}_+-\bm{a}_-)^\top \bm{K}^{-1}(\bm{v}-(\bm{a}_-+\bm{a}_+)/2)}{\sigma_{t_k}^2}
 \right\}\right)^{-1}
 \label{eq:denoiser-weight}
\end{equation}
is well-defined. 
\end{lemma}
\begin{proof}
For $\bm{v}\in\R^n$, by definition of the diffusion forward process, we have that 
\begin{equation}
 p(\bm{a}_{t_k}=\bm{v}\mid \bm{a}_0=\bm{a}_\pm)
 =\frac{\exp\{-(\bm{v}-\bm{a}_\pm)^\top \bm{K}^{-1}(\bm{v}-\bm{a}_\pm)/(2\sigma_{t_k}^2)\}}
 {(2\pi)^{n/2}\sigma_{t_k}^n(\det \bm{K})^{1/2}}>0.
 \label{eq:lemma-component-densities}
\end{equation}

By the law of total probability,
\begin{equation}
 p(\bm{a}_{t_k}=\bm{v})=\tfrac12p(\bm{a}_{t_k}=\bm{v}\mid \bm{a}_0=\bm{a}_-)
          +\tfrac12p(\bm{a}_{t_k}=\bm{v}\mid \bm{a}_0=\bm{a}_+)>0.
 \label{eq:noisy-marginal-density}
\end{equation}
Define
\begin{equation}
 k_{t_k}(\bm{v}):=\frac{\tfrac12p(\bm{a}_{t_k}=\bm{v}\mid \bm{a}_0=\bm{a}_+)}{p(\bm{a}_{t_k}=\bm{v})}.
 \label{eq:conditional-weight-definition}
\end{equation}
This function is Borel measurable and takes values in $(0,1)$.
Thus $k_{t_k}(\bm{a}_{t_k})$ is integrable and $\sigma(\bm{a}_{t_k})$-measurable.
For every Borel set $B\subseteq\R^n$,
\begin{equation}
\begin{aligned}
 \E[k_{t_k}(\bm{a}_{t_k})\mathbf1_{\{\bm{a}_{t_k}\in B\}}]
 &=\int_B k_{t_k}(\bm{v})p(\bm{a}_{t_k}=\bm{v})\,d\bm{v}\\
 &=\frac12\int_B p(\bm{a}_{t_k}=\bm{v}\mid \bm{a}_0=\bm{a}_+)\,d\bm{v}\\
 &= \mathbb{P}(\bm{a}_0 = \bm{a}_+)\mathbb{P}(\bm{a}_{t_k} \in B | \bm{a}_0 = \bm{a}_+)\\
 &=\mathbb{P}(\bm{a}_0=\bm{a}_+,\bm{a}_{t_k}\in B)\\
 &=\E[\mathbf1_{\{\bm{a}_0=\bm{a}_+\}}\mathbf1_{\{\bm{a}_{t_k}\in B\}}].
\end{aligned}
\label{eq:conditional-weight-verification}
\end{equation}

Since the events $\{\bm{a}_{t_k}\in B\}$ form $\sigma(\bm{a}_{t_k})$, by the definition of conditional expectation this verifies
$k_{t_k}(\bm{a}_{t_k})
=\E[\mathbf{1}_{\{\bm{a}_0=\bm{a}_+\}}\mid \sigma(\bm{a}_{t_k})]
=\mathbb{P}(\bm{a}_0=\bm{a}_+\mid \sigma(\bm{a}_{t_k}))$
almost surely. Define $Q(\cdot,\cdot)$ by
$Q(\bm{v},C)
=
k_{t_k}(\bm{v})\mathbf{1}_{\{\bm{a}_+\in C\}}
+
(1-k_{t_k}(\bm{v}))\mathbf{1}_{\{\bm{a}_-\in C\}}$
for $\bm{v}\in\mathbb{R}^n$ and Borel set $C\subset\mathbb{R}^n$.
For every fixed $\bm{v}$, the map $C\mapsto Q(\bm{v},C)$ is a probability measure on $\mathbb{R}^n$, and for every Borel set $C\subset\mathbb{R}^n$, the map $\bm{v}\mapsto Q(\bm{v},C)$ is Borel measurable. Moreover, since $\bm{a}_0$ takes values only in $\{\bm{a}_+,\bm{a}_-\}$,
\begin{align*}
\E[\mathbf{1}_{\{\bm{a}_0\in C\}}\mid \sigma(\bm{a}_{t_k})]
&=
\mathbf{1}_{\{\bm{a}_+\in C\}}
\E[\mathbf{1}_{\{\bm{a}_0=\bm{a}_+\}}\mid \sigma(\bm{a}_{t_k})]
\\
&\qquad
+
\mathbf{1}_{\{\bm{a}_-\in C\}}
\E[\mathbf{1}_{\{\bm{a}_0=\bm{a}_-\}}\mid \sigma(\bm{a}_{t_k})]
\\
&=
\mathbf{1}_{\{\bm{a}_+\in C\}}k_{t_k}(\bm{a}_{t_k})
+
\mathbf{1}_{\{\bm{a}_-\in C\}}(1-k_{t_k}(\bm{a}_{t_k}))
\\
&=
Q(\bm{a}_{t_k},C)
\end{align*}
almost surely, since $\bm{1}_{\{\bm{a}_0 = \bm{a}_-\}} = 1 - \bm{1}_{\{\bm{a}_0 = \bm{a}_+\}}$. Hence $Q$ is a regular conditional distribution of $\bm{a}_0$ given $\bm{a}_{t_k}$. In particular,
$\mathbb{P}(\bm{a}_0=\bm{a}_+\mid \bm{a}_{t_k}=\bm{v})
:=Q(\bm{v},\{\bm{a}_+\})
=k_{t_k}(\bm{v})$ is well-defined
for every $\bm{v}\in\mathbb{R}^n$.

Incorporating \cref{eq:lemma-component-densities,eq:noisy-marginal-density} into
\cref{eq:conditional-weight-definition} yields 
\begin{equation}
\begin{aligned}
 k_{t_k}(\bm{v})
 &=\left(1+\frac{p(\bm{a}_{t_k}=\bm{v}\mid \bm{a}_0=\bm{a}_-)}{p(\bm{a}_{t_k}=\bm{v}\mid \bm{a}_0=\bm{a}_+)}\right)^{-1}\\
 &=\left(1+\exp\left\{
 \frac{(\bm{v}-\bm{a}_+)^\top \bm{K}^{-1}(\bm{v}-\bm{a}_+)
       -(\bm{v}-\bm{a}_-)^\top \bm{K}^{-1}(\bm{v}-\bm{a}_-)}{2\sigma_{t_k}^2}
 \right\}\right)^{-1}.
\end{aligned}
\label{eq:conditional-weight-bayes}
\end{equation}

By symmetry of $\bm{K}^{-1}$,
\begin{equation}
\begin{aligned}
    &(\bm{v}-\bm{a}_+)^\top \bm{K}^{-1}(\bm{v}-\bm{a}_+)
    -(\bm{v}-\bm{a}_-)^\top \bm{K}^{-1}(\bm{v}-\bm{a}_-)
    \\
    &=
    \left(
        \bm{v}^\top\bm{K}^{-1}\bm{v}
        -\bm{v}^\top\bm{K}^{-1}\bm{a}_+
        -\bm{a}_+^\top\bm{K}^{-1}\bm{v}
        +\bm{a}_+^\top\bm{K}^{-1}\bm{a}_+
    \right)
    \\
    &\qquad
    -
    \left(
        \bm{v}^\top\bm{K}^{-1}\bm{v}
        -\bm{v}^\top\bm{K}^{-1}\bm{a}_-
        -\bm{a}_-^\top\bm{K}^{-1}\bm{v}
        +\bm{a}_-^\top\bm{K}^{-1}\bm{a}_-
    \right)
    \\
    &=
    -2\bm{v}^\top\bm{K}^{-1}\bm{a}_+
    +2\bm{v}^\top\bm{K}^{-1}\bm{a}_-
    +\bm{a}_+^\top\bm{K}^{-1}\bm{a}_+
    -\bm{a}_-^\top\bm{K}^{-1}\bm{a}_-
    \\
    &=
    -2\bm{v}^\top\bm{K}^{-1}(\bm{a}_+-\bm{a}_-)
    +\bm{a}_+^\top\bm{K}^{-1}\bm{a}_+
    -\bm{a}_-^\top\bm{K}^{-1}\bm{a}_-
    \\
    &=
    -2\bm{v}^\top\bm{K}^{-1}(\bm{a}_+-\bm{a}_-)
    +(\bm{a}_+-\bm{a}_-)^\top
        \bm{K}^{-1}(\bm{a}_++\bm{a}_-)
    \\
    &=
    -2(\bm{a}_+-\bm{a}_-)^\top\bm{K}^{-1}\bm{v}
    +(\bm{a}_+-\bm{a}_-)^\top
        \bm{K}^{-1}(\bm{a}_++\bm{a}_-)
    \\
    &=
    -2(\bm{a}_+-\bm{a}_-)^\top\bm{K}^{-1}
    \left(
        \bm{v}-\frac{\bm{a}_-+\bm{a}_+}{2}
    \right),
\end{aligned}
\label{eq:two-point-log-odds}
\end{equation}
proving \cref{eq:denoiser-weight}. 

Finally, since $\bm{a}_0$ takes only
the values $\bm{a}_-,\bm{a}_+$,
\begin{equation}
\begin{aligned}
 D^*(\bm{v},t_k)&=\E[\bm{a}_0\mid \bm{a}_{t_k}=\bm{v}]\\
 &=(1-k_{t_k}(\bm{v}))\bm{a}_-+k_{t_k}(\bm{v})\bm{a}_+\\
 &=\bm{a}_-+k_{t_k}(\bm{v})(\bm{a}_+-\bm{a}_-).
\end{aligned}
\label{eq:conditional-mean-expansion}
\end{equation}

Substitute $k_{t_k}(\bm{v})=\mathbb{P}(\bm{a}_0=\bm{a}_+\mid \bm{a}_{t_k}=\bm{v})$ to obtain
\cref{lemma:two-point-denoiser} as desired. 
\end{proof}

\begin{lemma}[Strict positivity of the conditional weights]
\label{lemma:positive-weights}
For every $\bm{v}\in\R^n$,
\begin{equation}
 0<\mathbb{P}(\bm{a}_0=\bm{a}_+\mid \bm{a}_{t_k}=\bm{v})<1,\qquad
 0<\mathbb{P}(\bm{a}_0=\bm{a}_-\mid \bm{a}_{t_k}=\bm{v})<1.
 \label{eq:positive-weights}
\end{equation}
\end{lemma}
\begin{proof}
By \cref{eq:lemma-component-densities}, both conditional densities
$p(\bm{a}_{t_k}=\bm{v}\mid \bm{a}_0=\bm{a}_+)$ and $p(\bm{a}_{t_k}=\bm{v}\mid \bm{a}_0=\bm{a}_-)$ are strictly positive.
Thus
\begin{equation}
 0<\frac{p(\bm{a}_{t_k}=\bm{v}\mid \bm{a}_0=\bm{a}_\pm)}
 {p(\bm{a}_{t_k}=\bm{v}\mid \bm{a}_0=\bm{a}_+)+p(\bm{a}_{t_k}=\bm{v}\mid \bm{a}_0=\bm{a}_-)}<1.
\end{equation}
By \cref{eq:conditional-weight-definition,eq:noisy-marginal-density},
these ratios are the two conditional probabilities, yielding \cref{eq:positive-weights} as desired. 
\end{proof}

\begin{lemma}[Global existence and uniqueness of the gradient flow]
\label{lemma:sitcom-existence}
For any current state $\bm{a}_{t_k}^{(B)}\in\R^n$, with the exact denoiser $D_{\bm{\theta}}(\cdot, \cdot) = D^*(\cdot, \cdot)$, the gradient flow in
\cref{eq:sitcom-input-flow}, initialized by
$\bm{v}(0)=\bm{a}_{t_k}^{(B)}$, has a solution $\bm{v}(s)$ that exists and is
unique for all $s\ge0$.
\end{lemma}

\begin{proof}

Throughout, set
\begin{equation}
    \bm{b}:=\frac{\bm{K}^{-1}(\bm{a}_+-\bm{a}_-)}{\sigma_{t_k}^2},
    \qquad
    c:=\frac{\|\bm{A}(\bm{a}_+-\bm{a}_-)\|_2^2}{\gamma^2},
    \qquad
    \ell_{t_k}(\bm{v})
    :=\bm{b}^\top
    \left(
        \bm{v}-\frac{\bm{a}_-+\bm{a}_+}{2}
    \right).
    \label{eq:flow-proof-notation}
\end{equation}

By \cref{lemma:two-point-denoiser},
\begin{equation}
    k_{t_k}(\bm{v})
    =
    \left(1+e^{-\ell_{t_k}(\bm{v})}\right)^{-1}.
\end{equation}
Since $\nabla_{\bm{v}}\ell_{t_k}(\bm{v})=\bm{b}$, differentiation gives
\begin{align*}
    \nabla_{\bm{v}}k_{t_k}(\bm{v})
    &=
    -\left(1+e^{-\ell_{t_k}(\bm{v})}\right)^{-2}
    \left(-e^{-\ell_{t_k}(\bm{v})}\bm{b}\right) \\
    &=
    \frac{e^{-\ell_{t_k}(\bm{v})}}
    {\left(1+e^{-\ell_{t_k}(\bm{v})}\right)^2}\bm{b}.
\end{align*}
Moreover,
\begin{equation}
    k_{t_k}(\bm{v})
    =
    \frac{1}{1+e^{-\ell_{t_k}(\bm{v})}},
    \qquad
    1-k_{t_k}(\bm{v})
    =
    \frac{e^{-\ell_{t_k}(\bm{v})}}
    {1+e^{-\ell_{t_k}(\bm{v})}}.
\end{equation}
Therefore,
\begin{equation}
    \nabla_{\bm{v}}k_{t_k}(\bm{v})
    =
    k_{t_k}(\bm{v})
    \bigl(1-k_{t_k}(\bm{v})\bigr)\bm{b}.
    \label{eq:denoiser-weight-gradient}
\end{equation}

Using $\bm{u}_{\mathrm{obs}}=\bm{A}\bm{a}_+$ and
\cref{lemma:two-point-denoiser},
\begin{equation}
\begin{aligned}
    \bm{A}D^*(\bm{v},t_k)-\bm{u}_{\mathrm{obs}}
    &=
    \bm{A}\bm{a}_-
    +k_{t_k}(\bm{v})\bm{A}(\bm{a}_+-\bm{a}_-)
    -\bm{A}\bm{a}_+ \\
    &=
    -(1-k_{t_k}(\bm{v}))
    \bm{A}(\bm{a}_+-\bm{a}_-).
\end{aligned}
\label{eq:sitcom-measurement-residual}
\end{equation}

Hence the measurement objective is
\begin{equation}
    J_{t_k}(\bm{v})
    :=
    \frac{
        \|\bm{A}D^*(\bm{v},t_k)-\bm{u}_{\mathrm{obs}}\|_2^2
    }{2\gamma^2}
    =
    \frac{c}{2}
    \bigl(1-k_{t_k}(\bm{v})\bigr)^2.
    \label{eq:sitcom-scalar-objective}
\end{equation}

The chain rule and \cref{eq:denoiser-weight-gradient} yield
\begin{align}
    \nabla_{\bm{v}}J_{t_k}(\bm{v})
    &=
    \frac{c}{2}\,
    2\bigl(1-k_{t_k}(\bm{v})\bigr)
    \bigl(-\nabla_{\bm{v}}k_{t_k}(\bm{v})\bigr) \notag\\
    &=
    -c\,k_{t_k}(\bm{v})
    \bigl(1-k_{t_k}(\bm{v})\bigr)^2\bm{b}.
    \label{eq:sitcom-objective-gradient}
\end{align}

Substituting the definitions of $\bm{b}$ and $c$ into
$\bm{v}'(s)=-\nabla J_{t_k}(\bm{v}(s))$ gives
\begin{equation}
\begin{aligned}
    \bm{v}'(s)
    &=
    \frac{
        \|\bm{A}(\bm{a}_+-\bm{a}_-)\|_2^2
    }{\gamma^2\sigma_{t_k}^2}
    \mathbb{P}
    \bigl(
        \bm{a}_0=\bm{a}_+
        \mid
        \bm{a}_{t_k}=\bm{v}(s)
    \bigr) \\
    &\hspace{2cm}\times
    \mathbb{P}
    \bigl(
        \bm{a}_0=\bm{a}_-
        \mid
        \bm{a}_{t_k}=\bm{v}(s)
    \bigr)^2
    \bm{K}^{-1}(\bm{a}_+-\bm{a}_-), \\
    \bm{v}(0)
    &=
    \bm{a}_{t_k}^{(B)}.
\end{aligned}
\label{eq:sitcom-explicit-input-flow}
\end{equation}

It remains to show that the vector field
$-\nabla J_{t_k}$ is globally Lipschitz. To compute the Hessian, fix
$i,j\in\{1,\ldots,n\}$. By \cref{eq:sitcom-objective-gradient},
\begin{equation}
    \frac{\partial J_{t_k}}{\partial v_i}(\bm{v})
    =
    -c\,b_i\,k_{t_k}(\bm{v})
    \bigl(1-k_{t_k}(\bm{v})\bigr)^2.
\end{equation}
Since $b_i$ and $c$ are constant with respect to $\bm{v}$, the product
rule gives
\begin{align*}
    \frac{\partial}{\partial v_j}
    \frac{\partial J_{t_k}}{\partial v_i}(\bm{v})
    &=
    -c\,b_i
    \frac{\partial}{\partial v_j}
    \left[
        k_{t_k}(\bm{v})
        \bigl(1-k_{t_k}(\bm{v})\bigr)^2
    \right] \\
    &=
    -c\,b_i
    \left[
        \frac{\partial k_{t_k}}{\partial v_j}(\bm{v})
        \bigl(1-k_{t_k}(\bm{v})\bigr)^2
        +
        k_{t_k}(\bm{v})
        \frac{\partial}{\partial v_j}
        \bigl(1-k_{t_k}(\bm{v})\bigr)^2
    \right].
\end{align*}
By the chain rule,
\begin{equation}
    \frac{\partial}{\partial v_j}
    \bigl(1-k_{t_k}(\bm{v})\bigr)^2
    =
    -2\bigl(1-k_{t_k}(\bm{v})\bigr)
    \frac{\partial k_{t_k}}{\partial v_j}(\bm{v}).
\end{equation}
Substituting this identity gives
\begin{align*}
    \frac{\partial}{\partial v_j}
    \frac{\partial J_{t_k}}{\partial v_i}(\bm{v})
    &=
    -c\,b_i
    \left[
        \bigl(1-k_{t_k}(\bm{v})\bigr)^2
        -
        2k_{t_k}(\bm{v})
        \bigl(1-k_{t_k}(\bm{v})\bigr)
    \right]
    \frac{\partial k_{t_k}}{\partial v_j}(\bm{v}) \\
    &=
    -c\,b_i
    \bigl(1-k_{t_k}(\bm{v})\bigr)
    \bigl(1-3k_{t_k}(\bm{v})\bigr)
    \frac{\partial k_{t_k}}{\partial v_j}(\bm{v}).
\end{align*}

By \cref{eq:denoiser-weight-gradient},
\begin{equation}
    \frac{\partial k_{t_k}}{\partial v_j}(\bm{v})
    =
    k_{t_k}(\bm{v})
    \bigl(1-k_{t_k}(\bm{v})\bigr)b_j.
\end{equation}
Therefore,
\begin{align*}
    \frac{\partial}{\partial v_j}
    \frac{\partial J_{t_k}}{\partial v_i}(\bm{v})
    &=
    -c\,b_i
    \bigl(1-k_{t_k}(\bm{v})\bigr)
    \bigl(1-3k_{t_k}(\bm{v})\bigr)
    k_{t_k}(\bm{v})
    \bigl(1-k_{t_k}(\bm{v})\bigr)b_j \\
    &=
    -c\,k_{t_k}(\bm{v})
    \bigl(1-k_{t_k}(\bm{v})\bigr)^2
    \bigl(1-3k_{t_k}(\bm{v})\bigr)b_i b_j.
\end{align*}

Since
$(\bm{b}\bm{b}^\top)_{ij}=b_i b_j$, we obtain
\begin{equation}
    \nabla_{\bm{v}}^2J_{t_k}(\bm{v})
    =
    -c\,k_{t_k}(\bm{v})
    \bigl(1-k_{t_k}(\bm{v})\bigr)^2
    \bigl(1-3k_{t_k}(\bm{v})\bigr)
    \bm{b}\bm{b}^\top.
    \label{eq:sitcom-hessian}
\end{equation}

By \cref{lemma:positive-weights},
$0<k_{t_k}(\bm{v})<1$. Hence
\begin{equation}
    k_{t_k}(\bm{v})\le1,
    \qquad
    \bigl(1-k_{t_k}(\bm{v})\bigr)^2\le1,
    \qquad
    \left|1-3k_{t_k}(\bm{v})\right|\le2.
\end{equation}
We also have
\begin{equation}
    \|\bm{b}\bm{b}^\top\|_{\mathrm{op}}
    \leq
    \|\bm{b}\|_2^2,
\end{equation}
since for every $\bm{x}\in\R^n$ with $\|\bm{x}\|_2=1$,
Cauchy--Schwarz gives
\begin{align*}
    \|\bm{b}\bm{b}^\top\bm{x}\|_2
    &=
    \|\bm{b}\|_2
    |\bm{b}^\top\bm{x}| \\
    &\le
    \|\bm{b}\|_2^2\|\bm{x}\|_2
    =
    \|\bm{b}\|_2^2. 
\end{align*}

Thus, using \cref{eq:sitcom-hessian},
\begin{align}
    \|\nabla_{\bm{v}}^2J_{t_k}(\bm{v})\|_{\mathrm{op}}
    &\le
    2c\|\bm{b}\|_2^2 \notag\\
    &=
    \frac{
        2
        \|\bm{A}(\bm{a}_+-\bm{a}_-)\|_2^2
        \|\bm{K}^{-1}(\bm{a}_+-\bm{a}_-)\|_2^2
    }{
        \gamma^2\sigma_{t_k}^4
    }
    =:L_{t_k}<\infty
    \label{eq:sitcom-hessian-bound}
\end{align}
for every $\bm{v}\in\R^n$.

Now let $\bm{x},\bm{y}\in\R^n$ and define
\begin{equation}
    \bm{g}(r)
    :=
    \nabla J_{t_k}
    \bigl(\bm{y}+r(\bm{x}-\bm{y})\bigr),
    \qquad r\in[0,1].
\end{equation}
Since $J_{t_k}$ is twice continuously differentiable, the chain rule gives
\begin{equation}
    \bm{g}'(r)
    =
    \nabla^2 J_{t_k}
    \bigl(\bm{y}+r(\bm{x}-\bm{y})\bigr)
    (\bm{x}-\bm{y}).
\end{equation}
The fundamental theorem of calculus therefore gives
\begin{align*}
    \nabla J_{t_k}(\bm{x})-\nabla J_{t_k}(\bm{y})
    &=
    \bm{g}(1)-\bm{g}(0) \\
    &=
    \int_0^1
    \nabla^2J_{t_k}
    \bigl(\bm{y}+r(\bm{x}-\bm{y})\bigr)
    (\bm{x}-\bm{y})\,dr.
\end{align*}
Taking norms and applying \cref{eq:sitcom-hessian-bound},
\begin{align}
    \|\nabla J_{t_k}(\bm{x})
    -\nabla J_{t_k}(\bm{y})\|_2
    &\le
    \int_0^1
    \left\|
        \nabla^2J_{t_k}
        \bigl(\bm{y}+r(\bm{x}-\bm{y})\bigr)
    \right\|_{\mathrm{op}}
    \|\bm{x}-\bm{y}\|_2\,dr \notag\\
    &\le
    \int_0^1
    L_{t_k}\|\bm{x}-\bm{y}\|_2\,dr \notag\\
    &=
    L_{t_k}\|\bm{x}-\bm{y}\|_2.
    \label{eq:sitcom-global-lipschitz}
\end{align}
Thus $-\nabla J_{t_k}$ is globally Lipschitz on $\R^n$. Consequently, the global Picard--Lindel\"of theorem gives a unique solution to \cref{eq:sitcom-input-flow}, as desired. 

\end{proof}

\begin{lemma}[Evolution of the conditional probability under the gradient flow]
\label{lemma:sitcom-scalar-flow}
Suppose $\bm{A}(\bm{a}_+ - \bm{a}_-) \neq 0$. For the solution of the gradient flow in \cref{eq:sitcom-input-flow} initialized by $\bm{v}(0)=\bm{a}_{t_k}^{(B)}$, $\bm{v}(s)$, let $w(s):=k_{t_k}(\bm{v}(s))$. Then, 
\begin{equation}
 \frac{d}{ds}w(s)
 =\frac{\|\bm{A}(\bm{a}_+-\bm{a}_-)\|^2\|\bm{K}^{-1}(\bm{a}_+-\bm{a}_-)\|^2}
 {\gamma^2\sigma_{t_k}^4}
 w(s)^2(1-w(s))^3.
\label{eq:sitcom-weight-flow}
\end{equation}
\end{lemma}
\begin{proof}
 The solution $\bm{v}(s)$ exists and is unique by \cref{lemma:sitcom-existence}. With $\bm{b},c$ as in \cref{eq:flow-proof-notation},
\cref{eq:sitcom-objective-gradient} and the definition of gradient flow give
\begin{equation}
 \bm{v}'(s)=-\nabla J_{t_k}(\bm{v}(s))=cw(s)(1-w(s))^2\bm{b}.
 \label{eq:sitcom-input-vector-field}
\end{equation}
Since $k_{t_k}(\cdot)$ and $\bm{v}(\cdot)$ are continuously differentiable, the chain rule and
\cref{eq:denoiser-weight-gradient} give
\begin{equation}
\begin{aligned}
 w'(s)&=\nabla_{\bm{v}}k_{t_k}(\bm{v}(s))^\top \bm{v}'(s)\\
 &=[w(s)(1-w(s))\bm{b}]^\top[cw(s)(1-w(s))^2\bm{b}]\\
 &=c\|\bm{b}\|^2w(s)^2(1-w(s))^3.
\end{aligned}
\label{eq:sitcom-scalar-chain-rule}
\end{equation}
Substituting the definitions of $\bm{b},c$ gives
\begin{equation}
 w'(s)=Cw(s)^2(1-w(s))^3,\qquad
 C:=\frac{\|\bm{A}(\bm{a}_+-\bm{a}_-)\|^2\|\bm{K}^{-1}(\bm{a}_+-\bm{a}_-)\|^2}
 {\gamma^2\sigma_{t_k}^4}.
 \label{eq:sitcom-scalar-constant}
\end{equation}
The constant is positive because
$\bm{A}(\bm{a}_+-\bm{a}_-)\ne\bm{0}$, $\bm{K}$ is invertible, and
$\gamma,\sigma_{t_k}>0$; this yields \cref{eq:sitcom-weight-flow} as desired.
\end{proof}

\begin{lemma}[Convergence of the conditional probability and denoiser]
\label{lemma:sitcom-convergence}
Suppose $\bm{A}(\bm{a}_+ - \bm{a}_-) \neq 0$. Along the gradient flow in \cref{eq:sitcom-input-flow}, initialized by $\bm{v}(0)=\bm{a}_{t_k}^{(B)}$, as $s\to\infty$,
\begin{equation}
 \mathbb{P}(\bm{a}_0=\bm{a}_+\mid \bm{a}_{t_k}=\bm{v}(s))\longrightarrow1,
 \qquad \mathbb{P}(\bm{a}_0=\bm{a}_-\mid \bm{a}_{t_k}=\bm{v}(s))\longrightarrow0,
 \label{eq:sitcom-weight-limit}
\end{equation}
and
\begin{equation}
 D^*(\bm{v}(s),t_k)\longrightarrow \bm{a}_+.
 \label{eq:sitcom-recovery}
\end{equation}
\end{lemma}
\begin{proof}
Set $w(s):=k_{t_k}(\bm{v}(s))$ and take $C$ from
\cref{eq:sitcom-scalar-constant}. Note that $\bm{v}(s)$ exists and is unique by \cref{lemma:sitcom-existence}. 
By \cref{lemma:positive-weights,lemma:sitcom-scalar-flow},
$0<w(s)<1$ and $w'(s)>0$ for every finite $s$.
Thus $w(\cdot)$ is increasing and bounded above, so its limit $L$ exists
and satisfies $0<w(0)\le L\le1$.
Suppose $L<1$. By continuity of the polynomial in
\cref{eq:sitcom-scalar-constant},
\begin{equation}
 w'(s)\longrightarrow CL^2(1-L)^3=:c_0>0.
 \label{eq:sitcom-positive-derivative-limit}
\end{equation}
Hence there is $S\ge0$ such that $w'(s)\ge c_0/2$ for all $s\ge S$.
The fundamental theorem of calculus then gives
\begin{equation}
 w(s)-w(S)=\int_S^s w'(r)\,dr
 \ge\frac{c_0}{2}(s-S),\qquad s\ge S.
 \label{eq:sitcom-limit-contradiction}
\end{equation}
The right-hand side grows without bound as $s\to\infty$, contradicting
$L<1$. Therefore $L=1$, which also gives $1-w(s)\to0$.

Finally, by \cref{lemma:two-point-denoiser},
\begin{equation}
\begin{aligned}
 D^*(\bm{v}(s),t_k)-\bm{a}_+
 &=\bm{a}_-+w(s)(\bm{a}_+-\bm{a}_-)-\bm{a}_+\\
 &=-(1-w(s))(\bm{a}_+-\bm{a}_-),\\
 \|D^*(\bm{v}(s),t_k)-\bm{a}_+\|
 &=(1-w(s))\|\bm{a}_+-\bm{a}_-\|\longrightarrow0,
\end{aligned}
\label{eq:sitcom-output-error}
\end{equation}
because $1-w(s)\to0$, as desired. 
\end{proof}

\begin{lemma}[DDIS preserves the denoiser's kernel component in expectation]
\label{lemma:ddis-kernel}
Fix $t_k>t_{k-1}$ and any current state
$\bm{a}_{t_k}^{(D)}\in\R^n$. The DDIS conditional mean
$\bar{\bm{a}}_{\mathrm D,t_{k-1}}:=\E[\bm{a}_{t_{k-1}}^{(D)}\mid
\bm{a}_{t_k}^{(D)},\bm{u}_{\mathrm{obs}}]$ from \cref{lemma:renoising} satisfies
\begin{equation}
    \bm{P}_{\ker \bm{A}}\bar{\bm{a}}_{\mathrm D,t_{k-1}}
    =\bm{P}_{\ker \bm{A}}D_{\bm{\theta}}(\bm{a}_{t_k}^{(D)},t_k).
    \label{eq:ddis-proposition-kernel}
\end{equation}
\end{lemma}

\begin{proof}

By \cref{lemma:ddis-projected-mean}, 
\begin{equation}
 \bar{\bm{a}}_{\mathrm D,t_{k-1}}
 =\left(\frac{\bm{A}^\top \bm{A}}{\gamma^2}+\frac{\bm{I}}{\sigma_{t_k}^2}\right)^{-1}
 \left(\frac{\bm{A}^\top \bm{u}_{\mathrm{obs}}}{\gamma^2}
       +\frac{D_{\bm{\theta}}(\bm{a}_{t_k}^{(D)},t_k)}{\sigma_{t_k}^2}\right).
\end{equation}
Multiplying by the invertible matrix
$\bm{A}^\top \bm{A}/\gamma^2+\bm{I}/\sigma_{t_k}^2$ gives
\begin{equation}
 \frac{\bm{A}^\top \bm{A}\bar{\bm{a}}_{\mathrm D,t_{k-1}}}{\gamma^2}
 +\frac{\bar{\bm{a}}_{\mathrm D,t_{k-1}}}{\sigma_{t_k}^2}
 =\frac{\bm{A}^\top \bm{u}_{\mathrm{obs}}}{\gamma^2}
  +\frac{D_{\bm{\theta}}(\bm{a}_{t_k}^{(D)},t_k)}{\sigma_{t_k}^2}.
 \label{eq:ddis-mean-proposition-system}
\end{equation}
Move both terms on the right to the left and factor:
\begin{equation}
\begin{aligned}
 \bm{0}&=\frac{\bm{A}^\top \bm{A}\bar{\bm{a}}_{\mathrm D,t_{k-1}}-\bm{A}^\top \bm{u}_{\mathrm{obs}}}{\gamma^2}
   +\frac{\bar{\bm{a}}_{\mathrm D,t_{k-1}}-D_{\bm{\theta}}(\bm{a}_{t_k}^{(D)},t_k)}{\sigma_{t_k}^2}\\
 &=\frac{\bm{A}^\top(\bm{A}\bar{\bm{a}}_{\mathrm D,t_{k-1}}-\bm{u}_{\mathrm{obs}})}{\gamma^2}
   +\frac{\bar{\bm{a}}_{\mathrm D,t_{k-1}}-D_{\bm{\theta}}(\bm{a}_{t_k}^{(D)},t_k)}{\sigma_{t_k}^2}.
\end{aligned}
\label{eq:ddis-mean-stationarity}
\end{equation}
Apply the linear map $\bm{P}_{\ker \bm{A}}$. By
\cref{lemma:projection-identities},
$\bm{P}_{\ker \bm{A}}\bm{A}^\top \bm{u}=\bm{0}$ for every
$\bm{u}\in\R^m$, in particular for $\bm{u}=\bm{A}\bar{\bm{a}}_{\mathrm D,t_{k-1}}-\bm{u}_{\mathrm{obs}}$.
Consequently,
\begin{equation}
\begin{aligned}
 \bm{0}&=\frac{\bm{P}_{\ker \bm{A}}\bm{A}^\top(\bm{A}\bar{\bm{a}}_{\mathrm D,t_{k-1}}-\bm{u}_{\mathrm{obs}})}{\gamma^2}
   +\frac{\bm{P}_{\ker \bm{A}}\bar{\bm{a}}_{\mathrm D,t_{k-1}}
             -\bm{P}_{\ker \bm{A}}D_{\bm{\theta}}(\bm{a}_{t_k}^{(D)},t_k)}{\sigma_{t_k}^2}\\
 &=\frac{\bm{P}_{\ker \bm{A}}\bar{\bm{a}}_{\mathrm D,t_{k-1}}
             -\bm{P}_{\ker \bm{A}}D_{\bm{\theta}}(\bm{a}_{t_k}^{(D)},t_k)}{\sigma_{t_k}^2}.
\end{aligned}
\label{eq:ddis-projection-stationarity-proof}
\end{equation}
Since $\sigma_{t_k}^2>0$, we have that 
$\bm{P}_{\ker \bm{A}}\bar{\bm{a}}_{\mathrm D,t_{k-1}}=\bm{P}_{\ker \bm{A}}D_{\bm{\theta}}(\bm{a}_{t_k}^{(D)},t_k)$ as desired. 
\end{proof}

\subsection{Proof of \Cref{prop:ddis_fail_prelim}}
\label{proof:proof_of_ddis_fail_prelim}

\begin{proof}
By \cref{lemma:ddis-kernel}, with
$D_{\bm{\theta}}(\cdot,t_k)=D^*(\cdot,t_k)$, and linearity of
$\bm{P}_{\ker\bm{A}}$,
\begin{align}
    \bm{P}_{\ker\bm{A}}
    (\bar{\bm{a}}_{\mathrm D,t_{k-1}}-\bm{a}_+)
    &=
    \bm{P}_{\ker\bm{A}}\bar{\bm{a}}_{\mathrm D,t_{k-1}}
    -\bm{P}_{\ker\bm{A}}\bm{a}_+ \\
    &=
    \bm{P}_{\ker\bm{A}}D^*(\bm{a}_{t_k}^{(D)},t_k)
    -\bm{P}_{\ker\bm{A}}\bm{a}_+ \\
    &=
    \bm{P}_{\ker\bm{A}}
    \left(D^*(\bm{a}_{t_k}^{(D)},t_k)-\bm{a}_+\right).
\end{align}
Taking norms on both sides proves the claim. The same argument also
applies when $t_{k-1}=t_0$.
\end{proof}

\subsection{Proof of \Cref{prop:exact_denoiser}}
\label{proof:proof_of_exact_denoiser}

\begin{proof}
Let
\begin{equation}
    \bm{d}:=\bm{a}_+-\bm{a}_-.
\end{equation}
By \cref{lemma:two-point-denoiser}, $k_{t_k}(\cdot)$ is  a regular conditional probability
$\mathbb{P}(\bm{a}_0=\bm{a}_+\mid\bm{a}_{t_k}=\cdot)$ and is well-defined, and
\begin{equation}
    D^*(\bm{v},t_k)
    =
    \bm{a}_-+k_{t_k}(\bm{v})\bm{d}.
    \label{eq:proof-exact-denoiser-mixture}
\end{equation}
The orthogonal decomposition of $\bm{d}$ with respect to
$\ker\bm{A}$ is
\begin{equation}
    \bm{d}
    =
    \bm{P}_{(\ker\bm{A})^\perp}\bm{d}
    +
    \bm{P}_{\ker\bm{A}}\bm{d}.
\end{equation}
Substituting this identity into
\cref{eq:proof-exact-denoiser-mixture} gives
\begin{equation}
    D^*(\bm{v},t_k)
    =
    \bm{a}_-
    +k_{t_k}(\bm{v})\bm{P}_{(\ker\bm{A})^\perp}\bm{d}
    +k_{t_k}(\bm{v})\bm{P}_{\ker\bm{A}}\bm{d}.
    \label{eq:proof-exact-denoiser-decomposition}
\end{equation}
Finally, applying $\bm{P}_{\ker\bm{A}}$ and using
$\bm{P}_{\ker\bm{A}}\bm{P}_{(\ker\bm{A})^\perp}=\bm{0}$ and
$\bm{P}_{\ker\bm{A}}^2=\bm{P}_{\ker\bm{A}}$ yields
\begin{equation}
    \bm{P}_{\ker\bm{A}}D^*(\bm{v},t_k)
    =
    \bm{P}_{\ker\bm{A}}\bm{a}_-
    +k_{t_k}(\bm{v})\bm{P}_{\ker\bm{A}}\bm{d},
\end{equation}
as desired. 
\end{proof}

\subsection{Proof of \Cref{corollary:ddis-fail}}
\label{proof:proof_of_ddis_fail}

\begin{proof}
Note that the probability is well-defined by \cref{lemma:two-point-denoiser}. Let
\begin{equation}
    \bm{d}:=\bm{a}_+-\bm{a}_-,
    \qquad
    w_D:=k_{t_k}(\bm{a}_{t_k}^{(D)}).
\end{equation}
By \cref{lemma:two-point-denoiser},
\begin{equation}
    D^*(\bm{a}_{t_k}^{(D)},t_k)-\bm{a}_+
    =
    -(1-w_D)\bm{d}.
    \label{eq:proof-ddis-denoiser-error}
\end{equation}
Incorporating this into  \cref{prop:ddis_fail_prelim}, we have that: 
\begin{equation}
\begin{aligned}
    \left\|
        \bm{P}_{\ker\bm{A}}
        (\bar{\bm{a}}_{\mathrm D,t_{k-1}}-\bm{a}_+)
    \right\|_2
    &=
    \left\|
        \bm{P}_{\ker\bm{A}}
        \left(D^*(\bm{a}_{t_k}^{(D)},t_k)-\bm{a}_+\right)
    \right\|_2 \\
    &=
    (1-w_D)
    \|\bm{P}_{\ker\bm{A}}\bm{d}\|_2.
\end{aligned}
\end{equation}
By \cref{lemma:projection-identities}, we then have 
\begin{equation}
    \|\bar{\bm{a}}_{\mathrm D,t_{k-1}}-\bm{a}_+\|_2
    \geq
    \left\|
        \bm{P}_{\ker\bm{A}}
        (\bar{\bm{a}}_{\mathrm D,t_{k-1}}-\bm{a}_+)
    \right\|_2.
\end{equation}
Hence
\begin{equation}
    \|\bar{\bm{a}}_{\mathrm D,t_{k-1}}-\bm{a}_+\|_2
    \geq
    (1-w_D)
    \|\bm{P}_{\ker\bm{A}}\bm{d}\|_2.
\end{equation}
Finally, by \cref{lemma:positive-weights},
$1-w_D>0$, and by assumption
$\bm{P}_{\ker\bm{A}}\bm{d}\neq\bm{0}$. Therefore,
\begin{equation}
    \|\bar{\bm{a}}_{\mathrm D,t_{k-1}}-\bm{a}_+\|_2
    \geq
    (1-w_D)
    \|\bm{P}_{\ker\bm{A}}\bm{d}\|_2
    >0,
\end{equation}
which proves \cref{corollary:ddis-fail}, as desired. 
\end{proof}

\subsection{Proof of \Cref{prop:sitcom_ker_contract}}
\label{proof:proof_of_sitcom_ker_contract}

\begin{proof}
Under the exact denoiser $D_{\bm{\theta}}(\cdot,t_k)=D^*(\cdot,t_k)$,
\cref{lemma:renoising} gives, for every finite $s\geq0$,
\begin{equation}
    \E\!\left[
        \bm{a}_{t_{k-1}}^{(B)}
        \,\middle|\,
        \bm{a}_{t_k}^{(B)},\bm{u}_{\mathrm{obs}}
    \right]
    =D^*(\bm{v}(s),t_k).
\end{equation}
Let
\begin{equation}
    w_S(s):=k_{t_k}(\bm{v}(s)),
    \qquad
    w_D:=k_{t_k}(\bm{a}_{t_k}^{(D)}),
    \qquad
    \bm{d}:=\bm{a}_+-\bm{a}_-.
\end{equation}
By \cref{lemma:two-point-denoiser},
\begin{equation}
\begin{aligned}
    D^*(\bm{v}(s),t_k)-\bm{a}_+
    &=\bm{a}_-+w_S(s)\bm{d}-\bm{a}_+\\
    &=-(1-w_S(s))\bm{d}.
\end{aligned}
\end{equation}
Therefore
\begin{equation}
    \|\bm{P}_{\ker\bm{A}}
       (D^*(\bm{v}(s),t_k)-\bm{a}_+)\|_2
    =
    (1-w_S(s))\|\bm{P}_{\ker\bm{A}}\bm{d}\|_2.
    \label{eq:proof-sitcom-nullspace-error}
\end{equation}

By \cref{lemma:sitcom-scalar-flow,lemma:positive-weights},
$w_S'(s)>0$ for every finite $s\geq0$. Since
$\bm{v}(0)=\bm{a}_{t_k}^{(B)}$, we have
$w_S(s)>w_S(0)=k_{t_k}(\bm{a}_{t_k}^{(B)})$ for every $s>0$.
Applying \cref{eq:proof-sitcom-nullspace-error} at $s$ and at $0$
therefore gives
\begin{equation}
    \|\bm{P}_{\ker\bm{A}}
       (D^*(\bm{v}(s),t_k)-\bm{a}_+)\|_2
    \leq
    \|\bm{P}_{\ker\bm{A}}
       (D^*(\bm{a}_{t_k}^{(B)},t_k)-\bm{a}_+)\|_2,
    \qquad s>0.
\end{equation}

On the other hand, by \cref{lemma:ddis-kernel},
\begin{equation}
    \bm{P}_{\ker\bm{A}}\bar{\bm{a}}_{\mathrm D,t_{k-1}}
    =
    \bm{P}_{\ker\bm{A}}
    D^*(\bm{a}_{t_k}^{(D)},t_k).
\end{equation}
Applying \cref{lemma:two-point-denoiser} at
$\bm{a}_{t_k}^{(D)}$ yields
\begin{equation}
\begin{aligned}
    \bm{P}_{\ker\bm{A}}
    (\bar{\bm{a}}_{\mathrm D,t_{k-1}}-\bm{a}_+)
    &=
    \bm{P}_{\ker\bm{A}}
    \bigl(D^*(\bm{a}_{t_k}^{(D)},t_k)-\bm{a}_+\bigr)\\
    &=
    -(1-w_D)\bm{P}_{\ker\bm{A}}\bm{d}.
\end{aligned}
\end{equation}
Hence
\begin{equation}
    \|\bm{P}_{\ker\bm{A}}
       (\bar{\bm{a}}_{\mathrm D,t_{k-1}}-\bm{a}_+)\|_2
    =
    (1-w_D)\|\bm{P}_{\ker\bm{A}}\bm{d}\|_2.
    \label{eq:proof-ddis-nullspace-error}
\end{equation}

By \cref{lemma:positive-weights}, $0<w_D<1$, so $1-w_D>0$.
By \cref{lemma:sitcom-convergence},
\begin{equation}
    w_S(s)\longrightarrow1
    \qquad\text{as }s\to\infty.
\end{equation}
Therefore there exists $s^*>0$ such that, for every
$s\geq s^*$,
\begin{equation}
    1-w_S(s)\leq1-w_D.
\end{equation}
Multiplying by $\|\bm{P}_{\ker\bm{A}}\bm{d}\|_2$ and using
\cref{eq:proof-sitcom-nullspace-error,eq:proof-ddis-nullspace-error}
gives
\begin{equation}
    \|\bm{P}_{\ker\bm{A}}
       (D^*(\bm{v}(s),t_k)-\bm{a}_+)\|_2
    \leq
    \|\bm{P}_{\ker\bm{A}}
       (\bar{\bm{a}}_{\mathrm D,t_{k-1}}-\bm{a}_+)\|_2,
\end{equation}
as desired. 
\end{proof}

\subsection{Proof of \Cref{thm:sitcom-success}}
\label{proof:proof_of_sitcom_success}

\begin{proof}
By \cref{lemma:renoising}, for every finite $s$,
\begin{equation}
    \pi_s^{(B)}
    =
    \mathcal N\!\left(
        D^*(\bm{v}(s),t_k),
        \sigma_{t_{k-1}}^2\bm{K}
    \right).
\end{equation}
By \cref{lemma:sitcom-convergence},
$D^*(\bm{v}(s),t_k)\to\bm{a}_+$ as $s\to\infty$. Hence, for every
$\bm{\xi}\in\R^n$, the characteristic function of $\pi_s^{(B)}$
satisfies
\begin{equation}
\begin{aligned}
    \widehat{\pi}_s^{(B)}(\bm{\xi})
    &=
    \exp\!\left(
        i\bm{\xi}^\top D^*(\bm{v}(s),t_k)
        -\frac12\sigma_{t_{k-1}}^2
         \bm{\xi}^\top\bm{K}\bm{\xi}
    \right)\\
    &\longrightarrow
    \exp\!\left(
        i\bm{\xi}^\top\bm{a}_+
        -\frac12\sigma_{t_{k-1}}^2
         \bm{\xi}^\top\bm{K}\bm{\xi}
    \right).
\end{aligned}
\end{equation}
The limit is the characteristic function of
$\mathcal N(\bm{a}_+,\sigma_{t_{k-1}}^2\bm{K})$ and is continuous at
$\bm{\xi}=\bm{0}$. Therefore, L\'evy's continuity theorem gives
\begin{equation}
    \pi_s^{(B)}
    \Rightarrow
    \mathcal N(\bm{a}_+,\sigma_{t_{k-1}}^2\bm{K})
    \qquad\text{as }s\to\infty.
\end{equation}

Finally, when $t_{k-1}=t_0$, for every $\bm{\xi}\in\R^n$,
\begin{equation}
    \exp\!\left(
        i\bm{\xi}^\top\bm{a}_+
        -\frac12\sigma_{t_0}^2\bm{\xi}^\top\bm{K}\bm{\xi}
    \right)
    \longrightarrow
    \exp(i\bm{\xi}^\top\bm{a}_+)
    \qquad\text{as } \sigma_{t_0} \to 0.
\end{equation}
The limit is the characteristic function of $\delta_{\bm{a}_+}$, so
another application of L\'evy's continuity theorem yields
\begin{equation}
    \mathcal N(\bm{a}_+,\sigma_{t_0}^2\bm{K})
    \Rightarrow
    \delta_{\bm{a}_+}.
\end{equation}
This proves \cref{thm:sitcom-success}, as desired. 
\end{proof}

\subsection{Proof of \Cref{prop:wrong-mode-slowdown}}
\label{proof:proof_of_wrong_mode_slowdown}

\begin{proof}
Let $\bm{d}:=\bm{a}_+-\bm{a}_-$,
$\bm{m}:=(\bm{a}_-+\bm{a}_+)/2$,
$Q:=\bm{d}^\top\bm{K}^{-1}\bm{d}$, and
\[
    C:=
    \frac{
        \|\bm{A}\bm{d}\|_2^2
        \|\bm{K}^{-1}\bm{d}\|_2^2
    }{\gamma^2}.
\]
Since $\bm{A}\bm{d}\neq 0$ and $\bm{K}$ is positive definite,
we have $Q,C>0$.

Fix a realization of $\bm{a}_{t_{k+1}}^{(B)}$ satisfying
$k_{t_{k+1}}(\bm{a}_{t_{k+1}}^{(B)})\leq 1/2-\delta$, and set
$p:=k_{t_{k+1}}(\bm{a}_{t_{k+1}}^{(B)})$. By
\cref{lemma:two-point-denoiser},
\begin{align}
    D^*(\bm{a}_{t_{k+1}}^{(B)},t_{k+1})
    &=\bm{a}_-+p\bm{d} \\
    &=\bm{m}-\left(\frac12-p\right)\bm{d}.
\end{align}
Therefore, since no optimization is performed at $t_{k+1}$,
\[
    \bm{a}_{t_k}^{(B)}
    =
    \bm{m}
    -
    \left(\frac12-p\right)\bm{d}
    +
    \sigma_{t_k}\bm{\varepsilon}_k^{(B)},
    \qquad
    \bm{\varepsilon}_k^{(B)}\sim\mathcal{N}(\bm{0},\bm{K}).
\]
Conditionally on $\bm{a}_{t_{k+1}}^{(B)}$,
$\bm{d}^\top\bm{K}^{-1}\bm{\varepsilon}_k^{(B)}
\sim\mathcal{N}(0,Q)$, since
\[
    \Var\left(
        \bm{d}^\top\bm{K}^{-1}\bm{\varepsilon}_k^{(B)}
    \right)
    =
    \bm{d}^\top\bm{K}^{-1}
    \bm{K}
    \bm{K}^{-1}\bm{d}
    =Q.
\]
Thus, with
$Z:=\bm{d}^\top\bm{K}^{-1}\bm{\varepsilon}_k^{(B)}/\sqrt{Q}$,
we have $Z\sim\mathcal{N}(0,1)$ and
\begin{equation}
    \bm{d}^\top\bm{K}^{-1}
    (\bm{a}_{t_k}^{(B)}-\bm{m})
    =
    -Q\left(\frac12-p\right)
    +\sigma_{t_k}\sqrt{Q}\,Z.
    \label{eq:prop4-initial-logit}
\end{equation}

Define
$E:=\{Z\leq\delta\sqrt{Q}/(2\sigma_{t_k})\}$.
Since $1/2-p\geq\delta$, on $E$,
\begin{align}
    \bm{d}^\top\bm{K}^{-1}
    (\bm{a}_{t_k}^{(B)}-\bm{m})
    &\leq
    -Q\delta
    +
    \sigma_{t_k}\sqrt{Q}
    \frac{\delta\sqrt{Q}}{2\sigma_{t_k}} \\
    &=
    -\frac{Q\delta}{2}.
    \label{eq:prop4-logit-bound}
\end{align}
For $x>0$, the standard Gaussian Chernoff bound gives
\[
    \Pr(Z>x)
    =
    \Pr(e^{xZ}>e^{x^2})
    \leq
    e^{-x^2}\E[e^{xZ}]
    =
    e^{-x^2/2}.
\]
Taking $x=\delta\sqrt{Q}/(2\sigma_{t_k})$ therefore yields
\begin{equation}
    \Pr\left(
        E
        \,\middle|\,
        \bm{a}_{t_{k+1}}^{(B)}
    \right)
    \geq
    1-
    \exp\left(
        -\frac{Q\delta^2}{8\sigma_{t_k}^2}
    \right).
    \label{eq:prop4-E-prob}
\end{equation}

Now set $w(s):=k_{t_k}(\bm{v}_{t_k}(s))$. Since
$\bm{v}_{t_k}(0)=\bm{a}_{t_k}^{(B)}$, \cref{lemma:two-point-denoiser}
and \cref{eq:prop4-logit-bound} give, on $E$,
\begin{align}
    w(0)
    &=
    \left(
        1+
        \exp\left\{
            -\frac{
                \bm{d}^\top\bm{K}^{-1}
                (\bm{a}_{t_k}^{(B)}-\bm{m})
            }{\sigma_{t_k}^2}
        \right\}
    \right)^{-1} \\
    &\leq
    \left(
        1+
        \exp\left\{
            \frac{Q\delta}{2\sigma_{t_k}^2}
        \right\}
    \right)^{-1}
    <\frac12.
    \label{eq:prop4-w0-bound}
\end{align}

By Lemma~7,
\begin{equation}
    w'(s)
    =
    \frac{C}{\sigma_{t_k}^4}
    w(s)^2(1-w(s))^3.
    \label{eq:prop4-w-ode}
\end{equation}
By Lemma~5, $0<w(s)<1$ for every finite $s$, so
$w'(s)>0$. By Lemma~8, $w(s)\to1$ as $s\to\infty$.
Hence, on $E$, $\tau_{t_k}<\infty$ and, by continuity,
$w(\tau_{t_k})=1/2$.

Integrating \cref{eq:prop4-w-ode} from $0$ to $\tau_{t_k}$ and
using $u=w(s)$ gives
\begin{align}
    \tau_{t_k}
    &=
    \frac{\sigma_{t_k}^4}{C}
    \int_{w(0)}^{1/2}
    \frac{du}{u^2(1-u)^3} \\
    &\geq
    \frac{\sigma_{t_k}^4}{C}
    \int_{w(0)}^{1/2}\frac{du}{u^2} \\
    &=
    \frac{\sigma_{t_k}^4}{C}
    \left(\frac{1}{w(0)}-2\right),
\end{align}
where the inequality follows from $(1-u)^{-3}\geq1$ for
$u\in(0,1/2]$. By \cref{eq:prop4-w0-bound},
\[
    \frac{1}{w(0)}
    \geq
    1+
    \exp\left(
        \frac{Q\delta}{2\sigma_{t_k}^2}
    \right),
\]
so, on $E$,
\[
    \tau_{t_k}
    \geq
    \frac{\sigma_{t_k}^4}{C}
    \left[
        \exp\left(
            \frac{Q\delta}{2\sigma_{t_k}^2}
        \right)-1
    \right].
\]
Consequently, by \cref{eq:prop4-E-prob},
\begin{align}
    &\Pr\left(
        \tau_{t_k}
        \geq
        \frac{\sigma_{t_k}^4}{C}
        \left[
            \exp\left(
                \frac{Q\delta}{2\sigma_{t_k}^2}
            \right)-1
        \right]
        \,\middle|\,
        \bm{a}_{t_{k+1}}^{(B)}
    \right) \\
    &\qquad\geq
    1-
    \exp\left(
        -\frac{Q\delta^2}{8\sigma_{t_k}^2}
    \right),
\end{align}
as desired.
\end{proof}

\subsection{Proof of \Cref{prop:interior-selection-window}}
\label{proof:proof_of_interior_selection_window}

\begin{proof}
Let $\bm{d}:=\bm{a}_+-\bm{a}_-$, 
$\bm{m}:=(\bm{a}_-+\bm{a}_+)/2$, and
$C:=\|\bm{A}\bm{d}\|_2^2\|\bm{K}^{-1}\bm{d}\|_2^2/\gamma^2$.
Since $\bm{A}\bm{d}\neq 0$, $\bm{K}$ is positive definite, and
$\gamma>0$, we have $C>0$.

By \cref{lemma:two-point-denoiser}, for every $\sigma>0$,
\begin{equation}
    k_\sigma(\bm{v})
    =
    \left(
        1+
        \exp\left\{
            -\frac{
                \bm{d}^\top\bm{K}^{-1}(\bm{v}-\bm{m})
            }{\sigma^2}
        \right\}
    \right)^{-1}.
    \label{eq:prop5-k-sigma}
\end{equation}
Set $h:=\bm{d}^\top\bm{K}^{-1}(\bm{v}^\star-\bm{m})$.
By assumption, $k_\sigma(\bm{v}^\star)<1/2$ for every $\sigma>0$.
Using \cref{eq:prop5-k-sigma},
\[
    k_\sigma(\bm{v}^\star)<\frac12
    \iff
    \exp\left(-\frac{h}{\sigma^2}\right)>1
    \iff
    h<0.
\]
Hence, setting $b:=-h>0$,
\begin{equation}
    w_\sigma(0)
    =
    k_\sigma(\bm{v}^\star)
    =
    \frac{1}{1+\exp(b/\sigma^2)}
    \in
    \left(0,\frac12\right).
    \label{eq:prop5-w0}
\end{equation}

Fix $\sigma>0$, and write $\bm{v}_\sigma(s)$ for the solution of
\cref{eq:sitcom-input-flow} at noise level $\sigma$, initialized by
$\bm{v}_\sigma(0)=\bm{v}^\star$. By
\cref{lemma:sitcom-existence}, this solution exists uniquely for all
$s\geq0$. Let $w_\sigma(s):=k_\sigma(\bm{v}_\sigma(s))$.

By \cref{lemma:sitcom-scalar-flow}, with $\sigma$ in place of
$\sigma_{t_k}$,
\begin{equation}
    w_\sigma'(s)
    =
    \frac{C}{\sigma^4}
    w_\sigma(s)^2(1-w_\sigma(s))^3.
    \label{eq:prop5-w-ode}
\end{equation}
By \cref{lemma:positive-weights}, $0<w_\sigma(s)<1$ for every finite
$s$. Since $C>0$, \cref{eq:prop5-w-ode} gives
$w_\sigma'(s)>0$, so $w_\sigma$ is strictly increasing.
Moreover, by \cref{lemma:sitcom-convergence},
$w_\sigma(s)\to1$ as $s\to\infty$. Together with
\cref{eq:prop5-w0} and continuity of $w_\sigma$, this implies
$0<\tau(\sigma)<\infty$ and
\begin{equation}
    w_\sigma(\tau(\sigma))=\frac12.
    \label{eq:prop5-hit}
\end{equation}

Rearranging \cref{eq:prop5-w-ode}, integrating from $0$ to
$\tau(\sigma)$, and using the change of variables
$u=w_\sigma(s)$ gives
\begin{align}
    \tau(\sigma)
    &=
    \frac{\sigma^4}{C}
    \int_0^{\tau(\sigma)}
    \frac{
        w_\sigma'(s)
    }{
        w_\sigma(s)^2(1-w_\sigma(s))^3
    }\,ds
    \\
    &=
    \frac{\sigma^4}{C}
    \int_{w_\sigma(0)}^{1/2}
    \frac{du}{u^2(1-u)^3}.
    \label{eq:prop5-tau}
\end{align}
The change of variables is valid because $w_\sigma$ is continuously
differentiable and strictly increasing.

We first consider $\sigma\downarrow0$. For
$u\in(0,1/2]$, $(1-u)^{-3}\geq1$, so
\begin{align}
    \tau(\sigma)
    &\geq
    \frac{\sigma^4}{C}
    \int_{w_\sigma(0)}^{1/2}\frac{du}{u^2}
    \\
    &=
    \frac{\sigma^4}{C}
    \left(
        \frac{1}{w_\sigma(0)}-2
    \right)
    \\
    &=
    \frac{\sigma^4}{C}
    \left(
        \exp\left(\frac{b}{\sigma^2}\right)-1
    \right),
    \label{eq:prop5-small-bound}
\end{align}
where the last equality follows from \cref{eq:prop5-w0}. For every
$x>0$, the power series for $e^x$ gives
$e^x-1\geq x^3/6$. Taking $x=b/\sigma^2$ in
\cref{eq:prop5-small-bound},
\begin{equation}
    \tau(\sigma)
    \geq
    \frac{\sigma^4}{C}
    \frac{1}{6}
    \left(\frac{b}{\sigma^2}\right)^3
    =
    \frac{b^3}{6C\sigma^2}
    \longrightarrow\infty
    \qquad\text{as }\sigma\downarrow0.
    \label{eq:prop5-small-limit}
\end{equation}

We next consider $\sigma\to\infty$. For
$u\in(0,1/2]$,
$u^2(1-u)^3\leq u^2\leq1/4$, so
$[u^2(1-u)^3]^{-1}\geq4$. Therefore
\begin{equation}
    \tau(\sigma)
    \geq
    \frac{4\sigma^4}{C}
    \left(
        \frac12-w_\sigma(0)
    \right).
    \label{eq:prop5-large-bound}
\end{equation}
Set $x:=b/\sigma^2$. By \cref{eq:prop5-w0},
\begin{equation}
    \frac12-w_\sigma(0)
    =
    \frac{e^x-1}{2(1+e^x)}.
    \label{eq:prop5-gap}
\end{equation}
For all sufficiently large $\sigma$, $0<x\leq1$. Then
$e^x-1\geq x$ and $e^x\leq e$, so
\[
    \frac12-w_\sigma(0)
    \geq
    \frac{x}{2(1+e)}
    =
    \frac{b}{2(1+e)\sigma^2}.
\]
Combining this with \cref{eq:prop5-large-bound},
\begin{equation}
    \tau(\sigma)
    \geq
    \frac{2b}{C(1+e)}\sigma^2
    \longrightarrow\infty
    \qquad\text{as }\sigma\to\infty.
    \label{eq:prop5-large-limit}
\end{equation}

It remains to show that $\tau$ attains its minimum. Define
$H(q):=\int_q^{1/2}u^{-2}(1-u)^{-3}\,du$ for
$q\in(0,1/2)$. The fundamental theorem of calculus gives
$H'(q)=-q^{-2}(1-q)^{-3}$, so $H$ is continuous. By
\cref{eq:prop5-w0}, $\sigma\mapsto w_\sigma(0)$ is continuous, and
therefore \cref{eq:prop5-tau} gives
\[
    \tau(\sigma)
    =
    \frac{\sigma^4}{C}H(w_\sigma(0)),
\]
so $\tau$ is continuous on $(0,\infty)$.

Fix any $\bar{\sigma}>0$ and set
$M:=\tau(\bar{\sigma})+1$. By
\cref{eq:prop5-small-limit,eq:prop5-large-limit}, there exist
$0<\sigma_{\mathrm{low}}<\bar{\sigma}
<\sigma_{\mathrm{high}}<\infty$ such that
\begin{equation}
    \tau(\sigma)>M
    \qquad
    \text{whenever }
    0<\sigma\leq\sigma_{\mathrm{low}}
    \text{ or }
    \sigma\geq\sigma_{\mathrm{high}}.
    \label{eq:prop5-outside}
\end{equation}
Since $\tau$ is continuous on the compact interval
$[\sigma_{\mathrm{low}},\sigma_{\mathrm{high}}]$, the extreme value
theorem gives some
$\sigma^\star\in[\sigma_{\mathrm{low}},\sigma_{\mathrm{high}}]$ such that
\begin{equation}
    \tau(\sigma^\star)
    =
    \min_{\sigma\in[\sigma_{\mathrm{low}},\sigma_{\mathrm{high}}]}
    \tau(\sigma).
\end{equation}
Since $\bar{\sigma}\in[\sigma_{\mathrm{low}},\sigma_{\mathrm{high}}]$,
we have $\tau(\sigma^\star)\leq\tau(\bar{\sigma})<M$. Thus,
\cref{eq:prop5-outside} gives
$\tau(\sigma)>\tau(\sigma^\star)$ for every
$\sigma\notin[\sigma_{\mathrm{low}},\sigma_{\mathrm{high}}]$, while
minimality on the compact interval gives
$\tau(\sigma)\geq\tau(\sigma^\star)$ for every
$\sigma\in[\sigma_{\mathrm{low}},\sigma_{\mathrm{high}}]$. Therefore,
\begin{equation}
    \tau(\sigma^\star)
    =
    \min_{\sigma>0}\tau(\sigma),
    \qquad
    \sigma^\star
    \in[\sigma_{\mathrm{low}},\sigma_{\mathrm{high}}],
\end{equation}
as desired.
\end{proof}

\subsection{Proof of \Cref{prop:measurement_consistency_grad}}
\label{proof:proof_of_measurement_consistency_grad}

\begin{proof}
Fix $\bm{a}\in\R^n$.
For any $\bm{h}\in\R^n$, define $  \phi(\varepsilon)
    :=
    \mathcal{L}(\bm{a}+\varepsilon\bm{h})$. 
Then
\begin{align}
    \phi(\varepsilon)
    &=
    \frac{1}{2\gamma^2}
    \left\|
        \bm{A}\bm{a}-\bm{u}_{\mathrm{obs}}
        +\varepsilon\bm{A}\bm{h}
    \right\|_2^2 \\
    &=
    \frac{1}{2\gamma^2}
    \left(
        \|\bm{A}\bm{a}-\bm{u}_{\mathrm{obs}}\|_2^2
        +2\varepsilon
        (\bm{A}\bm{h})^\top
        (\bm{A}\bm{a}-\bm{u}_{\mathrm{obs}})
        +\varepsilon^2\|\bm{A}\bm{h}\|_2^2
    \right).
\end{align}
Therefore,
\begin{align}
    \phi'(0)
    &=
    \frac{1}{\gamma^2}
    (\bm{A}\bm{h})^\top
    (\bm{A}\bm{a}-\bm{u}_{\mathrm{obs}}) \\
    &=
    \bm{h}^\top
    \left[
        \frac{1}{\gamma^2}
        \bm{A}^\top
        (\bm{A}\bm{a}-\bm{u}_{\mathrm{obs}})
    \right].
\end{align}
Since
\begin{equation}
    \phi'(0)
    =
    D_{\bm{h}}\mathcal{L}(\bm{a})
    =
    \left\langle
        \nabla_{\bm{a}}\mathcal{L}(\bm{a}),\bm{h}
    \right\rangle
    =
    \bm{h}^\top\nabla_{\bm{a}}\mathcal{L}(\bm{a}),
\end{equation}
for every $\bm{h}\in\R^n$, comparison with the expression above gives
\begin{equation}
    \nabla_{\bm{a}}\mathcal{L}(\bm{a})
    =
    \frac{1}{\gamma^2}
    \bm{A}^\top
    (\bm{A}\bm{a}-\bm{u}_{\mathrm{obs}}),
\end{equation}
which proves the first claim.

We now show that this gradient has no component in $\ker\bm{A}$.
Set
\begin{equation}
    \bm{y}
    :=
    \bm{A}\bm{a}-\bm{u}_{\mathrm{obs}}
    \in\R^m.
\end{equation}
For every $\bm{z}\in\ker\bm{A}$,
\begin{equation}
    \left\langle
        \bm{A}^\top\bm{y},\bm{z}
    \right\rangle
    =
    \left\langle
        \bm{y},\bm{A}\bm{z}
    \right\rangle
    =
    0,
\end{equation}
because $\bm{A}\bm{z}=0$. Hence
\begin{equation}
    \bm{A}^\top\bm{y}
    \in
    (\ker\bm{A})^\perp.
\end{equation}
The orthogonal projection of a vector in $(\ker\bm{A})^\perp$
onto $\ker\bm{A}$ is zero. Consequently,
\begin{align}
    \bm{P}_{\ker\bm{A}}
    \nabla_{\bm{a}}\mathcal{L}(\bm{a})
    &=
    \frac{1}{\gamma^2}
    \bm{P}_{\ker\bm{A}}
    \bm{A}^\top
    (\bm{A}\bm{a}-\bm{u}_{\mathrm{obs}}) \\
    &=0,
\end{align}
as desired.
\end{proof}

\subsection{Proof of \Cref{prop:preconditioning}}
\label{proof:proof_of_preconditioning}

\begin{proof}
Let $\bm{d}:=\bm{a}_+-\bm{a}_-,
    \qquad
    w(s):=k_{t_k}(\bm{v}(s))$. 
Define the Jacobian of the exact denoiser by $ \bm{J}_{D^*}(\bm{v},t_k)
    :=
    \nabla_{\bm{v}}D^*(\bm{v},t_k)
    \in\R^{n\times n}$, 
and define
\begin{equation}
    \bm{P}(\bm{v},t_k)
    :=
    \bm{J}_{D^*}(\bm{v},t_k)
    \bm{J}_{D^*}(\bm{v},t_k)^\top.
    \label{eq:preconditioner_definition}
\end{equation}

Define $ \mathcal{J}_{t_k}(\bm{v})
    :=
    \mathcal{L}(D^*(\bm{v},t_k))$. 
The gradient flow in \cref{eq:sitcom-input-flow} is therefore $ \bm{v}'(s)
    =
    -\nabla_{\bm{v}}\mathcal{J}_{t_k}(\bm{v})
    \big|_{\bm{v}=\bm{v}(s)}$. 
By the multivariate chain rule,
\begin{equation}
    \nabla_{\bm{v}}\mathcal{J}_{t_k}(\bm{v})
    =
    \bm{J}_{D^*}(\bm{v},t_k)^\top
    \nabla_{\bm{a}}\mathcal{L}(\bm{a})
    \big|_{\bm{a}=D^*(\bm{v},t_k)}.
\end{equation}
Consequently,
\begin{equation}
    \bm{v}'(s)
    =
    -
    \bm{J}_{D^*}(\bm{v}(s),t_k)^\top
    \nabla_{\bm{a}}\mathcal{L}(\bm{a})
    \big|_{\bm{a}=\bm{a}(s)},
    \label{eq:preconditioner_v_flow}
\end{equation}
where $\bm{a}(s)
    :=
    D^*(\bm{v}(s),t_k)$. 
Differentiating $\bm{a}(s)$ and applying the chain rule once more gives
\begin{align}
    \bm{a}'(s)
    &=
    \bm{J}_{D^*}(\bm{v}(s),t_k)\bm{v}'(s) \\
    &=
    -
    \bm{J}_{D^*}(\bm{v}(s),t_k)
    \bm{J}_{D^*}(\bm{v}(s),t_k)^\top
    \nabla_{\bm{a}}\mathcal{L}(\bm{a})
    \big|_{\bm{a}=\bm{a}(s)} \\
    &=
    -
    \bm{P}(\bm{v}(s),t_k)
    \nabla_{\bm{a}}\mathcal{L}(\bm{a})
    \big|_{\bm{a}=\bm{a}(s)},
    \label{eq:preconditioned_output_flow}
\end{align}
as desired. 

We next show explicitly that the preconditioner in
\cref{eq:preconditioner_definition} can introduce a nonzero
measurement nullspace component. By \cref{lemma:two-point-denoiser}, we have that 
\begin{equation}
    D^*(\bm{v},t_k)
    =
    \bm{a}_-+k_{t_k}(\bm{v})\bm{d}.
    \label{eq:precond_exact_denoiser}
\end{equation}
Moreover, by \cref{lemma:two-point-denoiser}, we have that 
\begin{equation}
    k_{t_k}(\bm{v})
    =
    \left(
        1+
        \exp\left\{
            -
            \frac{
                \bm{d}^\top\bm{K}^{-1}
                \left(
                    \bm{v}
                    -\frac{\bm{a}_-+\bm{a}_+}{2}
                \right)
            }{\sigma_{t_k}^2}
        \right\}
    \right)^{-1}.
\end{equation}
Thus,
\begin{equation}
    \nabla_{\bm{v}}k_{t_k}(\bm{v})
    =
    \frac{
        k_{t_k}(\bm{v})
        (1-k_{t_k}(\bm{v}))
    }{\sigma_{t_k}^2}
    \bm{K}^{-1}\bm{d}.
    \label{eq:precond_grad_k}
\end{equation}
Differentiating \cref{eq:precond_exact_denoiser} with respect to
$\bm{v}$ therefore gives
\begin{equation}
    \bm{J}_{D^*}(\bm{v},t_k)
    =
    \bm{d}
    \left(
        \nabla_{\bm{v}}k_{t_k}(\bm{v})
    \right)^\top.
    \label{eq:precond_denoiser_jacobian}
\end{equation}
Substituting \cref{eq:precond_grad_k} into
\cref{eq:precond_denoiser_jacobian} and then into
\cref{eq:preconditioner_definition} yields
\begin{align}
    \bm{P}(\bm{v},t_k)
    &=
    \bm{d}
    \left(
        \nabla_{\bm{v}}k_{t_k}(\bm{v})
    \right)^\top
    \nabla_{\bm{v}}k_{t_k}(\bm{v})
    \bm{d}^\top \\
    &=
    \left\|
        \nabla_{\bm{v}}k_{t_k}(\bm{v})
    \right\|_2^2
    \bm{d}\bm{d}^\top \\
    &=
    \frac{
        k_{t_k}(\bm{v})^2
        (1-k_{t_k}(\bm{v}))^2
    }{\sigma_{t_k}^4}
    \|\bm{K}^{-1}\bm{d}\|_2^2
    \bm{d}\bm{d}^\top.
    \label{eq:preconditioner_explicit}
\end{align}

We now evaluate the output space optimization along $\bm{v}(s)$. By
\cref{eq:precond_exact_denoiser},  $ \bm{a}(s)
    =
    \bm{a}_-+w(s)\bm{d}.$
Since $\bm{u}_{\mathrm{obs}}=\bm{A}\bm{a}_+$ and
$\bm{a}_+=\bm{a}_-+\bm{d}$,
\begin{align}
    \bm{A}\bm{a}(s)-\bm{u}_{\mathrm{obs}}
    &=
    \bm{A}\bm{a}_-
    +w(s)\bm{A}\bm{d}
    -\bm{A}\bm{a}_+ \\
    &=
    -(1-w(s))\bm{A}\bm{d}.
    \label{eq:precond_measurement_residual}
\end{align}
By \cref{prop:measurement_consistency_grad},
\begin{align}
    \nabla_{\bm{a}}\mathcal{L}(\bm{a})
    \big|_{\bm{a}=\bm{a}(s)}
    &=
    \frac{1}{\gamma^2}
    \bm{A}^\top
    \left(
        \bm{A}\bm{a}(s)-\bm{u}_{\mathrm{obs}}
    \right) \\
    &=
    -\frac{1-w(s)}{\gamma^2}
    \bm{A}^\top\bm{A}\bm{d}.
    \label{eq:precond_measurement_gradient}
\end{align}
Using \cref{eq:preconditioner_explicit} and
\cref{eq:precond_measurement_gradient} in
\cref{eq:preconditioned_output_flow} gives
\begin{align}
    \bm{a}'(s)
    &=
    \frac{
        w(s)^2(1-w(s))^3
    }{\gamma^2\sigma_{t_k}^4}
    \|\bm{K}^{-1}\bm{d}\|_2^2
    \bm{d}\bm{d}^\top
    \bm{A}^\top\bm{A}\bm{d} \\
    &=
    \frac{
        \|\bm{A}\bm{d}\|_2^2
        \|\bm{K}^{-1}\bm{d}\|_2^2
    }{
        \gamma^2\sigma_{t_k}^4
    }
    w(s)^2(1-w(s))^3
    \bm{d},
    \label{eq:precond_output_derivative}
\end{align}
since $ \bm{d}^\top\bm{A}^\top\bm{A}\bm{d}
    =
    \|\bm{A}\bm{d}\|_2^2$. 
Applying $\bm{P}_{\ker\bm{A}}$ to
\cref{eq:precond_output_derivative} yields
\begin{equation}
    \bm{P}_{\ker\bm{A}}\bm{a}'(s)
    =
    \frac{
        \|\bm{A}\bm{d}\|_2^2
        \|\bm{K}^{-1}\bm{d}\|_2^2
    }{
        \gamma^2\sigma_{t_k}^4
    }
    w(s)^2(1-w(s))^3
    \bm{P}_{\ker\bm{A}}\bm{d}.
    \label{eq:precond_kernel_derivative}
\end{equation}
For every finite $s\geq0$, we have that $w(s) \in (0,1)$ by \cref{lemma:positive-weights}, and furthermore $ \bm{A}\bm{d}\neq0$ and $
    \bm{P}_{\ker\bm{A}}\bm{d}\neq0$ by assumption. 
Since $\bm{K}$ is positive definite, it is invertible, and
$\bm{d}\neq0$ implies $  \|\bm{K}^{-1}\bm{d}\|_2^2>0$. 
Together with $\gamma>0$ and $\sigma_{t_k}>0$, every scalar factor
multiplying $\bm{P}_{\ker\bm{A}}\bm{d}$ in
\cref{eq:precond_kernel_derivative} is strictly positive. Therefore,
$\bm{P}_{\ker\bm{A}}
    \frac{d}{ds}\bm{a}(s)
    =
    \bm{P}_{\ker\bm{A}}\bm{a}'(s)
    \neq0$ 
for every finite $s\geq0$, as desired.
\end{proof}

\subsection{Statement and Proof of \Cref{prop:finite_diff_assumps}}
\label{proof:proof_of_finite_diff_assumps}

We first state the proposition: 
\begin{proposition}
    Assume $0 < z_- < z_+ < 1$ without loss of generality, and $1 \leq m \leq n$. Let $\bm{A} = \bm{OG}$ and $\bm{P}_{\mathrm{ker} \bm{A}}$ be the orthogonal projection matrix of a vector onto $\mathrm{ker} \bm{A}$. Under a centered finite difference discretization of $\mathcal{G}$ as $\bm{G}$ with $n > \max\{\frac{1}{z_-}, \frac{1}{1-z_+}, \frac{2}{z_+ - z_-}\}$, if $\bm{O}$ samples $m$ points uniformly at random, without replacement, from $\bm{u}$, then $\bm{A}\bm{a}_+ \neq \bm{A}\bm{a}_-$ and with probability at least $1-(m/n)^4$, $\bm{P}_{\ker \bm{A}}\bm{a}_+ \neq \bm{P}_{\ker \bm{A}}\bm{a}_- $. 
    \label{prop:finite_diff_assumps}
\end{proposition}

\textbf{Discussion}: The condition on $n$ ensures that the resolution is fine enough that there are some grid points between the locations of $z_-$ and $z_+$. Additionally, we ensure that $z_-$ and $z_+$ are not at the end of the grid. 

We then provide the proof, which leverages the classic result by \citet{thomas1949elliptic}: 

\begin{proof}
Let
\begin{equation}
    h:=\frac{1}{n+1},
    \qquad
    x_i:=ih,
    \qquad
    i=0,\ldots,n+1,
\end{equation}
so that $x_0=0$ and $x_{n+1}=1$. Under the centered finite
difference discretization of $u''(\cdot)=a(\cdot)$ on $(0,1)$ with homogeneous Dirichlet
boundary conditions, define $\bm{L}\in\R^{n\times n}$ by
\begin{equation}
    (\bm{L}\bm{u})_i
    :=
    \frac{u_{i-1}-2u_i+u_{i+1}}{h^2},
    \qquad i=1,\ldots,n,
\end{equation}
where $u_0=u_{n+1}=0$. The corresponding discrete solution
operator is $\bm{G}=\bm{L}^{-1}$. 

We first identify the indices on which
$\bm{d}:=\bm{a}_+-\bm{a}_-$ is nonzero. Define
\begin{equation}
    p:=\min\{i\in\{1,\ldots,n\}:x_i\ge z_-\},
    \qquad
    q:=\max\{i\in\{1,\ldots,n\}:x_i<z_+\}.
\end{equation}
The assumption $    n>
    \max\left\{
        \frac{1}{z_-},
        \frac{1}{1-z_+},
        \frac{2}{z_+-z_-}
    \right\}$ 
implies $h=\frac{1}{n+1}<\frac{1}{n}<z_-$, 
and hence $x_1=h<z_-$, so $p\ge2$. Similarly, $h<\frac{1}{n}<1-z_+$, 
so that $ x_n=1-h>z_+$ 
and therefore $q\le n-1$. Finally, we have that $ 2h<\frac{2}{n}<z_+-z_-$.
By the definition of $p$, $    x_{p-1}<z_-\le x_p,$ so that $   x_{p+1}
    =
    x_{p-1}+2h
    <
    z_-+2h
    <
    z_+$, and thus 
\begin{equation}
    2\le p<q\le n-1.
\end{equation}

Since $z_-<z_+$, the two discretized coefficient fields differ
precisely at the grid points $x_i\in[z_-,z_+)$. Hence, setting $c:=a^{(1)}-a^{(2)}\neq0$, 
we have
\begin{equation}
    \bm{d}
    =
    c\sum_{j=p}^{q}\bm{e}_j,
\end{equation}
where $\bm{e}_j$ is the $j$th standard basis vector of $\R^n$.

We first show that $\bm{A}\bm{d}\neq0$. For $i,j\in\{1,\ldots,n\}$, define
\begin{equation}
    g_{ij}
    :=
    -h\min(x_i,x_j)
    \left(1-\max(x_i,x_j)\right),
\end{equation}
and let $\bm{g}_j
    :=
    (g_{1j},\ldots,g_{nj})^\top$. We verify that  $ \bm{L}\bm{g}_j=\bm{e}_j$. Firstly, extend the definition of $g_{ij}$ so that $g_{0j}=g_{n+1,j}=0$. If $i<j$,
then $g_{ij}=-h x_i(1-x_j)$, which yields $ \frac{g_{i-1,j}-2g_{ij}+g_{i+1,j}}{h^2}=0$. Similarly, if $i>j$, then $ g_{ij}=-h x_j(1-x_i)$, which then yields that $ \frac{g_{i-1,j}-2g_{ij}+g_{i+1,j}}{h^2}=0$. When $i=j$, using $x_{j-1}=x_j-h$ and
$x_{j+1}=x_j+h$,
\begin{align}
    g_{j-1,j}-2g_{jj}+g_{j+1,j}
    &=
    -h\Big[
        (x_j-h)(1-x_j)
        -2x_j(1-x_j)
        +x_j(1-x_j-h)
    \Big]
    \\
    &=h^2.
\end{align}
Consequently, $\frac{g_{j-1,j}-2g_{jj}+g_{j+1,j}}{h^2}=1$. Thus $\bm{L}\bm{g}_j=\bm{e}_j$ for every $j$, so $ \bm{G}
    =
    \bm{L}^{-1}
    =
    [\bm{g}_1,\ldots,\bm{g}_n]$. 
In particular, $\bm{G}$ is symmetric, since $g_{ij}=g_{ji}$, and
for every $i,j\in\{1,\ldots,n\}$, $ g_{ij}<0$, 
because $h>0$ and $x_i,x_j\in(0,1)$.

It follows that, for every $i=1,\ldots,n$, $   (\bm{G}\bm{d})_i
    =
    c\sum_{j=p}^{q}g_{ij}$, and since $g_{ij} <0, c \neq 0$, $ (\bm{G}\bm{d})_i\neq0$ for every $i = 1, ..., n$, and thus 
$\bm{A}\bm{d}
    =
    \bm{O}\bm{G}\bm{d}
    \neq0$ as desired. 

It remains to characterize
$\bm{P}_{\mathrm{ker}\,\bm{A}}\bm{d}$. We have that 
\begin{align}
    \bm{P}_{\mathrm{ker}\,\bm{A}}\bm{d}=0
    &\iff 
    \bm{d} \in \operatorname{range}(\bm{A}^T)
    \\
    &\iff
    \bm{d}\in
    \operatorname{range}(\bm{G}\bm{O}^\top)
    \\
    &\iff
    \bm{L}\bm{d}\in
    \operatorname{range}(\bm{O}^\top),
\end{align}
where the last equivalence follows from
$\bm{L}=\bm{G}^{-1}$.

We now compute $\bm{L}\bm{d}$. Since
\begin{equation}
    d_i
    =
    \begin{cases}
        c, & p\le i\le q,\\
        0, & \text{otherwise},
    \end{cases}
\end{equation}
and $2\le p<q\le n-1$, the centered difference gives
\begin{equation}
    (\bm{L}\bm{d})_i
    =
    \begin{cases}
        \dfrac{c}{h^2}, & i=p-1,\\[1ex]
        -\dfrac{c}{h^2}, & i=p,\\[1ex]
        -\dfrac{c}{h^2}, & i=q,\\[1ex]
        \dfrac{c}{h^2}, & i=q+1,\\[1ex]
        0, & \text{otherwise}.
    \end{cases}
\end{equation}
Equivalently, $ \bm{L}\bm{d}
    =
    \frac{c}{h^2}
    \left(
        \bm{e}_{p-1}
        -\bm{e}_p
        -\bm{e}_q
        +\bm{e}_{q+1}
    \right)$. 
Since $c\neq0$ and $p<q$, its support is therefore exactly
\begin{equation}
    B
    :=
    \operatorname{supp}(\bm{L}\bm{d})
    =
    \{p-1,p,q,q+1\},
    \qquad |B|=4.
\end{equation}

Let $S\subset\{1,\ldots,n\},
    \qquad
    |S|=m > 0$
be the set of indices sampled by $\bm{O}$. Since $\bm{O}$ is a
selection matrix, $\operatorname{range}(\bm{O}^\top)
    =
    \operatorname{span}\{\bm{e}_i:i\in S\}$. 
Thus $\bm{L}\bm{d}\in\operatorname{range}(\bm{O}^\top)
    \iff
    B\subset S$, so that $ \bm{P}_{\mathrm{ker}\,\bm{A}}\bm{d}=0
    \iff
    B\subset S$. 
    
If $m<4$, then $\mathbb{P}(B\subset S)=0$. If $m\ge4$, since $S$ is
chosen uniformly from the $m$-element subsets of
$\{1,\ldots,n\}$,
\begin{align}
    \mathbb{P}(B\subset S)
    &=
    \frac{\binom{n-4}{m-4}}{\binom{n}{m}}
    \\
    &=
    \frac{m(m-1)(m-2)(m-3)}
    {n(n-1)(n-2)(n-3)}
    \\
    &=
    \prod_{r=0}^{3}\frac{m-r}{n-r}.
\end{align}
For $r=0,1,2,3$ and $m\le n$, $ \frac{m-r}{n-r}
    \le
    \frac{m}{n}$, since $n(m-r)\le m(n-r)
    \iff
    mr\le nr$. 
Therefore, in either case $m<4$ or $m\ge4$,
\begin{equation}
    \mathbb{P}\left(
        \bm{P}_{\mathrm{ker}\,\bm{A}}\bm{d}=0
    \right)
    =
    \mathbb{P}(B\subset S)
    \le
    \left(\frac{m}{n}\right)^4.
\end{equation}
Taking complements yields
\begin{equation}
    \mathbb{P}\left(
        \bm{P}_{\mathrm{ker}\,\bm{A}}\bm{d}\neq0
    \right)
    \ge
    1-\left(\frac{m}{n}\right)^4.
\end{equation}
Together with $\bm{A}\bm{d}\neq0$, this proves the proposition as desired. 
\end{proof}

\typeout{get arXiv to do 4 passes: Label(s) may have changed. Rerun}
\end{document}